\documentclass[opre]{informs3aa}

\OneAndAHalfSpacedXI 

\usepackage{longtable}
\usepackage{tabularx}
\usepackage{booktabs}
\usepackage{multirow}
\usepackage{array}
\usepackage{bbm}

\usepackage[normalem]{ulem}

\usepackage[margin=1in]{geometry} 
\usepackage{tcolorbox}
\usepackage{amssymb}
\usepackage{lastpage}
\usepackage{fancyhdr}
\usepackage[mathscr]{euscript}

\usepackage{amsmath,amsfonts,bm,physics,bbm}

\def\eqref#1{(\ref{#1})}

\def\1{\bm{1}}

\def\eps{{\epsilon}}

\DeclareMathAlphabet{\mathsfit}{\encodingdefault}{\sfdefault}{m}{sl}
\SetMathAlphabet{\mathsfit}{bold}{\encodingdefault}{\sfdefault}{bx}{n}

\newcommand{\E}{\mathbb{E}}

\usepackage{url}
\usepackage{verbatim}
\usepackage[inline]{enumitem}
\usepackage{xcolor}
\usepackage[ruled,noline,noend,linesnumbered]{algorithm2e}
\usepackage{epstopdf}

\usepackage[colorlinks=true,bookmarks=false,urlcolor=black, citecolor=blue,linkcolor=blue,bookmarksopen=false,draft=false]{hyperref}

\setlist[enumerate, 1]{label=(\alph*)}
\setlist[enumerate]{nosep}

\SetCommentSty{rmfamily}

\SetKwComment{Comment}{$\triangleright$\ }{}

\newcommand{\blue}[1]{{#1}}

\renewcommand{\E}{\mathop{\mathbb{E}}}

\let\hat\widehat 
\let\tilde\widetilde

\definecolor{DSgray}{cmyk}{0,1,0,0}

\usepackage{natbib}
 \bibpunct[, ]{(}{)}{,}{a}{}{,}%
 \def\bibfont{\small}%
\TheoremsNumberedThrough     
\ECRepeatTheorems

\EquationsNumberedThrough    

\MANUSCRIPTNO{} 

\begin{document}




\RUNTITLE{Fairness-aware Online Price Discrimination}

\TITLE{Fairness-aware Online Price Discrimination with Nonparametric Demand Models}


\ARTICLEAUTHORS{%
\AUTHOR{Xi Chen\thanks{Author names listed in alphabetical order.}}
\AFF{Leonard N.~Stern School of Business, New York University, \EMAIL{xc13@stern.nyu.edu}}
\AUTHOR{Jiameng Lyu}
\AFF{Department of Mathematical Sciences, Tsinghua University, 
\EMAIL{lvjm21@mails.tsinghua.edu.cn}}
\AUTHOR{Xuan Zhang}
\AFF{Department of Industrial and Enterprise Systems Engineering, University of Illinois at Urbana-Champaign, 
\EMAIL{xuan6@illinois.edu}
}
\AUTHOR{Yuan Zhou}
\AFF{Yau Mathematical Sciences Center and Department of Mathematical Sciences, Tsinghua University, 
\EMAIL{yuan-zhou@tsinghua.edu.cn}
}
} 

\ABSTRACT{Price discrimination, which refers to the strategy of setting different prices for different customer groups, has been widely used in online retailing. Although it helps boost the collected revenue for online retailers, it might create serious concerns about fairness, which even violates the regulation and laws. This paper studies the problem of dynamic discriminatory pricing under fairness constraints. In particular, we consider a finite selling horizon of length $T$ for a single product with two groups of customers. Each group of customers has its unknown demand function that needs to be learned. For each selling period, the seller determines the price for each group and observes their purchase behavior. While existing literature mainly focuses on maximizing revenue, ensuring fairness among different customers has not been fully explored in the dynamic pricing literature. This work adopts the fairness notion from \citet{cohen2022price}. For price fairness, we propose an optimal dynamic pricing policy regarding regret, which enforces the strict price fairness constraint. In contrast to the standard $\sqrt{T}$-type regret in online learning, we show that the optimal regret in our case is $\tilde{O}(T^{4/5})$. We further extend our algorithm to a more general notion of fairness, which includes demand fairness as a special case. To handle this general class, we propose a soft fairness constraint and develop a dynamic pricing policy that achieves $\tilde{O}(T^{4/5})$ regret. We also demonstrate that our algorithmic techniques can be adapted to more general scenarios such as fairness among multiple groups of customers.

}




\maketitle

%


\section{Introduction}

Data-driving algorithms have been widely applied to automated decision-making in operations management, such as personalized pricing, online recommendation. Traditional operational decisions mainly seek "globally optimal" decisions. However, such decisions could be unfair to a certain population segment (e.g., a demographic group or protected class). The issue of fairness is particularly critical in e-commerce. Indeed, the increasing prominence of e-commerce has given retailers unprecedented power to understand customers as individuals and to implement discriminatory pricing strategies that could be unfair to a specific customer group. As pointed by \cite{Booth:personalized}, the ``biggest drawback of individualized prices is that they could offend customers' sense of fairness.''   For example, the study in \cite{Pandey:21:disparate} analyzes more than 100 million ride-hailing observations in Chicago. It shows that higher fare prices appear in neighborhoods with larger ``non-white populations, higher poverty levels, younger residents, and higher education levels.'' In the auto loan market, ``the Black and Hispanic applicant's loan approval rates are 1.5 percentage points lower, even controlling for creditworthiness'' \citep{Butler:21:racial}. 
Amazon once charged the customers who discussed DVDs at the website \url{DVDTalk.com} more than 40\% than other customers for buying DVDs \citep{Steritfeld:00}.  As a consequence, Amazon publicly apologized and made refunds to 6,896 customers. In addition to customers' backfiring,  many regulations have been made to ensure fairness in various industries (see Section~\ref{sec: comparision} for detailed discussions). 
Despite the importance of fairness in decision-making, the study of fairness-aware dynamic pricing is still somewhat limited.

This paper studies the problem of dynamic pricing with nonparametric demand functions under fairness constraints. There is a wide range of fairness notions from the online learning community, which will be briefly surveyed in Section \ref{sec:related}. We have to admit that ``fairness'' is a somewhat ambiguous definition, and there is no consensus on the best notion for pricing applications. Therefore, the main purpose of the paper is not to argue the most suitable fairness definition in dynamic pricing. Instead, we adopt the fairness notation from a recent paper on static pricing in the operations management literature \citep{cohen2022price}  and extend it to the dynamic pricing setting. More explanations on the reason why we adopt the fairness notion from \cite{cohen2022price} will be provided after Eq. \eqref{eq:price_fairness}.

To highlight our main idea, we consider the simplest setup of monopoly selling over $T$ periods to two customer groups without any inventory constraints. Although the two-group setting might be too simple in practice, it serves as the foundation in studying group fairness. Each group of customers has its underlying demand function in price, which is \emph{unknown} to the seller. At each period $t=1,\ldots, T$, the seller offers a single product to each group $i \in \{1,2\}$ with the price $p_i^{(t)}$, and observes the realized demands from each group. Most existing dynamic pricing literature only focuses on learning the demand function to maximize revenue over time. This work enforces the fairness constraints into this dynamic pricing problem with demand learning. We first consider a well-received notion of \emph{price fairness} introduced by \citet{cohen2022price}. Later, we will extend to a more general fairness notion, which includes demand fairness as a special case. 
More specifically, let $p_1^{\sharp}$  and $p_2^{\sharp}$ denote the optimal price for each customer group without any fairness constraint. Price fairness requires that for all time periods $t$, we have 
\begin{equation}\label{eq:price_fairness}
|p_1^{(t)}-p_2^{(t)}|\leq \lambda |p_1^{\sharp}-p_2^{\sharp}|.
\end{equation}
The parameter $\lambda \in (0,1)$ controls the fairness level, which should be pre-defined by the seller. The smaller the value of $\lambda$, the more strict fairness constraint the seller needs to achieve.  

As we have explained, ``fairness'' is a subtle concept, and each user might have her preference depending on the context. We adopt this specific form of ``fairness'' from \cite{cohen2022price} due to three reasons. First, this fairness can be viewed as  ``outcome fairness'' among groups, which is easy to understand and can be formulated into a well-defined optimization problem. Second, the recent work by \cite{cohen2022price} has been well-received in the OM community, and thus we decided to build our work based on the notation in \cite{cohen2022price}. Finally, using the fairness notion from \cite{cohen2022price} leads to an interesting ``globally constrained optimization'' formulation with unknown quantities in constraint (e.g., $p_1^{\sharp}$, $p_2^{\sharp}$). Such a formulation has not been well-explored in existing online learning literature, and the developed techniques in this paper could shed light on other ``global constrained'' online learning problems.  We further note that a concurrent work by \cite{Cohen:21:dynamic} studies a dynamic pricing under the \emph{absolute} price fairness with parametric demand function, which restricts the price difference to be upper bounded by a \emph{fixed known constant} (i.e., $|p_1^{(t)}-p_2^{(t)}| < C$). It turns out the relative-gap fairness constraint in \eqref{eq:price_fairness} leads to a fundamental difference in terms of regret behavior. When using an absolute price gap in \cite{Cohen:21:dynamic} without any unknown quantity in the fairness constraint, the regret is the same as the standard dynamic pricing regret of $\sqrt{T}$. However, as we will prove in this paper, our minimax lower bound of the regret becomes $\Omega(T^{4/5})$ and this phenomenon has its independent interest in online learning literature. We will provide detailed comparisons to \cite{Cohen:21:dynamic} in \blue{Section~\ref{sec: comparision}} and Section \ref{sec:related} (see Page \pageref{page:fixed}).

 For the ease of presentation, we refer to such a constraint in Eq.~\eqref{eq:price_fairness} as the ``hard fairness constraint'', in contrast to the  ``soft fairness constraint'' introduced later, which only requires the pricing policy to \emph{approximately} satisfy the fairness constraint. The main goal of this paper is to develop an efficient dynamic pricing policy that ensures the fairness constraint in Eq.~\eqref{eq:price_fairness} over the entire selling horizon and to quantify the revenue gap between the dynamic pricing and static pricing under the fairness constraint.

For the price fairness constraint in Eq.~\eqref{eq:price_fairness}, we developed a dynamic pricing algorithm that achieves the regret at the order of $\tilde{O}(T^{4/5})$, where $\tilde{O}(\cdot)$ hides logarithmic factors in $T$. The cumulative regret is defined as the revenue gap between our pricing policy and the static clairvoyant prices under the fairness constraint. This regret bound is fundamentally different from the $\sqrt{T}$-type bound in ordinary dynamic pricing \citep{broder2012dynamic,wang2014close} and absolute-price-fairness-constrained pricing in \citet{Cohen:21:dynamic}. Please see Section~\ref{sec:related} on related works for more discussions and comparisons.

Our pricing policy contains three stages. The first stage tries to estimate the optimal prices $p_1^{\sharp}$  and $p_2^{\sharp}$  without fairness constraints. By leveraging the common assumption that the revenue is a strongly concave function in demand \citep{Jasin:14}, we developed a \emph{tri-section search} algorithm to obtain the estimates $\hat p_1^{\sharp}$  and $\hat p_2^{\sharp}$. The estimates $\hat p_1^{\sharp}$  and $\hat p_2^{\sharp}$ from the first stage enable us to construct an approximate fairness constraint. Based on  $\hat p_1^{\sharp}$  and $\hat p_2^{\sharp}$, the second stage uses a \emph{discretization technique} to estimate optimal prices $\hat p_1^*$ and $\hat p_2^*$ of the static pricing problem under the fairness constraint.  For the rest of the time periods, we will offer the prices $\hat p_1^*$ and $\hat p_2^*$ to two groups of customers. The policy is easy to implement as it is essentially an explore-exploit scheme.  We further establish the information-theoretical lower bound $\Omega(T^{4/5})$ to show that the regret of the policy is optimal up to a logarithmic factor.

The second part of the paper extends the price fairness to a general fairness measure. Consider a fairness measure $M_i(p)$ for each group $i \in \{1,2\}$, where $M_i$ is a Lipschitz continuous function in price. For example, the case $M_i(p)=p$ reduces the price fairness.  We could also consider demand fairness by defining $M_i(p)=d_i(p)$, where $d_i$ is the expected demand for the $i$-th customer group. Now the fairness constraint can be naturally extended to 
\begin{align}\label{eq: relative constraint general form}
    |M_1(p_1^{(t)})-M_2(p_2^{(t)})|<\lambda |M_1(p_1^{\sharp})-M_2(p_2^{\sharp})|.
\end{align}
However, in practice, it is impossible to enforce such a hard constraint over all time periods since the fairness measure $M_i$ is \emph{unknown} to the seller. For example, in demand fairness, the demand function $d_i$ needs to be learned via the interactions between the seller and customers. To this end, we propose the ``soft fairness constraint'', which adds the following penalty term to the regret minimization problem,
\begin{equation}\label{eq:soft_constraint}
\gamma \max\left(|M_1(p_1^{(t)})-M_2(p_2^{(t)})|-\lambda |M_1(p_1^{\sharp})-M_2(p_2^{\sharp})|, 0\right),
\end{equation}
where the parameter $\gamma$ balances between the regret minimization and the fairness constraint. When $\gamma=0$, there would be no fairness constraint; while when $\gamma = \infty$, it is equivalent to ``hard fairness constraint''.  Under the general fairness measure $M_i$ and the soft fairness constraint in Eq.~\eqref{eq:soft_constraint}, we develop a dynamic pricing policy, which achieves the penalized regret at the order of $\tilde{O}(T^{4/5})$ for $\gamma \leq O(1)$.\footnote{As discussed before, an extremely large $\gamma$ enforces the hard fairness constraint which is almost impossible to be met due to the unknown $M_i$. Therefore, we have to assume an upper bound on $\gamma$ to control the penalized regret. In our result, $\gamma$ can be as large as $O(1)$ to achieve the desired regret bound. This is a mild constraint since the maximum possible profit per selling period is also $O(1)$, and it makes sense to assume that the penalty imposed due to the fairness violation is comparable to the profit.}

\subsection{\blue{Comparison with Other Fairness Constraints}}\label{sec: comparision}
\blue{In this section, we will  provide several realistic contexts where our relative fairness constraint \eqref{eq: relative constraint general form} is shown to be more suitable and practical than other types of constraints, such as the absolute constraint $|p_1^{(t)}-p_2^{(t)}|\leq C$   and another potentially useful relative constraint}
\begin{align}\label{eq:another relative constraint}
   \blue{|p_1^{(t)}-p_2^{(t)}|/\min\{p_1^{(t)}, p_2^{(t)}\}\leq \lambda} .
\end{align}

\blue{One important regulatory approach employed by regulators such as the Federal Trade Commission (FTC) in the US and the Financial Conduct Authority (FCA) in the UK is to require retailers to disclose their pricing algorithms \citep{ftc2020,fca2018}. Regulators subsequently review the design, input, and output of these algorithms to ensure compliance with relevant fairness regulations. Retailers, on the other hand, may proactively adopt fairness-aware dynamic pricing strategies to prevent customer dissatisfaction and reputational harm resulting from price discrimination practices, while also meeting regulatory requirements.}

\blue{\noindent\underline{\bf Comparisons between the relative constraint and the absolute constraint.} The relative constraints offer several distinct advantages over the absolute constraint for both regulators and retailers. Firstly, the relative parameter $\lambda$ of the relative constraints, which can effectively represent the fairness level of the pricing strategy, is independent of the specific product. In contrast, the appropriate absolute constraint may vary significantly for different products. This particular characteristic of the relative constraint is highly advantageous for both parties. 
For regulators, regulation is costly and their aim is to conduct regulation efficiently and cost-effectively, since both direct and indirect regulatory costs are likely to be passed onto individuals and businesses through higher prices (see Chapter 3 in \cite{fca2017}).
With the relative constraints, regulators can establish a consistent regulation policy  that applies uniformly to all products, which is efficient and cost-effective.
For retailers, the relative constraints greatly streamline their fairness-aware pricing strategies, since they no longer need to individually adjust the absolute price constraint for each product. 

Secondly, in practical terms, tuning hyper-parameters within a smaller range $[0,1]$ is often considerably easier compared to an unbounded and larger range $\mathbb{R}^+$. Furthermore, in terms of achieving demand fairness, setting a predetermined absolute constraint $C$ is significantly more challenging and impractical since the demand function is unknown beforehand.}

\blue{\noindent\underline{\bf Comparisons between two types of relative constraints.} 

Comparing the relative constraints \eqref{eq: relative constraint general form} and \eqref{eq:another relative constraint}, we note that our focused constraint \eqref{eq: relative constraint general form} uses the unconstrained optimal pricing decision as benchmark and reflects the cost of the retailer due to the fairness constraint. Such a cost directly influences the willingness of the retailers to comply with the fairness regulations. Indeed, fairness regulation can sometimes resemble a negotiation between retailers and regulators, where the retailers aim to maximize their profit, and the regulators prioritize social welfare. The cost reflected by \eqref{eq: relative constraint general form} is also an important consideration from the regulators' side. This is because the regulators usually follow the principle of proportionality and take into account the costs of their interventions on both retailers and consumers (see Chapter 4 of \cite{fca2017} and Chapter 5 of \cite{fca2018}). Altogether, we find that \eqref{eq: relative constraint general form} would be a more appropriate fairness constraint than \eqref{eq:another relative constraint}.

Additionally, there have been several theoretical studies on the fairness constraint  \eqref{eq: relative constraint general form}. For example,   \cite{cohen2022price} systematically investigate how social welfare is affected by the relative parameter $\lambda$ under various definitions of fairness, and \cite{yang2022fairness}  revealed an intriguing interaction between market competition and price fairness regulation under the relative fairness constraints  \eqref{eq: relative constraint general form}. These findings regarding the constraint  \eqref{eq: relative constraint general form} form a theoretical cornerstone for the regulators to make more confident and convincing fairness policies.}

\subsection{\blue{Technical Contributions}}

With the problem setup, our main technical contributions are summarized as follows.

\noindent \underline{\bf Technical Contribution I:}  On the algorithm side, we design a two-stage exploration procedure (corresponding to the first two stages of the algorithm) to learn the fairness-aware optimal prices. While the tri-section and discretization techniques have been used in pricing literature (see, e.g., \cite{lei2014near} and \cite{wang2014close} respectively), we adapt them to the new fairness constraint and make sure that the constraint is not violated even during the exploration stages. The efficiency of the new exploration procedure becomes worse (due to the fairness constraint) compared to the ordinary pricing problem, and the balance between exploration and exploitation also changes. We find this new optimal tradeoff between exploration and exploitation, leading to the $\tilde{O}(T^{4/5})$ regret.

To establish this regret bound, we establish a key structural result between the price gap of the constrained optimal solution and that of the unconstrained optimal solution (see Lemma \ref{lem:exp-constrained-opt-price-gap} in Section \ref{sec:alg-exp-constrained-opt}).
At a higher level, the lemma states that the optimal fairness-aware pricing strategy should utilize the ``discrimination margin'' allowed by the fairness constraint as much as possible. This result could shed light on other fairness-aware problems.

\blue{In addition to the two-group setting, we extend our algorithms and theoretical results  to encompass multi-group setting in terms of  both price fairness  and general fairness cases (please refer to Section~\ref{sec:extension multiple group} in the supplementary materials).}

\noindent \underline{\bf Technical Contribution II:}  On the lower bound side, we show (somewhat surprisingly) that the compromise of the two-stage exploration is necessary and our $\tilde{O}(T^{4/5})$ regret is indeed \emph{optimal}.  The lower bound construction is technically quite non-trivial. In particular, we construct pairs of hard instances where 1) the demand functions are similar to their counterparts in the pair, and 2) the unconstrained optimal prices for each demand function are quite different, which leads to different fairness constraints. By contrasting these two properties, we are eventually able to derive that $\Omega(T^{4/5})$ regret has to be paid in order to properly learn the fairness constraint and the optimal fairness-aware prices.  An additional layer of challenge in our lower bound proof is that our constructed hard instances has to satisfy the standard assumptions in pricing literature (such as demands inversely correlated with prices and the law of diminishing returns) in order to become a real pricing problem instance, where in contrast the usual online learning lower bounds (such as bandits) do not have such requirements. We also note that in the existing lower bounds in dynamic pricing  \citep{besbes2009dynamic,broder2012dynamic,wang2014close}, it suffices to analyze the linear demand functions, which are relatively simple and automatically satisfy the assumptions. In our work, however, we have to construct more complicated demand functions and it requires substantially more technical efforts to make these functions also satisfy the desired assumptions. We believe that the technical tools developed in this paper considerably enrich the lower bound techniques in dynamic pricing literature.

The rest of the paper is organized as follows. In Section \ref{sec:related}, we review the relevant literature in dynamic pricing and fairness-aware online learning. We formally introduce the problem setup in Section \ref{sec:form}. Section \ref{sec:upper} provides the dynamic pricing policy under price fairness and establishes the regret upper bound. The matching lower bound is provided in Section \ref{sec:lower}. In Section \ref{sec:general}, we extend the price fairness to the general fairness measure and develop the corresponding dynamic pricing algorithm. The numerical simulation study will be provided in Section \ref{sec:numerical}, followed by 
the conclusion in Section \ref{sec:con}. Some technical proofs are relegated to the supplementary materials.

\section{Related Works} 
\label{sec:related}

There are two lines of relevant research: one on dynamic pricing and the other on fairness machine learning. This section briefly reviews related research from both lines. 

\medskip
\noindent {\bf Dynamic Pricing.}
Due to the increasing popularity of online retailing, dynamic pricing has become an active research area in the past decade. Please refer to  \cite{bitran2003overview,elmaghraby2003dynamic, den2015dynamic} for comprehensive surveys. We will only focus on the single-product pricing problem.  The seminal work by \citet{gallego1994optimal} laid out the foundation of dynamic pricing. Earlier work in dynamic pricing often assumes that demand information is known to the retailer \textit{a priori}. However, in modern retailing industries,    such as fast fashion, the underlying demand function cannot be easily estimated from historical data. This motivates a body of research on dynamic pricing with demand learning (see, e.g., \cite{araman2009dynamic, besbes2009dynamic,farias2010dynamic,broder2012dynamic,harrison2012bayesian,den2013simultaneously,keskin2014dynamic,wang2014close,lei2014near, chen2015real, Bastani:21:meta,  Wang:21:uncertainty, Miao:19} and references therein).

Along this line of research, \citet{besbes2009dynamic} first proposed a separate explore-exploit policy, which leads to sub-optimal regret of $\tilde{O}(T^{3/4})$ for  nonparametric demands and $\tilde{O}(T^{2/3})$ for parametric demands.  \citet{wang2014close}  improved this result by developing joint exploration and exploitation policies that achieve the optimal regret of $\tilde{O}(T^{1/2})$.  \citet{lei2014near} further improved the result by removing the logarithmic factor in $T$.  For more practical considerations, \citet{den2013simultaneously} proposed a controlled variance pricing policy and \citet{keskin2014dynamic} proposed a semi-myopic pricing policy for a class of parametric demand functions.  \citet{broder2012dynamic} established the lower bound of $\Omega(\sqrt{T})$ for the general dynamic pricing setting and proposed a $O(\log T)$-regret policy when demand functions satisfy a  ``well-separated'' condition.
In addition, several works proposed Bayesian policies for dynamic pricing \citep{farias2010dynamic,harrison2012bayesian}. 

As compared to the obtained regret bounds in existing dynamic pricing literature, the fairness constraint would completely change the order of the regret. Our results show that with the fairness constraint, the optimal regret becomes $\tilde{O}(T^{4/5})$. In our algorithm, the first stage of the pure exploration phase, i.e., learning the difference $|M_1(p_1^\sharp) - M_2(p_2^\sharp)|$ in the fairness constraint, alone produces the $T^{4/5}$-type regret. The popular learning-while-doing techniques (such as the Upper Confidence Bound and Thompson Sampling algorithms)  in many existing dynamic pricing and online learning papers seem not helpful in our problem to further reduce the regret. Intuitively, this is due to the fundamental difference between exploring the fairness constraint and exploiting the (near-) optimal fairness-aware pricing strategy. Such an intuition has been rigorously justified by our lower bound theorem, showing that our explore-exploit algorithm cannot be further improved in terms of the minimax regret. The separation between our regret bound and the usual $\sqrt{T}$-type regret in dynamic pricing literature also illustrates the \emph{intrinsic difficulty} from the information-theoretical perspective raised by the fairness constraint.

There are many interesting extensions of single product dynamic pricing, such as network revenue management (see, e.g., \cite{gallego1997multiproduct, ferreira2018online, chen2019network} and references therein), dynamic pricing in a changing  environment \citep{besbes2015non, keskin2016chasing},  infrequent price updates with a limited number of price changes \citep{cheung2017dynamic}, 
personalized pricing with potentially high-dimensional covariates \citep{nambiar2016dynamic,ban2017personalized,lobel2018multidimensional,chen2018nonparametric,javanmard2016dynamic,Chen:21:privacy,chen2021statistical}, \blue{pricing with reference price \citep{popescu2007dynamic,den2022dynamic}.  
It would be interesting future directions to study the fairness issue under these more general dynamic pricing setups.  Particularly, reference price, which reflect customers' price expectations based on their price history and exert a direct influence on their behavioral response, has a close relation with the fairness issue. In the absence of fairness constraints in pricing, the reference prices formed by two distinct customer groups can vary significantly, resulting in a substantial gap between them. 
When customers with higher reference prices perceive this gap, they may experience dissatisfaction and, as a result, switch to a competitor or reduce their purchasing frequency.
Implementing a fairness-aware pricing policy that incorporates a constraint on the price gap can help mitigate this issue by narrowing the reference price gap between different customer groups, ultimately enhancing customer satisfaction. For future direction, one can consider to study the interplay between fairness constraints and reference prices theoretically.
}

\medskip
\noindent \textbf{Fairness.} The topic of fairness has been extensively studied in economics and recently attracts a lot of attention from the machine learning community. There is a wide range of different definitions of ``fairness'', and many of them are originated from economics literature and are relevant to causal inference (e.g., the popular ``predictive parity'' definition \citep{Kasy:21:fairness}). Due to space limitations, we will omit detailed discussions on these definitions and only highlight a few relevant to the online learning setting. The interested readers might refer to the book \citep{barocas:19:book} and the survey \citep{Hutchinson:19} for a comprehensive review of different notions of fairness.

One classical notion of fairness is the ``individual fairness'' introduced by \citet{Dwork:12:fairness}. Considering an action $a$ that maps a context $x \in \mathcal{X}$ to a real number, the individual fairness is essentially the Lipschitz continuity of the action $a$, i.e., $|a(x_1)-a(x_2)|\leq \lambda \cdot d(x_1, x_2)$, where $d(\cdot, \cdot)$ is a certain distance metric. The notion of ``individual fairness'' can be extended to the ``fairness-across-time'' and ``fairness-in-hindsight'' for a sequence of decisions \citep{Gupta:19:individual}, which requires that actions cannot be changed too fast over time nor too much over different contextual information. However, we believe that individual fairness might not suit the pricing problem since society is more interested in protecting different customer groups. In contrast,  the fairness notation adopted in our paper can be viewed as a kind of group fairness. Other types of group-fairness have been adopted in different problems, e.g., classification \citep{Jiang:22:group}, group bandit models \citep{Baek:21:fair}, online bipartite matching \citep{Ma:21:Group}, queueing models \citep{Zhang:22:routing},  and fair allocation \citep{Cai:21:fairall}. For multi-armed bandit models, one notion of fairness introduced by \citet{Liu:17:calibrated} is the ``smooth fairness'', which requires that for two arms with similar reward distributions, the choice probabilities have to be similar.  Another popular notion of fairness in bandit literature is defined as follows: if the arm/choice $A$ has higher expected utility than another arm $B$, the arm $A$ should have a higher chance to be pulled \citep{Joseph:16:fairness}. In the auto-loan example, it means a more-qualified applicant should consistently get a higher chance of approval. A similar notion of fairness based on $Q$ functions has been adopted by \citet{Jabbari:17:fairness} in reinforcement learning. However, these notions of fairness are not designed to protect different customer groups. In pricing applications (e.g., auto-loan example), these fairness definitions may not  capture the requirement of regulations.

In recent years,  fairness has been incorporated into a wide range of operations problems. For example, \citet{chen2018why} investigated the fairness of service priority and price discount in a shared last-mile transportation system. \citet{Bateni:16} studied fair resource allocation in the ads market and \citet{Chen:22:fairer} studied fair online resource allocation based on linear programming. It adopts weighted proportional fairness proposed by \citet{Bertsimas:12}. \citet{Balseiro:21:regularized} introduced the regularized online allocation problem and proposed a primal-dual approach to handle the max-min fairness regularizer.  \citet{Manshadi:22:fair} studied the problem of fair dynamic rationing and investigated the expected minimum fill rate (ex-post fairness) and the minimum expected fill rate (ex-ante fairness). \citet{Kandasamy:20:online} studied online demand estimation and allocation under the max-min fairness with applications to cloud computing.
In the pricing application considered in this paper, research has been devoted to game-theoretical models with fairness constraint in duopoly markets (see, e.g., \cite{Li:16:behavior} and references therein). In contrast, we consider a monopoly market and thus do not adopt game-theoretical modeling. For the static pricing problem with known demand functions, \citet{Kallus:21} formulated a multi-objective optimization problem, which takes the price parity and long-run welfare into consideration.

\citet{cohen2022price} proposed a ``group'' fairness notion designed for the pricing problem (e.g., fairness in price and demand) and investigated the impact of these types of fairness on social welfare. Our work is built on these fairness notions and extends them to dynamic pricing with nonparametric demand learning.  As we have explained, this paper does not try to argue the fairness notion in \citet{cohen2022price} is the most suitable notion. As a technical-oriented paper, we choose this specific notion mainly because this notion has been well-received in the OM community and leads to  interesting globally constrained online learning formulations and somewhat surprising regret behavior.
We also note that the work by  \cite{cohen2022price} has already studied the impact of the fairness constraint in the static problem, which is measured by the revenue gap between the static pricing problem with and without fairness constraint.  For example, Proposition 2 in \cite{cohen2022price} characterizes the revenue loss as a function of $\lambda$.  Therefore, we will focus on the revenue gap between dynamic and static problems both under the fairness constraint.

A very recent work by \citet{Cohen:21:dynamic} studies the learning-while-doing problem for dynamic pricing under fairness.  The key difference is that they defined the fairness constraint as an absolute upper bound of the price gap between different groups (i.e., $|p_1^{(t)}-p_2^{(t)}|\leq C$ for some fixed constant $C$), while we consider a relative price gap in \eqref{eq:price_fairness} (\blue{please refer to Section~\ref{sec: comparision} for a through comparison between the absolute constraint and our relative constraint}).  In addition, they studied a simple parametric demand model in the form of generalized linear model in price. In contrast, we allow a fully nonparametric demand model without any parametric assumption of $d_i(\cdot)$.  

\label{page:fixed}

More importantly,  due to the need of learning and staying consistent with the unknown constraint over the entire time horizon, the relative gap constraint in our paper leads to different behavior in terms of regret, and this phenomenon has its independent interest in online learning literature. Under the absolute price gap with parametric demand models, the work by \citet{Cohen:21:dynamic} achieves a standard $\sqrt{T}$-type regret. On the other hand, with the relative gap constraint, we need to learn the unknown unconstrained optimal prices while making price decisions. 
Thus, technically, using a relative gap makes the problem fundamentally more challenging from an information-theoretical perspective and leads to the minimax lower bound $\Omega(T^{4/5})$ of the regret.\footnote{Without the relative gap constraint, the standard dynamic pricing problem with a nonparametric demand model only incurs a $\sqrt{T}$-type regret \citep{wang2014close}. This contrast shows that learning and obeying the relative gap constraint is the key reason of the higher regret.}We also note that \cite{Cohen:21:dynamic} considered absolute ``time-fairness''. We would like to leave further exploration of this constraint as a future direction.

Finally,  we assume the protected group information is available to the seller. Recent work by \citet{Kallus:21:fair} studied how to assess fairness when the protected group membership is not observed in the data. It would also be an interesting direction to extend our work to the setting with hidden group information.

\section{Problem Formulation}
\label{sec:form}

We consider a dynamic discriminatory pricing problem with fairness constraints. Suppose that there are $T$ selling periods in total, and two groups of customers (labeled by $1$ and $2$). At each selling period $t=1,2,\dots, T$, the seller offers a single product, with a marginal cost $c \geq 0$, to two groups of customers. The seller also decides a price $p_{i}^{(t)} \in [\underline p, \overline p]$ for each group $i \in \{1, 2\}$, where $[\underline p, \overline p]$ is the known feasible price range. We assume that each group $i \in \{1, 2\}$ of customers has its own demand function $d_i(\cdot): [\underline p, \overline p] \rightarrow [0,1] $, where $d_i(p)$ is the expected demand from group $i$ when offered price $p$. The demand functions $\{d_i(\cdot)\}$ are unknown to the seller beforehand. 

When offering the product to each group of customers, we denote the realized demand from group $i$ by $D_{i}^{(t)} \in [0, 1]$ (up to normalization), which is essentially a random variable satisfying $\mathbb E\left[D_{i}^{(t)}~\big|~ p_{i}^{(t)}, \mathcal{F}_{t-1}\right] = d_i(p_{i})$ and $\mathcal{F}_{t-1}$ is the natural filtration up to selling period $t-1$. For example, when $D_i(t)$ follows a Bernoulli distribution with mean $d_i(p_i^{(t)})$, the binary value of $D_i(t)$ represents whether the  customer group $i$ makes a purchase (i.e., $D_i(t)=1$) or not (i.e., $D_i(t)=0$). By observing $D_{i}^{(t)}$, the seller earns the profit  $\sum_{i \in \{1, 2\}} (p_{i}^{(t)}-c) D_{i}^{(t)}$ at the $t$-th selling period.

If the seller has known the demand functions $\{d_i(\cdot)\}$ beforehand, and is not subject to any fairness constraint, her optimal prices for two groups are the following \emph{unconstrained clairvoyant} solutions:
\begin{equation}\label{eq:rev}
p_i^{\sharp} = \argmax_{p \in [\underline p, \overline p]} \; R_i(p):= (p - c) d_i(p), \qquad \forall i \in \{1, 2\},
\end{equation}
where $R_i(p)$ is the expected single-period revenue function for group $i$. Following the classical pricing literature \citep{gallego1997multiproduct}, under the one-to-one correspondence between the price and demand and other regularity conditions, we could also express the price as a function of demand for each group (i.e., $p_i(d)$ for $i \in \{1,2\}$). This enables us to define the so-called \emph{revenue-demand function}, which expresses the revenue as a function of demand instead of price: $R_i(d):=(p_i(d)-c)d$. For commonly used demand functions (e.g., linear, exponential, power, and logit), the revenue-demand function is \emph{concave} in the demand. The concavity assumption is widely assumed in pricing literature (see, e.g., \cite{Jasin:14}) and is critical in designing our policy. We also note that the revenue function is not concave in \emph{price} for many examples (e.g., logit demand).

Now, we are ready to formally introduce the fairness constraint. Let $M_i(p)$ denotes a fairness measure of interest for group $i$ at price $p$, where $M_i$ can be any Lipschitz function. When $M_i(p)=p$, it reduces to the price fairness in \cite{cohen2022price}. When $M_i(p) = d_i(p)$, it corresponds to the demand fairness in \cite{cohen2022price}. For any given fairness measure, the hard fairness constraint requires that 
\begin{equation}\label{eq:hard_constarint}
\left|M_1(p_{1}^{(t)}) - M_2(p_{2}^{(t)})\right| \leq \lambda \left|M_1(p_1^\sharp) - M_2(p_2^\sharp)\right|, \qquad \forall t \in \{1, 2, 3, \dots, T\},
\end{equation}
where $\lambda > 0$ is the parameter for the \emph{fairness level} that is selected by the seller to meet internal goals or satisfy regulatory requirements. The smaller $\lambda$ is, the more strict fairness constraint the seller has to meet. We also note that the parameter $\lambda$ is not a tuning-parameter of the algorithm. Instead, this is the fundamental fairness level that should be pre-determined by the seller either based on a certain internal/external regulation or on how much revenue the seller is willing to sacrifice (see Eq.~\eqref{eq:trade-off}).

With the fairness measure in place, let $\{p_1^*, p_2^*\}$ denote the \emph{fairness-aware clairvoyant} solution, i.e., the optimal solution to the following static optimization problem,
\begin{align}\label{eq:static_fair}
\max_{p_1, p_2 \in [\underline p, \overline p]} \qquad & R_1(p_1) + R_2(p_2),\\
\text{subject to} \qquad & \left|M_1(p_{1}) - M_2(p_{2})\right| \leq \lambda \left|M_1(p_1^\sharp) - M_2(p_2^\sharp)\right|. \nonumber
\end{align}
For the ease of notation, we omit the dependency of $\{p_1^*, p_2^*\}$ on $\lambda$. We also note that the work by \citet{cohen2022price} quantifies the tradeoff between the strictness of the fairness constraint and the overall revenue in the static problem. In particular, it shows that for linear demands and price or demand fairness,
\begin{equation}\label{eq:trade-off}
R_1(p_1^{\sharp})+R_2(p_2^{\sharp})-(R_1(p_1^{*})+R_2(p_2^{*}))=O((1-\lambda)^2).
\end{equation}
In other words, the revenue loss due to imposing a stronger fairness constraint (as $\lambda$ decreases to zero) grows at the rate of $O((1-\lambda)^2)$. As this tradeoff has already been explored in \cite{cohen2022price}, the main goal of our paper is on developing an online policy that can guarantee the fairness through the entire time horizon.

In the learning-while-doing setting where the seller does not know the demand beforehand, she has to learn demand functions during selling periods, and maximize her total revenue, while in the meantime,  obeying the fairness constraint. Equivalently, the seller would like to minimize the \emph{regret}, which is  the difference between her expected total revenue and the fairness-aware clairvoyant solution:
\begin{equation}
\label{eq:def_reg}
\mathrm{Reg}_T := \mathbb E \sum_{t=1}^T\left[ R_1(p_1^*) + R_2(p_2^*) - R_1(p_{1}^{(t)}) - R_2(p_{2}^{(t)}) \right],
\end{equation}
where $p_1^*$ and $p_2^*$ are the fairness-aware clairvoyant solution defined in Eq.~\eqref{eq:static_fair}.

In this paper, we will first focus on price fairness (i.e., $M_i(p) = p$) and establish matching regret upper and lower bounds for price-fairness-aware dynamic pricing algorithms. We will then extend our algorithm to general fairness measure function $\{M_i(p)\}$. However, in practical scenarios where $M_i(p)$ is not accessible to the seller beforehand, only the noisy observation $M_i^{(t)}$ with $\mathbb{E}\left[M_i^{(t)} ~|~ p_i^{(t)}, \mathcal{F}_{t-1}\right] = M_i(p_i^{(t)})$ (for $i \in \{1, 2\}$) is revealed after the seller's pricing decisions during selling period $t$. A natural example is the demand fairness, where $M_i(p) = d_i(p)$, and the seller could only observe a noisy demand realization at the offered price. In this case, it is impossible for the seller to satisfy the hard constraint in Eq.~\eqref{eq:hard_constarint} at first a few selling periods (as there is a limited number of observations of $M_i$ available). To this end, we propose the ``soft fairness constraint'' and add the soft fairness constraint as a penalty term to the regret minimization problem. In particular, the \emph{penalized regret} incurred at time $t$ takes the following form,
\begin{align}
\left[R_1(p_1^*) + R_2(p_2^*) - R_1(p_1^{(t)}) - R_2(p_2^{(t)})\right] + \gamma \max\left(\left| M_1(p_1^{(t)}) - M_2(p_2^{(t)})\right| - \lambda \left|M_1(p_1^\sharp) - M_2(p_2^\sharp)\right|, 0\right), \label{eq:general-fairness-soft-constraint}
\end{align}
where the first term is the standard regret, the second term is the penalty term for violating the fairness constraint, and $\gamma$ is a pre-defined parameter to balance between the regret and the fairness constraint and assumed to be $O(1)$. Subsequently, for general fairness measure, the seller aims to minimize the following cumulative penalized regret:
\begin{align*}
\mathrm{Reg}_T^{\mathrm{soft}} &:= \mathbb E \sum_{t=1}^T\Big\{ \left[R_1(p_1^*) + R_2(p_2^*) - R_1(p_1^{(t)}) - R_2(p_2^{(t)})\right]\\
& \qquad\qquad\qquad\qquad + \gamma \max\left(\left| M_1(p_1^{(t)}) - M_2(p_2^{(t)}\right| - \lambda \left|M_1(p_1^\sharp) - M_2(p_2^\sharp)\right|, 0\right) \Big\}.
\end{align*}

Throughout the paper, we will make the following standard assumptions on demand functions and fairness measure: 

\begin{assumption}\label{assumption:1}

\begin{enumerate}
\item \label{item:assumption-1-lipschitz} The demand-price functions are monotonically decreasing and injective Lipschitz, i.e., there exists a constant $K \geq 1$ such that for each group $i \in \{1, 2\}$, it holds that 
\[
\frac{1}{K} |p-p'| \leq |d_i(p) - d_i(p')| \leq K |p-p'|, \qquad \forall p,p' \in [\underline{p}, \overline{p}] .
\]

\item \label{item:assumption-1-strong-concavity} The revenue-demand functions are strongly concave, i.e., there exists a constant $C > 0$ such that for each group $i \in \{1, 2\}$, it holds that
\[
R_i(\tau d + (1-\tau) d') \geq \tau R_i(d) + (1-\tau) R_i(d') + \frac{1}{2}C \tau(1-\tau)(d-d')^2 , \qquad \forall d, d', \tau \in [0, 1].
\]

\item \label{item:assumption-1-lipchtiz-M} The fairness measures are Lipschitz, i.e., there exists a constant $K'$ such that for each group $i \in \{1, 2 \}$, it holds that:
\[
|M_i(p) - M_i(p')| \leq K' |p-p'|, \qquad \forall p, p' \in [\underline{p}, \overline{p}]
\]

\item \label{item:assumption-1-M-UB} There exits a constant $\overline{M} \geq 1$ such that the noisy observation $M_i^{(t)} \in [0, \overline{M}]$ for every selling period $t$ and customer group $i \in \{1, 2\}$.
\end{enumerate}
\end{assumption}

Assumptions~\ref{assumption:1}\ref{item:assumption-1-lipschitz} and \ref{assumption:1}\ref{item:assumption-1-strong-concavity} are  rather standard assumptions made for the demand-price and revenue-demand functions in pricing literature (see, e.g., \cite{wang2014close} and references therein). On the other hand, the fairness measure functions $M_i(p)$ are first studied in the context of dynamic pricing. Assumptions~\ref{assumption:1}\ref{item:assumption-1-lipchtiz-M} and \ref{assumption:1}\ref{item:assumption-1-M-UB} assert necessary and mild regularity conditions on these functions and their noisy realizations.

In the rest of this paper, we will first investigate the optimal regret rate that can be achieved under the setting of price fairness. Once we obtain a clear understanding about price fairness, we will proceed to study more general fairness settings.

\section{Dynamic Pricing Policy under Price Fairness }
\label{sec:upper}

Starting with the price fairness (i.e., $M_i(p)=p$), we develop the fairness-aware pricing algorithm and establish its theoretical property in terms of regret.

Our algorithm (see Algorithm \ref{alg:price-fairness-main}) runs in the explore-and-exploit scheme. In the exploration phase, the algorithm contains two stages. The first stage \blue{(Stage I)} separately estimates the optimal prices for two groups without any fairness constraint. Using the estimates as input, the second stage \blue{(Stage II)} learns the (approximately) optimal prices under the fairness constraint \blue{by a discretization method}. Then the algorithm enters the exploitation phase \blue{(Stage III)} , and uses the learned prices for each group to optimize the overall revenue.  The algorithm is presented in Algorithm \ref{alg:price-fairness-main}. Note that the algorithm will terminate whenever the time horizon $T$ is reached (which may happen during the exploration stage).

\begin{algorithm}[!t]
 \label{alg:price-fairness-main}
\caption{Fairness-aware Dynamic Pricing with Demand Learning}		
For each group $i \in \{1, 2\}$, run {\sc ExploreUnconstrainedOPT} (Algorithm~\ref{alg:exp-unconstrainted-opt}) separately with the input $z = i$, and obtain the estimate of the optimal price without fairness constraint $\hat{p}_i^\sharp$.

Given $\hat{p}_1^\sharp$ and $\hat{p}_2^\sharp$, run {\sc ExploreConstrainedOPT} (Algorithm~\ref{alg:exp-constrained-opt}), and estimate the optimal price under the fairness constraint $(\hat{p}_1^*, \hat{p}_2^*)$.

For each of the remaining selling periods, offer the price $\hat{p}_i^*$ to each customer group $i \in \{1, 2\}$.
\end{algorithm}

\blue{After designing the algorithm at a high level, it remains to decide the critical parameters such as the number of periods to use in Stage I and the number of checkpoints to set in Stage II. We address these issues below and explain intuitively why the algorithm achieves the $T^{4/5}$-type regret.}

\blue {Given that the optimal convergence rate of $|\hat p_{(t)}^{\sharp}-p^\sharp|$ is $\Theta(t^{-1/4})$,\footnote{Optimal convergence rate of $|p_{(t)}^{\sharp}-p^\sharp|$  can be deduced from the optimal regret of stochastic bandit optimization \citep{shamir2013complexity} using the properties of $\mathcal R(p)$.} suppose we use $a(T)$ periods to collect samples Stage I would incur a regret  of $\tilde{\Theta}(a(T))$ and lead to unconstrained-optimal-price-gap estimator converging at $\tilde{\Theta}(a^{-1/4}(T))$ rate. 
Note that the estimate error $\tilde{\Theta}(a^{-1/4}(T))$ will propagate into the following decision period, leading to roughly $\tilde{\Theta}(a^{-1/4}(T))$ rate per period. 
As a result, setting more checkpoints than $\tilde{\Theta}(a^{1/4}(T))$ won't improve the total regret. Thus in Stage II, we only need to discretize the price range to get $\tilde{\Theta}(a^{1/4}(T))$ checkpoints. 
In this way, Stage II and Stage III together incur a regret of  $\tilde{\Theta} ((T-a(T))a^{-1/4}(T))+\tilde{\Theta}(a^{1/12}(T)(T-a(T))^{2/3})$, where the first term is due to the $\tilde{\Theta}(a^{-1/4}(T))$ error of the discretization, and the second term is  the regret of the Explore-Then-Commit algorithm over $\tilde{\Theta}(T-a(T))$ periods  with the arm number $\tilde{\Theta}(a^{1/4}(T))$.
Therefore the total regret of all three stages is $\tilde{\Theta}(a(T))+ \tilde{\Theta} ((T-a(T))a^{-1/4}(T)) +\tilde{\Theta}(a^{1/12}(T)(T-a(T))^{2/3})$. Since the first two terms dominate the total regret, trading off these two terms will lead to an optimal choice of $a(T) =
	\tilde{\Theta}(T^{4/5})$ and the $\tilde{\Theta}(T^{4/5})$ regret.}

Now we describe two subroutines used in exploration phase: {\sc ExploreUnconstrainedOPT} and {\sc ExploreConstrainedOPT}. We note that according to our theoretical results in Theorems \ref{thm:exp-unconstrained-opt} and \ref{thm:exp-constrained-opt}, {\sc ExploreUnconstrainedOPT} and {\sc ExploreConstrainedOPT} will only run in $\tilde{O}(T^{4/5})$ and $\tilde{O}(T^{3/5})$ time periods, respectively. Therefore, the estimated prices $\hat{p}_1^*$ and $\hat{p}_2^*$ will be offered for the most time periods in the entire selling horizon of length $T$.

\paragraph{\underline{The {\sc ExploreUnconstrainedOPT} subroutine.}} Algorithm~\ref{alg:exp-unconstrainted-opt} takes the group index $z \in \{1, 2\}$ as input, and estimates the unconstrained clairvoyant solution $\hat{p}^\sharp_z$ for group $z$ (i.e., without the fairness constraint).

Algorithm~\ref{alg:exp-unconstrainted-opt} runs in a trisection fashion. The algorithm keeps an interval $[p_L, p_R]$ and shrinks the interval by a factor of $2/3$ during each iteration while keeping the estimation target $p^\sharp_z$ within the interval with high probability. The iterations are indexed by the integer $r$, and during each iteration, the two trisection prices $p_{m_1}$ and $p_{m_2}$ are selected. For either trisection price, both customer groups are offered the price (so that the price fairness is always satisfied) for a carefully chosen number of selling periods (as in Lines~\ref{line:alg-exploreunconstrainedopt-5} and \ref{line:alg-exploreunconstrainedopt-6}, where $K$ and $C$ are defined in Assumption~\ref{assumption:1}), and the estimated demand from group $z$ is calculated. Note that in practice, one may not have the access to the exact value $K$ and $C$, and the algorithm may use a large enough estimate for $K$ and $1/C$, or use $\mathrm{poly}\log T$ and $1/\mathrm{poly}\log T$ instead. In the latter case, the theoretical analysis will work through for sufficiently large $T$ and the regret remains at the same order up to $\mathrm{poly}\log T$ factors. Finally, in Line~\ref{line:alg-exploreunconstrainedopt-7} of the algorithm, we construct the new (shorter) interval based on estimated demands corresponding to the trisection prices.

\begin{algorithm}[!t]
\caption{{\sc ExploreUnconstrainedOPT}\label{alg:exp-unconstrainted-opt}}		
\SetKwInOut{Input}{Input}
\SetKwInOut{Output}{Output}
\SetKw{True}{TRUE}
\Input{the customer group index $z \in \{1, 2\}$}
\Output{the estimated unconstrained optimal price $\hat{p}^\sharp$ for group $z$}
$p_L \leftarrow \underline{p}$, $p_R \leftarrow \overline{p}$, $r \leftarrow 0$\;
\While{$|p_L - p_R| > 4 T^{-1/5}$ \label{line:alg-exploreunconstrainedopt-2}}{
	$r \leftarrow r + 1$\;
	$p_{m_1} \leftarrow \frac{2}{3}p_L + \frac{1}{3}p_R$, $p_{m_2} \leftarrow \frac{1}{3}p_L + \frac{2}{3}p_R$\;
		
	Offer price $p_{m_1}$ to \emph{both customer groups} for $\frac{25 K^4 \overline{p}^2}{C^2 } T^{4/5} \ln T$ selling periods and denote the average demand from customer group $z$ by $\hat{d}_{m_1}$\; \label{line:alg-exploreunconstrainedopt-5}
	Offer price $p_{m_2}$ to \emph{both customer groups} for $\frac{25 K^4 \overline{p}^2}{C^2 } T^{4/5} \ln T$ selling periods and denote the average demand from customer group $z$ by $\hat{d}_{m_2}$\; \label{line:alg-exploreunconstrainedopt-6}
	{\bf if} $\hat{d}_{m_1}(p_{m_1}-c) > \hat{d}_{m_2}(p_{m_2}-c)$ {\bf then} $p_R \leftarrow p_{m_2}$; {\bf else} $p_L \leftarrow p_{m_1}$\; \label{line:alg-exploreunconstrainedopt-7}
}

\Return $\hat{p}^\sharp =  \frac{1}{2}(p_L + p_R)$\;
\end{algorithm}

Concretely, for Algorithm \ref{alg:exp-unconstrainted-opt}, we prove the following upper bounds on the number of selling periods used by the algorithm and its estimation error.
    
\begin{theorem} \label{thm:exp-unconstrained-opt}
For any input $z \in \{1, 2\}$, Algorithm~\ref{alg:exp-unconstrainted-opt} uses at most  $O(\frac{K^4 \overline{p}^2}{C^2}  T^\frac{4}{5}\log T \log (\overline{p}  T))$ selling periods and satisfies the fairness constraint during each period. Let $\hat{p}_z^\sharp$ be the output of the procedure. With probability $(1 - O(T^{-2} \log(\overline{p}  T)))$, it holds that $|\hat{p}_z^\sharp - p_z^\sharp| \leq   4 T^{-\frac{1}{5}}$. Here, only universal constants are hidden in the $O(\cdot)$ notations.
\end{theorem}

For the ease of presentation, the proof of Theorem \ref{thm:exp-unconstrained-opt} will be provided in later in Section \ref{sec:alg:exp-unconstrainted-opt}.

\paragraph{\underline{The {\sc ExploreConstrainedOPT} subroutine.}} Suppose we have run {\sc ExploreUnconstrainedOPT} for each $z \in \{1, 2\}$ and obtained both $\hat{p}_1^\sharp$ and $\hat{p}_2^\sharp$. {\sc ExploreConstrainedOPT} estimates the constrained (i.e., fairness-aware) clairvoyant solution for both groups. The pseudo-code of the procedure is presented in Algorithm~\ref{alg:exp-constrained-opt}. In this procedure, we assume without loss of generality that $\hat{p}_1^\sharp \leq \hat{p}_2^\sharp$ since otherwise we can always switch the labels of the two customer groups.

To investigate the property of this algorithm, we establish a key  relation between the gap of the constrained optimal prices and that of the unconstrained optimal prices (see Lemma~\ref{lem:exp-constrained-opt-price-gap} in Section \ref{sec:alg-exp-constrained-opt}). In particular,  Lemma~\ref{lem:exp-constrained-opt-price-gap} will show that the optimal offline clairvoyant fairness-aware pricing solution would fully exploit the fairness constraint so that Eq.~\eqref{eq:static_fair} becomes tight, i.e., $|p_1^*-p_2^*|=\lambda |p_1^\sharp-p_2^\sharp|$. This key relationship is proved by a monotonicity argument for the optimal total revenue as a function of the discrimination level (measured by the ratio between the price gap and that of the unconstrained optimal solution).

Using this key relationship, Algorithm~\ref{alg:exp-constrained-opt} first sets $\xi$ so that $\xi$ is a lower estimate of the unconstrained optimal price gap (i.e.,  $|\hat{p}_1^\sharp - \hat{p}_2^\sharp|$) and $\lambda\xi$ is a lower estimate of the constrained optimal price gap (i.e., $\lambda|\hat{p}_1^\sharp - \hat{p}_2^\sharp|$). The algorithm then tests the mean price $(p_1^* + p_2^*)/2$ using the discretization technique. More specifically, the algorithm identifies a grid of possible mean prices $\{\ell_1, \ell_2, \dots, \ell_J\}$. For each price checkpoint $\ell_j$, the algorithm would try $\ell_j - \frac{\lambda \xi}{2}$ and $\ell_j + \frac{\lambda \xi}{2}$ as the fairness-aware prices for the two customer groups (so that the price gap $\lambda \xi$ is bounded by $\lambda |\hat{p}_1^\sharp - \hat{p}_2^\sharp|$), and estimate the corresponding demands and revenue. The algorithm finally reports the optimal prices among these price checkpoints based on the estimated revenue.

Formally, we state the following guarantee for Algorithm~\ref{alg:exp-constrained-opt}, and its proof will be relegated to Section \ref{sec:alg-exp-constrained-opt}.

\begin{algorithm}[!t]
\caption{\sc ExploreConstrainedOPT\label{alg:exp-constrained-opt}}		
\SetKwInOut{Input}{Input}
\SetKwInOut{Output}{Output}
\Input{the estimated unconstrained optimal prices $\hat{p}_1^\sharp$ and $\hat{p}_2^\sharp$,  assuming that $\hat{p}_1^\sharp \leq \hat{p}_2^\sharp$ (without loss of generality)}
\Output{the estimated constrained optimal prices $\hat{p}_1^*$ and  $\hat{p}_2^*$}

$\xi \leftarrow \max\{ |\hat{p}_1^\sharp - \hat{p}_2^\sharp| - 8  T^{-1/5}, 0\}$\; 

$J \leftarrow \lceil (\overline{p}-\underline{p}) T^\frac{1}{5}\rceil$ and create $J$ price checkpoints $\ell_1, \ell_2, \dots, \ell_J$ where $\ell_j \leftarrow \underline{p} + \frac{j}{J} (\overline{p} - \underline{p})$\; 

\For{each $\ell_j$}{
    Repeat the following offerings for $6 T^{2/5} \ln T$ selling periods: offer price $p_1(j) \leftarrow \max\{\underline{p}, \ell_j - \frac{\lambda \xi}{2}\}$ to customer group $1$ and price $p_2(j) \leftarrow \min\{\overline{p}, \ell_j + \frac{\lambda \xi}{2}\}$ to customer group $2$\; 
    Denote the average demand from customer group $i \in \{1, 2\}$ by $\hat{d}_i(j)$\;
    $\hat{R}(j) \leftarrow \hat{d}_1(j) (p_1(j) - c) + \hat{d}_2(j) (p_2(j) - c)$\;
}
$j^* \leftarrow \argmax_{j \in \{1, 2, \dots, J\}} \{\hat{R}(j)\}$\;
\Return $\hat{p}^*_1 \leftarrow p_1(j^*)$ and $\hat{p}^*_2 \leftarrow p_2(j^*)$\;
\end{algorithm}

\begin{theorem} \label{thm:exp-constrained-opt}
Suppose that $|\hat{p}_1^\sharp - p_1^\sharp| \leq  4 T^{-1/5}$ and $|\hat{p}_2^\sharp - p_2^\sharp| \leq 4  T^{-1/5}$. Algorithm~\ref{alg:exp-constrained-opt} uses at most  $O(\overline{p}  T^{3/5} \ln T)$ selling periods and satisfies the price fairness constraint during each selling period. With probability $(1- O(\overline{p}  T^{-2}))$, the procedure returns a pair of price $(\hat{p}_1^*, \hat{p}_2^*)$ such that $|\hat{p}_1^* - \hat{p}_2^*| \leq \lambda |p_1^\sharp - p_2^\sharp|$ and 
\[
R_1(p_1^*) + R_2(p_2^*) - R_1(\hat{p}_1^*) - R_2(\hat{p}_2^*) \leq O( K \overline{p} T^{-1/5}).
\]
Here, only universal constants are hidden in the $O(\cdot)$ notations.
\end{theorem}

Based on Theorems \ref{thm:exp-unconstrained-opt} and \ref{thm:exp-constrained-opt}, we are ready to state the regret bound of our main algorithm. 

\begin{theorem}\label{thm:main_upper}
With probability $(1 - O(T^{-1}))$, Algorithm~\ref{alg:price-fairness-main} satisfies the fairness constraint and its regret is at most $O(T^{\frac45}\log^2 T)$. Here, the $O(\cdot)$ notation only hides the polynomial dependence on $\overline{p}$, $K$ and $1/C$.
\end{theorem}

\blue{Suppose there 
 is a constant penalty for the constraint violation.  According to Theorem~\ref{thm:main_upper} the fairness constraints are violated with a small probability  $O(T^{-1})$,  which will only result in a regret of $T\times O(T^{-1}) = O(1)$. }

	\begin{figure}[!t]
\centering
\includegraphics[width =0.5\textwidth]{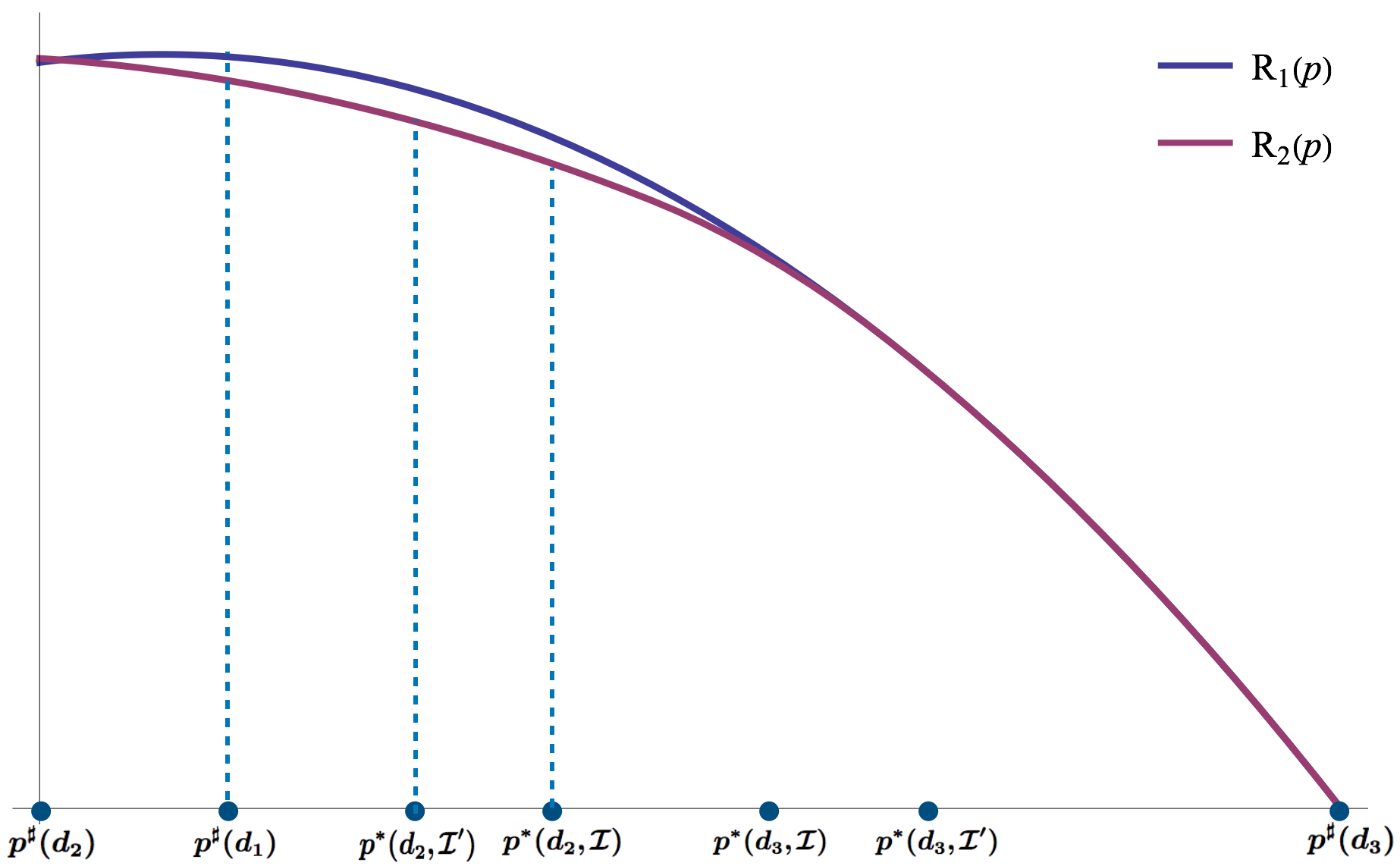}
\caption{\blue{Graphic illustration of hard instances $R_1(p)$, $R_2(p)$ \& relative positions of ${p^{\sharp}(d_i)}, i\in\{1,2,3\}$ and  $p^*(d_i; \mathcal J)$, $i\in\{2,3\}$, $\mathcal J \in \{\mathcal I, \mathcal I'\}$.     }}
\label{fig:hard instance}
\end{figure}	
\section{Lower Bound}
\label{sec:lower}

When $\lambda \in (\epsilon, 1 - \epsilon)$ where $\epsilon > 0$ is a positive constant, we will show that the expected regret of a fairness-aware algorithm is at least $\Omega(T^{4/5})$.  Formally, we prove the following lower bound theorem.

\begin{theorem} \label{theorem:lb}
Suppose that $\pi$ is an online pricing algorithm that satisfies the price fairness constraint with probability at least $0.9$ for any problem instance. Then for any $\lambda \in (\epsilon, 1-\epsilon)$ and $T \geq \epsilon^{-C_\mathrm{LB}}$ (where $C_\mathrm{LB} > 0$ is a universal constant), there exists a pricing instance such that the expected regret of $\pi$ is at least $\frac{1}{160} \epsilon^2 T^{4/5}$.
\end{theorem}

\begin{remark}
 \blue{ By imposing a constant penalty for the constraint violation, any policy would need to satisfy the constraint  with a probability of at least  $1-o(1)$ to get a sublinear regret.    Therefore, the assumption that $\pi$ satisfies the price fairness constraint with probability at least $0.9$ for any problem instance  could be seen as the ``reasonable policy assumption''. }
\end{remark}

To prove such a lower bound, we need to construct hard instances for any fairness-aware algorithm. We first set $\underline p = 1$, $\overline p = 2$, and $c = 0$. For any two expected demand rate functions $d, d' : [\underline p, \overline p] \to [0, 1]$, we define a problem instance $\mathcal I(d, d')$ as follows: at each time step, when offered a price $p \in [\underline p, \overline p]$, the stochastic demand from group $1$ follows the Bernoulli distribution $\mathrm{Ber}(d(p))$, and the stochastic demand from group $2$ follows  $\mathrm{Ber}(d'(p))$. 

\medskip
\noindent{\bf \underline{Construction of the hard instances.}} We now construct two problem instances $\mathcal I = \mathcal I(d_1, d_3)$ and $\mathcal I' = \mathcal I(d_2, d_3)$, where $d_i(p) =  R_i(p)/ p$ for $i \in \{1, 2, 3\}$, and we define $R_i$'s as follows.
\begin{align*}
R_1(p) &= \frac{1}{4} - \frac{1}{A}(p - 1 - \frac{\sqrt{h}}{4})^2, \qquad\qquad \quad ~ p \in [1, 2],\\
R_2(p) &=
		\left\{ \begin{aligned}
			& \frac{1}{4} - \frac{1}{2A}(p - 1 + \frac{\sqrt{h}}{4})^2, & & p \in [1, 1 + \frac{5\sqrt{h}}{4}), \\
			& \frac{1}{4} - \frac{3}{2A}(p - 1 - \frac{3\sqrt{h}}{4})^2 - \frac{3h}{4A}, & & p \in [1 + \frac{5\sqrt{h}}{4}, 1 + \frac{7\sqrt{h}}{4}), \\
			& \frac{1}{4} - \frac{1}{A}(p - 1 - \frac{\sqrt{h}}{4})^2, & & p \in [1 + \frac{7\sqrt{h}}{4},2],\\
		\end{aligned}\right.			\\
R_3(p) &= \frac{1}{8} - \frac{1}{A}(p - 2)^2, \qquad \qquad \qquad\quad~~~p \in [1,2].
\end{align*}
Here, $A \geq 1$ is a large enough universal constant and $h \geq 0$ depends on $T$, both of which will be chosen later.

\blue{For any problem instance $\mathcal J \in \{\mathcal I, \mathcal I'\}$, and any demand function $d$ that is employed by a customer group in $\mathcal J$, we denote by $p^*(d; \mathcal J)$ the price for the customer group in the optimal fairness-aware clairvoyant solution to $\mathcal J$.
  Since the unconstrained-optimal-price-gap is smaller in instance $\mathcal{I}$, as depicted in Figure~\ref{fig:hard instance}, mistaking $\mathcal{I'}$ for $\mathcal{I}$ would result in a tighter constraint. Consequently, instance $\mathcal{I'}$ would suffer a revenue loss of $R_2(p^*(d_2;\mathcal{I'}))+R_3(p^*(d_3;\mathcal{I'})) - R_2(p^*(d_2;\mathcal{I}))-R_3(p^*(d_3;\mathcal{I}))$.  The two instances are closely related, with the only difference occurring in the interval $[1,1 + \frac{7\sqrt{h}}{4}]$. As a result, to avoid misidentification it is necessary to explore more prices within $[1,1 + \frac{7\sqrt{h}}{4}]$. However, as it will be demonstrated in Lemma~\ref{lemma:lb-regret-fairness-aware}, prices charged in this interval will also harm the total revenue. Intuitively, we establish the $\Omega(T^{4/5})$ lower bound by carefully balancing the above trade-off.}

We verify the following properties of the constructed demand and profit rate functions.
\begin{lemma} \label{lemma:lb-demand-profit-properties}
When $h \in (0, 0.01)$ and $20 \leq A \leq 30$, the following statements hold.\footnote{We do not make efforts to optimize the  constants such as $-\frac{1}{40}$ (in Item~\ref{item:lb-dpp-c}), $\frac{1}{4}$ (in Item~\ref{item:lb-dpp-d}), and $5/3$ (in Item \ref{item:lb-dpp-e}). The proof would still go through (with minor modifications) if they were different values (as long as they remain positive/negative respectively).}
\begin{enumerate}
\item \label{item:lb-dpp-a} $d_i(p) \in [1/20, 1/4]$ for all $i \in \{1, 2, 3\}$ and $p \in [1, 2]$.
\item \label{item:lb-dpp-b}  $d_i(p)$ and $R_i(p)$ are continuously differentiable functions for all $i \in \{1, 2, 3\}$ and $p \in [1, 2]$.
\item \label{item:lb-dpp-c}  For each $p \in [1,2], i \in \{1, 2, 3\}$, $\frac{\partial d_i}{\partial p} < -\frac{1}{40} < 0$, and $R_i$ is strongly concave as a function of $d_i$. 
\item \label{item:lb-dpp-d} For each $p \in [1, 1 + \frac{7\sqrt{h}}{4}]$, it holds that $|d_1(p) - d_2(p)| \leq \frac{h}{4A}$.
\item \label{item:lb-dpp-e} For each $p \in [1, 1 + \frac{7\sqrt{h}}{4}]$, it holds that $D_{\mathrm{KL}}(\mathrm{Ber}(d_1(p))\| \mathrm{Ber}(d_2(p))) \leq 5h^2/3A^2$.
\item \label{item:lb-dpp-f} For any demand rate function $d(p)$ defined on $p \in [1, 2]$, let $p^\sharp(d) = \argmax_{p \in [1, 2]} \{p \cdot d(p)\}$ be the unconstrained clairvoyant solution; we have that $p^\sharp(d_1) = 1 + \frac{\sqrt{h}}{4}$, $p^\sharp(d_2) = 1$, and $p^\sharp(d_3) = 2$.
\end{enumerate}
\end{lemma}
In the above lemma, Items~\ref{item:lb-dpp-a}-\ref{item:lb-dpp-c} show that the constructed functions are real demand functions satisfying the standard assumptions in literature (also listed in Assumption~\ref{assumption:1}); Items~\ref{item:lb-dpp-d}-\ref{item:lb-dpp-e} show that the first two demand functions ($d_1$ and $d_2$) are very similar to each other and therefore it requires relatively more observations from noisy demands to differentiate them; Item~\ref{item:lb-dpp-f} simply asserts the unconstrained optimal price for each demand function. Items~\ref{item:lb-dpp-d}-\ref{item:lb-dpp-f} will be used later in our lower bound proof.

The proof of Lemma~\ref{lemma:lb-demand-profit-properties} and all other proofs in the rest of this section can be found in the supplementary materials.

Using Item~\ref{item:lb-dpp-f} of Lemma~\ref{lemma:lb-demand-profit-properties}, we may first compute the optimal fairness-aware solutions to both of our constructed problem instances as follows. 
\begin{lemma}\label{lemma:lb-optimal-fairness-aware-solution}
Suppose that $h \leq \epsilon^2/40$, we have the following equalities.
\begin{align*}
p^*(d_1; \mathcal I) &= \frac{1}{2}(p^\sharp(d_1) + p^\sharp(d_3)) - \frac{\lambda}{2}(p^{\sharp}(d_3) - p^\sharp(d_1)), \\
p^*(d_3; \mathcal I) &= \frac{1}{2}(p^\sharp(d_1) + p^\sharp(d_3)) + \frac{\lambda}{2}(p^{\sharp}(d_3) - p^\sharp(d_1)),\\
p^*(d_2; \mathcal I') &= \frac{1}{2}(p^\sharp(d_1) + p^\sharp(d_3)) - \frac{\lambda}{2}(p^{\sharp}(d_3) - p^\sharp(d_2)),\\
p^*(d_3; \mathcal I') &= \frac{1}{2}(p^\sharp(d_1) + p^\sharp(d_3)) + \frac{\lambda}{2}(p^{\sharp}(d_3) - p^\sharp(d_2)).
\end{align*}
\end{lemma}

\medskip
\noindent{\bf \underline{The price of a cheap first-group price.}} By Lemma~\ref{lemma:lb-optimal-fairness-aware-solution}, we see that when $h \leq  \epsilon^2/40$, we have that both $p^*(d_1; \mathcal I)$ and $p^*(d_2; \mathcal I')$ are greater than $1 + \frac{7\sqrt{h}}{4}$. For any pricing strategy $(p, p')$, we say it is \emph{cheap for the first group} if $p \leq 1 + \frac{7\sqrt{h}}{4}$. The following lemma lower bounds the regret of a fairness-aware pricing strategy when it is cheap for the first group (and therefore deviates from the optimal solution).

\begin{lemma}\label{lemma:lb-regret-fairness-aware}
Suppose that $h \leq \eps^4/400$.\footnote{This is a stricter assumption than Lemma~\ref{lemma:lb-optimal-fairness-aware-solution}. We do not make effort to optimize the dependence between $h$ and $\epsilon$ (which also affects the choice of the constant $c_{\mathrm{LB}}$).} For any fairness-aware pricing strategy $(p_1, p_3)$ for the problem instance $\mathcal I = \mathcal I(d_1, d_3)$, if $p_1 \in [1, 1 + \frac{7\sqrt{h}}{4}]$, we have that
\[
\left[R_1(p^*(d_1; \mathcal I)) + R_3(p^*(d_3; \mathcal I))\right] - \left[R_1(p_1) + R_3(p_3) \right] \geq \frac{\epsilon^2}{4A}.
\]
Similarly, for any fairness-aware pricing strategy $(p_2, p_3)$ for the problem instance $\mathcal I' = \mathcal I(d_2, d_3)$, if $p_2 \in [1, 1 + \frac{7\sqrt{h}}{4}]$, we have that
\[
\left[R_2(p^*(d_2; \mathcal I')) + R_3(p^*(d_3; \mathcal I'))\right] - \left[R_2(p_2) + R_3(p_3) \right] \geq \frac{\epsilon^2}{4A}.
\]
\end{lemma}

\medskip
\noindent{\bf \underline{The price of identifying the wrong instance.}} If a pricing strategy misidentifies the underlying instance $\mathcal I'$ by $\mathcal I$ and satisfies the fairness condition of $\mathcal I$, we show in the following lemma that the significant regret would occur when we apply such a pricing strategy to $\mathcal I'$. The proof of Lemma~\ref{lemma:lb-regret-I-prime} also relies on the optimal fairness-aware solutions solved by Lemma~\ref{lemma:lb-optimal-fairness-aware-solution}.
\begin{lemma}\label{lemma:lb-regret-I-prime}
Suppose that $h \leq \epsilon^2/40$ and $(p_2, p_3)$ is a pricing strategy that satisfies the fairness condition of $\mathcal I$, i.e., 
\[
|p_2 - p_3| \leq \lambda |p^\sharp(d_3) - p^\sharp(d_1)|.
\]
Then we have that
\[
\left[R_2(p^*(d_2; \mathcal I')) + R_3(p^*(d_3; \mathcal I'))\right] - \left[R_2(p_2) + R_3(p_3) \right] \geq \frac{\epsilon \lambda \sqrt{h}}{4A}.
\]
\end{lemma}

\medskip
\noindent{\bf \underline{The necessity of cheap first-group prices to separate the two instances apart.}} We now show that one has to offer cheap first-group prices to separate $\mathcal I$ from $\mathcal I'$. This is intuitively true because the only difference between $\mathcal I$ and $\mathcal I'$ is the demand of the first group when the price is less than $1 + \frac{7\sqrt{h}}{4}$. Formally, for any online pricing policy algorithm $\pi$ and any problem instance $\mathcal J \in \{\mathcal I, \mathcal I'\}$, let $\mathcal P_{\mathcal J, \pi}$ be the probability measure induced by running $\pi$ in $\mathcal J$ for $T$ time periods. For each time period $t \in \{1, 2, \dots, T\}$, let $p^{(t)}(d; \mathcal J, \pi)$ denote the price offered by $\pi$ to the customer group with demand function $d$ in the problem instance $\mathcal J$. The following lemma upper bounds the KL-divergence between $\mathcal P_{\mathcal I, \pi}$ and $\mathcal P_{\mathcal I', \pi}$ (note that the upper bound relates to the expected number of cheap first-group prices). 

\begin{lemma}\label{lemma:lb-KL-decomposition}
For any $\pi$, it holds that
\begin{align*}
D_{\mathrm{KL}}(\mathcal P_{\mathcal I, \pi}\|\mathcal P_{\mathcal I', \pi}) \leq \sum_{t=1}^{T} \Pr_{\mathcal P_{\mathcal I, \pi}}\left[p^{(t)}(d_1; \mathcal I, \pi) \in \left[1, 1+\frac{7\sqrt{h}}{4}\right]\right] \cdot \frac{4h^2}{A^2} ,
\end{align*}   
where $\Pr_{\mathcal P}[\cdot]$ denotes the probability under the probability measure $\mathcal P$.
\end{lemma}

Since the difference between the probability measures $\mathcal P_{\mathcal I, \pi}$ and $P_{\mathcal I', \pi}$ boils down to the demand distributions defined by $d_1$ and $d_2$ (when the cheap first-group prices are offered), in the proof of Lemma~\ref{lemma:lb-KL-decomposition}, we use the additivity property of KL-divergence to relate $D_{\mathrm{KL}}(\mathcal P_{\mathcal I, \pi}\|\mathcal P_{\mathcal I', \pi})$ to the KL-divergence between the demand distributions defined by $d_1$ and $d_2$, multiplied by the expected number of cheap first-group prices, and use Items~\ref{item:lb-dpp-d}-\ref{item:lb-dpp-e} of Lemma~\ref{lemma:lb-demand-profit-properties} to upper bound the latter quantity.

Now we have all the technical tools prepared. To prove our main lower bound theorem, note that any pricing algorithm has to offer enough amount of cheap first-group prices in order to learn whether the underline instance is $\mathcal I$ or $\mathcal I'$ (otherwise, misidentifying the instances would lead to a large regret). On the other hand, the learning process itself also incurs regret. In the proof of Theorem~\ref{theorem:lb} (please refer to Section~\ref{subsec:proofoflbtheorem} in the supplementary materials), we rigorously lower bound any possible tradeoff between these two types of regret and show the desired $\Omega(T^{4/5})$ bound.

\blue{
\section{Discussion on the Key Causes of the $\tilde{\Theta}(T^{4/5})$ Optimal Regret}\label{sec:discussion-key-causes}

We now have derived that $\tilde{\Theta}(T^{4/5})$ is the optimal regret (both the upper and lower bound) for the price-fairness-aware learning algorithm. This type of regret might seem unfamiliar compared to the usual $\sqrt{T}$-type optimal bounds in online learning literature. In this section, we identify and explain the two key elements in our problem that jointly lead to the new regret regime -- the relative fairness constraint and the nonparametric demand model. Since we have explained how our algorithm achieves the $\tilde{O}(T^{4/5})$ regret upper bound (at the beginning of Section~\ref{sec:upper}), here we will focus on why these two problem elements are the key causes to our $\Omega(T^{4/5})$ regret lower bound.

\noindent\underline{\bf The key role of the relative constraints.}
We first argue that a better regret is possible for the absolute fairness constraint instead of the relative constraint (while still with the nonparametric demand model), which shows that the slow convergence rate of the estimated relative constraint plays a vital role in the $\Omega(T^{4/5})$ regret lower bound.

Indeed, if we work with the absolute constraint $|p_1^{(t)} - p_2^{(t)}| \leq C$, we may assume that $|p_1^* - p_2^*| = C$.\footnote{There is another case, $|p_1^* - p_2^*| < C$, which implies that $(p_1^\sharp, p_2^\sharp) = (p_1^*, p_2^*)$. In this case, the problem becomes learning the unconstrained optimal prices, and can be quite easily resolved.} Since we do not need to learn $C$, we may always set $p_2^{(t)} = p_1^{(t)} \pm C$, and focus on learning $p_1^{(t)}$. We may discretize the price range for $p_1$, get $K(T)$ checkpoints and treat the problem as a multi-armed bandit with $2K(T)$ arms (where each arm corresponds to a checkpoint $\ell_j$ and pricing decision $(p_1^{(t)}, p_2^{(t)}) = (\ell_j, \ell_j \pm C)$). Applying the well-known Upper-Confidence-Bound (UCB)  algorithm to this problem, we achieve an $\tilde{O}(\sqrt{2K(T) \cdot T} + T/K(T))$ regret, where the first term is the standard UCB regret, and the second term is the error due to  discretization. Choosing $K(T) = T^{1/3}$, we get an $\tilde{O}(T^{2/3})$ regret for the absolute constraint, better than $T^{4/5}$.\footnote{An improved algorithm may achieve the optimal $\tilde{O}(\sqrt{T})$ regret. However, this is not the focus of this paper.}

\noindent\underline{\bf A simple $\Omega(T^{4/5})$ lower bound argument under an additional assumption.}
To facilitate discussion we now work with the three-stage algorithmic framework described at the beginning of Section~\ref{sec:upper} and assume that the unconstrained-optimal-price-gap estimator can not be improved during periods in Stage II. Let $\xi$ be the unconstrained-optimal-price-gap estimate with estimation error $\tilde{\Theta}(a^{-1/4}(T))$ established in Stage I. Under this assumption, even though we have the full knowledge of the demands to estimate the constrained optimal prices through the following static optimization problem
\begin{align*}
\max_{p_1, p_2 \in [\underline p, \overline p]} \quad  R_1(p_1) + R_2(p_2)\quad
\text{s.t.} \quad  \left|M_1(p_{1}) - M_2(p_{2})\right| \leq \lambda \xi , \nonumber
\end{align*}
the estimate error of the constrained-optimal-price estimators is still  $\tilde{\Theta}(a^{-1/4}(T))$. This would incur $\tilde{\Theta} ((T-a(T))a^{-1/4}(T))$ in the remaining selling periods.
Combining this with the regret incurred in Stage  I (which is $O(a(T))$), we see that the best choice for $a(T)$ is $T^{4/5}$ and the regret has to be at least $\Omega(T^{4/5})$ if we follow the three-stage algorithmic framework and the above assumption.

\noindent\underline{\bf The key role of the nonparametric demand model.} Although the above argument is based on a heavy assumption, it is the basic idea behind the construction and formal analysis of the lower bound instances in Section~\ref{sec:lower}. Moreover, it reveals the essential role of the nonparametric demand in our lower bound proof --  the assumption essentially characterizes the limitation of a learner in the nonparametric setting, but it may not be valid for the parametric demand functions. 

Indeed, for nonparametric demands, the unconstrained optimal prices are learned by constantly shrinking the active interval. When we learn the constrained optimal prices through discretization and Explore-Then-Commit, it is impossible for us to obtain a more accurate estimate of the unconstrained optimal prices. Thus, the assumption seems valid when the demand function is nonparametric. For the parametric demand functions, however, samples collected in all periods can be used to improve the accuracy of the parameter estimation, and thereby improve the unconstrained optimal price estimator. Therefore, the assumption may not be valid for the parametric demand functions.

\noindent\underline{\bf Towards a better regret for parametric demands.} The simplest and most popularly considered parametric demand classes are the linear demand functions (e.g., \cite{Cohen:21:dynamic}). If we still follow the aforementioned three-stage framework, however, we might not be able to improve the regret bound (as explained in Section~\ref{sec:three-stage-linear-demand}). We believe that the key to a better regret for linear (and more general parametric) demand functions is to stay out of the aforementioned assumption and keep improving the unconstrained-optimal-price-gap estimator (closely related to the relative fairness constraint) throughout the algorithm.

One possible solution could be a multi-stage framework. Specifically, the time periods after the original Stage I could be divided into multiple stages, and in each stage, the updated unconstrained-optimal-price-gap estimator may be used to restart the original Stage II and Stage III. In this way, we would derive a learning-while-doing algorithm which might be helpful to achieve a lower regret bound. While the development and analysis of such an algorithm for a better regret for parametric demands is beyond the scope of this paper, it is an interesting direction for future research.

}

\section{Extensions }

\label{sec:general}

In this section, we extend our fairness-aware dynamic pricing  algorithm in Section~\ref{sec:upper} to the general fairness measure $\{M_i(p)\}$ with soft constraints. We will present a policy and prove that its regret can also be controlled by the order of $\tilde{O}(T^{4/5})$. \blue{We also conduct several other extensions outlined at the end of this section. Due to space constraints we leave them to the supplementary materials.}

Our policy is presented in Algorithm~\ref{alg:price-fairness-general}.
Similar to the algorithm for price fairness, Algorithm~\ref{alg:price-fairness-general} also works in the explore-and-exploit manner, where the first two stages are the explore phases. The first exploration stage, the {\sc ExploreUnconstrainedOPT} subroutine, is exactly Algorithm~\ref{alg:exp-unconstrainted-opt} introduced in Section~\ref{sec:upper}, which serves to estimate the unconstrained optimal prices $p_1^\sharp$ and $p_2^\sharp$. Below we describe the new subroutine {\sc ExploreConstrainedOPTGeneral} used in the second step.

    \begin{algorithm}[!t]
        \caption{Fairness-aware Dynamic Pricing for General Fairness Measure \label{alg:price-fairness-general}}		
        For each group $i \in \{1, 2\}$, run {\sc ExploreUnconstrainedOPT} (Algorithm~\ref{alg:exp-unconstrainted-opt}) separately with the input $z = i$, and obtain the estimation of the optimal price without fairness constraint $\hat{p}_i^\sharp$.

        Given $\hat{p}_1^\sharp$ and $\hat{p}_2^\sharp$, run {\sc ExploreConstrainedOPTGeneral} (Algorithm~\ref{alg:exp-constrained-opt-general}), and obtain  $(\hat{p}_1^*, \hat{p}_2^*)$.

        For each of the remaining selling periods, offer $\hat{p}_i^*$ to the customer group $i$.
    \end{algorithm}   

\paragraph{\underline{The {\sc ExploreConstrainedOPTGeneral} subroutine.}}
Suppose we have already run {\sc ExploreUnconstrainedOPT} and obtained both $\hat{p}_1^\sharp$ and $\hat{p}_2^\sharp$. The {\sc ExploreConstrainedOPTGeneral} estimates the  clairvoyant solution with the soft fairness constraint for both groups. The pseudo-code of this procedure is presented in Algorithm~\ref{alg:exp-constrained-opt-general}.

\begin{algorithm}[!t]
\caption{\sc ExploreConstrainedOPTGeneral\label{alg:exp-constrained-opt-general}}		
\SetKwInOut{Input}{Input}
\SetKwInOut{Output}{Output}
\Input{the estimated unconstrained optimal prices $\hat{p}_1^\sharp$ and $\hat{p}_2^\sharp$,  assume that $\hat{p}_1^\sharp \leq \hat{p}_2^\sharp$ (without loss of generality)}
\Output{the estimated constrained optimal prices $\hat{p}_1^*$ and  $\hat{p}_2^*$}

$\xi \leftarrow \max\{ |\hat{p}_1^\sharp - \hat{p}_2^\sharp| , 0\}$\; 

$J \leftarrow \lceil (\overline{p}-\underline{p}) T^\frac{1}{5}\rceil$ and create $J$ price checkpoints $\ell_1, \ell_2, \dots, \ell_J$ where $\ell_j \leftarrow \underline{p} + \frac{j}{J} (\overline{p} - \underline{p})$\; 

\For{each $\ell_j$}{
    Repeat the following offering for $6  T^{2/5} \ln T$ selling periods: offer price $\ell_j$ to both of the customer groups\; 
    For each customer group $i \in \{1, 2\}$, denote the average demand from the customer group $\hat{d}_i(\ell_j)$, and the average of the observed fairness measurement value by $\hat{M}_i(\ell_j)$\;
    Let $\hat{R}_i(\ell_j) \leftarrow \hat{d}_i(\ell_j) \cdot (\ell_j - c)$, for each $i \in \{1, 2\}$\;
}

For each $i \in \{1, 2\}$, round up $\hat{p}_i^\sharp$ to the nearest price checkpoint, namely $\ell_{t_i}$\; \label{line:alg-explore-constrained-opt-general-7}


For all pairs $j_1, j_2 \in \{1, 2, \dots, J\}$, let $\hat{G}(\ell_{j_1}, \ell_{j_2}) \leftarrow \hat{R}_1(\ell_{j_1}) + \hat{R}_2(\ell_{j_2}) - \gamma \max\left(|\hat{M}_1(\ell_{j_1}) - \hat{M}_2(\ell_{j_2})| - \lambda \left| \hat{M}_1(\ell_{t_1})  - \hat{M}_2(\ell_{t_2}) \right|,  0\right)$\; \label{line:alg-explore-constrained-opt-general-8}

Let $(j_1^*, j_2^*) \leftarrow \argmax_{j_1, j_2 \in \{1, 2, \dots, J\}} \{ \hat{G}(\ell_{j_1}, \ell_{j_2})\}$\; \label{line:alg-explore-constrained-opt-general-9}

\Return $(\hat{p}^*_1, \hat{p}^*_2) \leftarrow (\ell_{j_1^*}, \ell_{j_2^*})$\;
\end{algorithm}

Similar to the  {\sc ExploreConstrainedOPT} (Algorithm~\ref{alg:exp-constrained-opt}) in Section~\ref{sec:upper}, Algorithm~\ref{alg:exp-constrained-opt-general} also adopts the discretization technique. The key differences are that: (1) we also need to calculate the estimation of the fairness measure functions $M_i(\cdot)$ at each price checkpoint $\ell_j$; and (2) with soft fairness constraints, we are allowed to consider every pair of prices to the two customer groups; however, we need to deduct the fairness penalty term from the estimated revenue from each pair of discretized prices at Line~\ref{line:alg-explore-constrained-opt-general-8} of the algorithm.

Formally, we state the following guarantee for Algorithm~\ref{alg:exp-constrained-opt-general} and its proof will be provided in Section~\ref{sec:alg-exp-constrained-opt-general}.

\begin{theorem} \label{thm:exp-constrained-opt-general}
Suppose that $|\hat{p}_1^\sharp - p_1^\sharp| < 4T^{-\frac{1}{5}} $ and $|\hat{p}_2^\sharp - p_2^\sharp| < 4T^{-\frac{1}{5}}$. Also assume that $\gamma \leq O(1)$.  Algorithm~\ref{alg:exp-constrained-opt-general} uses at most $O(T^\frac{3}{5}\ln T)$ selling periods in total, and with probability at least $(1 - O(T^{-1}))$, the procedure returns a pair of prices $(\hat{p}_1^*, \hat{p}_2^*)$ such that
\begin{align}
& \left[R_1(p_1^*) + R_2(p_2^*) - R_1(\hat{p}_1^*) - R_2(\hat{p}_2^*)\right] \nonumber \\
& \qquad\qquad\qquad\qquad + \gamma \max\left(|M_1(\hat{p}_1^*) - M_2(\hat{p}_2^*)| - \lambda \left| M_1(p_1^\sharp)  - M_2(p_2^\sharp) \right|,  0\right)\leq O\left(T^{-\frac{1}{5}}\right). \label{eq:thm-exp-constrained-opt-general}
\end{align}
Here the $O(\cdot)$ notation hides the polynomial dependence on $\overline{p}$, $\gamma$, $K$, $K'$, $\overline{M}$ and $C$.
\end{theorem}
Note that the Left-Hand-Side of Eq.~\eqref{eq:thm-exp-constrained-opt-general} is the penalized regret incurred by a single selling period when the offered prices are $\hat{p}_1^*$ and $\hat{p}_2^*$.

Combining Theorem~\ref{thm:exp-unconstrained-opt} and Theorem~\ref{thm:exp-constrained-opt-general}, we are ready to state the regret bound of the algorithm.
\begin{theorem} \label{thm:main-upper-general}
Assume that $\gamma \leq O(1)$. With probability $(1 - O(T^{-1}))$, the cumulative penalized regret of Algorithm~\ref{alg:price-fairness-general} is at most $\mathrm{Reg}_T^{\mathrm{soft}} \leq O(T^{4/5}\log^2 T)$. Here the $O(\cdot)$ notation hides the polynomial dependence on $\overline{p}$, $\gamma$, $K$, $K'$, $\overline{M}$, and $C$.
\end{theorem}

\blue{
\noindent\underline{\bf Other extensions.} In the supplementary materials, we provide the following further extensions.

\begin{enumerate}
\item  The general discrepancy function: while we have aimed at achieving fairness via mandating the small \emph{difference} between the prices (or other fairness measures defined by $\{M_i(\cdot)\}$), in this extension, we consider a \emph{general discrepancy function} $f(\cdot, \cdot)$ between the prices (or other fairness measures) to define the fairness constraints. This further broadens the scope of fairness constraints supported by our algorithmic framework. Please refer to Section~\ref{sec: extension general f} for details.
\item The multi-group setting: we generalize our algorithms to the setting of $N \geq 2$ groups under the fairness constraint  introduced in \cite{cohen2022price}: $\left|M_i(p_{i}) - M_j(p_{j})\right| \leq \lambda \max_{1 \leq i' < j' \leq N}\left|M_{i'}(p_{i'}^\sharp) - M_{j'}(p_{j'}^\sharp)\right|$ for $1 \leq i < j \leq N$. Please refer to Section~\ref{sec:extension multiple group} for details.
\item Lower bounds for the regret in the soft constraint setting: we establish the same $\Omega(T^{4/5})$ penalized regret lower bound for learning algorithms under the soft constraints. Please refer to Section~\ref{sec:lb-soft-constraint} for details.
\end{enumerate}

}

\section{Numerical Study}
\label{sec:numerical}
In this section, we provide experimental results to demonstrate Algorithm~\ref{alg:price-fairness-main} for price fairness and Algorithm~\ref{alg:price-fairness-general} for demand fairness.  For simplicity, we refer to Algorithm~\ref{alg:price-fairness-main} as FDP-DL (Fairness-aware Dynamic Pricing with Demand Learning) 
and Algorithm~\ref{alg:price-fairness-general} as FDP-GFM (Fairness-aware Dynamic Pricing - Generalized Fairness Measure).

\blue{In the experiment, we pick the following forms of demand functions to illustrate the strength and robustness of our algorithm compared to the baseline algorithms:
 \begin{enumerate}
     \item \label{experiment:Exp} $d_1(p_1) = \frac{1}{2} \exp(1 - p_1)$ and $d_2(p_2) = \frac{1}{2} \exp(\frac{1 - p_2}{2})$.
     In this setting, the demand functions are the classical exponential function with $p^\sharp_1 = 1$ and $p^\sharp_2 = 2$.
     \item \label{experiment:Linear}$d_1(p_1) = - \frac{p_1}{10} + \frac{3}{5}$ and $d_2(p_2) = - \frac{p_2}{10} + \frac{4}{5}$.
     In this setting, the demand functions are the linear function with $p^\sharp_1 = 3$ and $p^\sharp_2 = 4$. 
     \item \label{experiment:Unimodal}$d_1(p_1) = \max(0, \min(1, \frac{2}{p} - 1))$, $d_2(p_2) = \max(0, \min(1, \frac{4}{p} - 1))$.
     In this setting, the demand functions are the inverse proportional function, bounded by $[0,1]$, and the reward functions are the unimodal piecewise linear function, with $p^\sharp_1 = 2$ and $p^\sharp_2 = 4$. Note that the demand functions here do not meet the Assumptions~\ref{assumption:1}\ref{item:assumption-1-lipschitz} and \ref{assumption:1}\ref{item:assumption-1-strong-concavity}. We use this instance to illustrate the robustness of our algorithm when the theoretical assumptions are not satisfied.
 \end{enumerate}
The realized demand at each time period $t$, $D_i^t$, follows the Bernoulli distribution with the mean $d_i(p_i^{(t)})$.  We further set the price range to be $[\underline{p}, \overline{p}] = [0,5]$.}

For the ease of illustration, we assume the cost $c=0$.  We vary the key fairness parameter $\lambda$ from $0$, $0.2$, $0.5$ to $0.8$ and $1$ (a larger $\lambda$ indicates more relaxed fairness requirement), and vary the selling periods $T$  from 100,000 to 1,000,000. For each different parameter setting, we would repeat the experiments for 1000 times and report the average performance in terms of the cumulative regret.

For comparison, we consider two methods in the dynamic pricing literature that handles nonparametric demand functions: (1) a tri-section search algorithm adapted from \cite{lei2014near}, and (2) a nonparametric Dynamic Pricing Algorithm (DPA) adapted from \cite{wang2014close}. Both baseline algorithms try to learn the optimal price by shrinking the price interval. The key difference between the tri-section search  and DPA is the number of difference prices to be tested at each learning period: the tri-section search will only test two prices while DPA will test poly($T$) prices at each learning period. 

\begin{figure}[!t]
\centering
\includegraphics[width =0.32\textwidth]{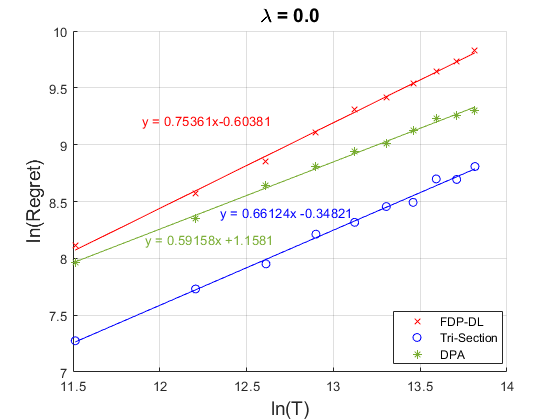}
\includegraphics[width =0.32\textwidth]{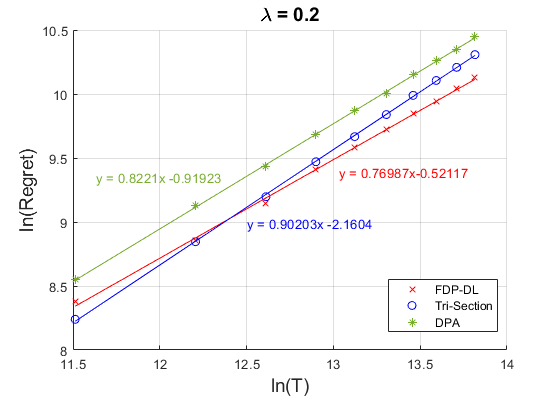}
\includegraphics[width =0.32\textwidth]{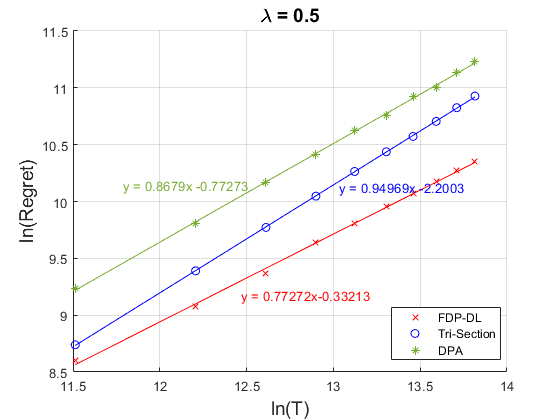}
\includegraphics[width =0.32\textwidth]{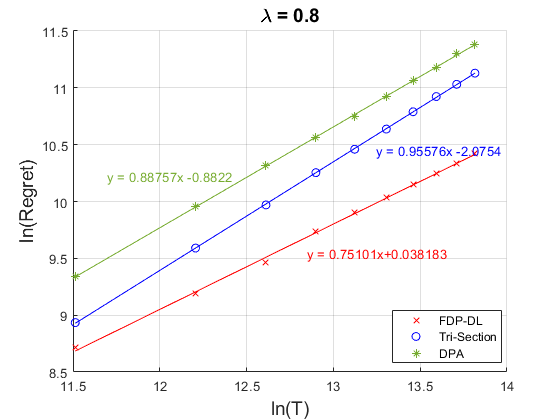}
\includegraphics[width =0.32\textwidth]{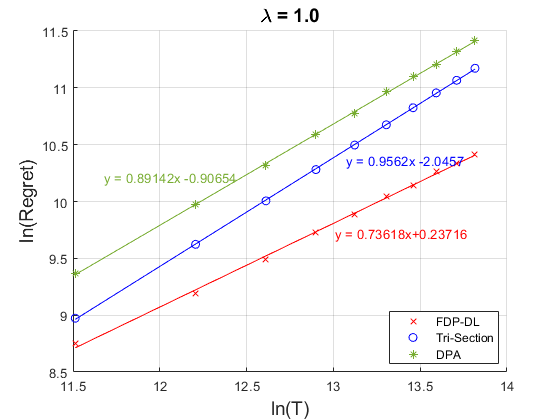}
\caption{\blue{The regret performance of Algorithm~\ref{alg:price-fairness-main}. Here the $x$-axis is the logarithm of the total number periods $T$ and the $y$-axis is the logarithm of the cumulative regret. We consider three values of the fairness-ware parameter $\lambda=0$, $0.2$, $0.5$, $0.8$ and $1.0$.}}
\label{experiment-figure-1}
\end{figure}

As previous dynamic pricing algorithms with nonparametric demand learning do not take fairness into consideration, it is hard to make a direct comparison. For the illustration purpose, we simply assume that the benchmark algorithms provide the same price to both customer groups at each time period.  This is perhaps the most intuitive way to guarantee the fairness for benchmark algorithms. 

Note that under the single-customer-group setting, both baseline algorithms provide the almost optimal regret bound $\tilde{O}(\sqrt{T})$  up to poly-logarithmic factors. 

On the other hand,  a single-price-at-a-time algorithm would have a theoretical regret lower bound of $\Omega(T)$. Indeed, in the worst-case scenario, always offering the same price might not well satisfy at least one customer group. \blue{This limitation is due to that both baseline algorithms focus on optimizing a single price in the more and more refined neighborhood, while in contrast fairness is a global constraint that requires demand information at prices far apart from the optimal decisions. This phenomenon resembles the known incomplete learning issue of the ``myopic policies'' in the parametric online optimization problems with certain decision constraints \citep{lai1982iterated,keskin2018incomplete}. Our numerical results  demonstrate a significant improvement of our algorithm over the baselines, which also shows the importance of the dedicated stage in our algorithm to learn fairness constraints.} 

In Figure~\ref{experiment-figure-1}, we present the performance of our algorithm and the benchmark algorithms  under the classical exponential function \ref{experiment:Exp}. We use log scales on both axes to better show the relationship between the regret and the total time period. For better illustration, we fit the experiment results with linear functions. As one can see, the slope of the line for our algorithm is close to \blue{or better than} $0.8=4/5$, while the slopes of the baseline algorithms are close to $1$ \blue{when $\lambda$ become larger}. These results are consistent with our theoretical result (Theorem \ref{thm:main_upper}).

Another interesting observation is that the benefit of our algorithm, comparing to baseline algorithms, becomes more significant when $\lambda$ becomes larger. This is because when $\lambda$ is smaller, the benefits of distinguishing the best prices of two customer groups also get smaller. Indeed, the single-price-at-a-time baselines achieves (theoretically) optimal regret when $\lambda=0$.

\blue{Due to space constraints, we leave the numerical results for demand formulations~\ref{experiment:Linear} and \ref{experiment:Unimodal} to
Section~\ref{sec:extra_exp} in the supplementary materials.}
\begin{figure}[!t]
\centering
\includegraphics[width =0.32\textwidth]{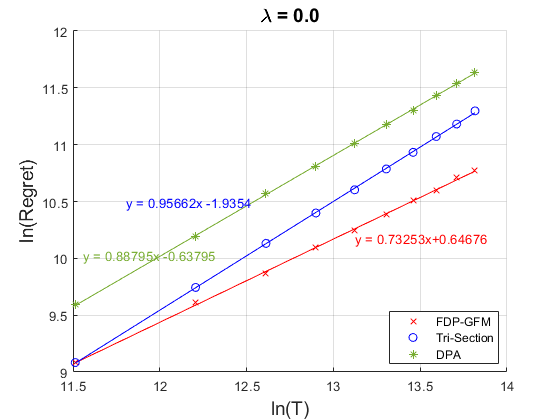}
\includegraphics[width =0.32\textwidth]{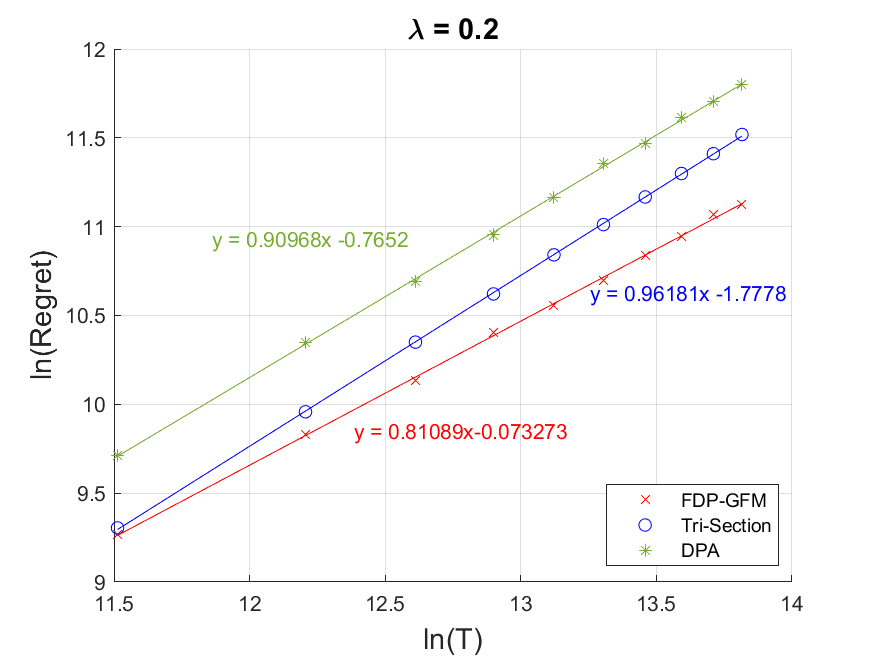}
\includegraphics[width =0.32\textwidth]{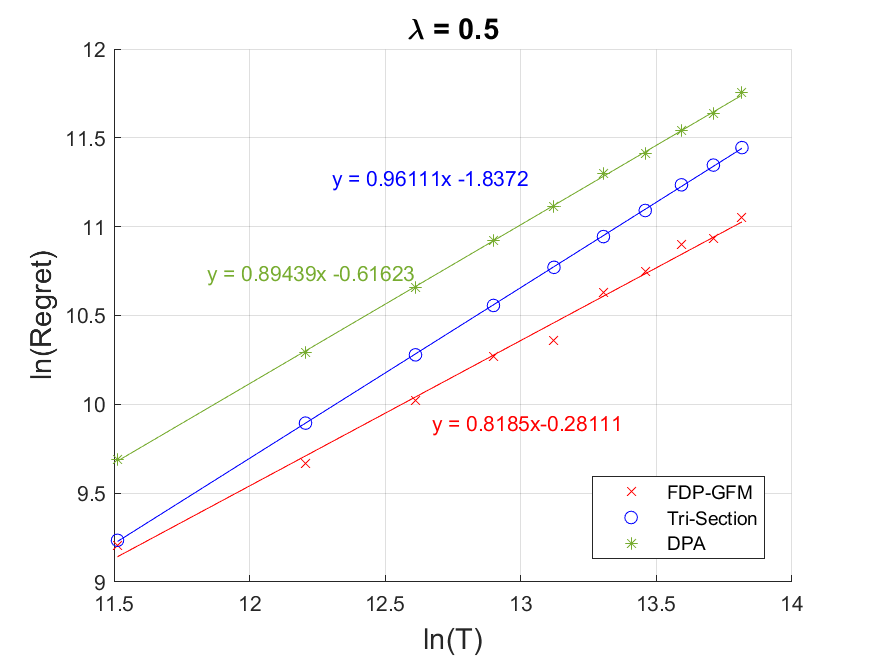}
\includegraphics[width =0.32\textwidth]{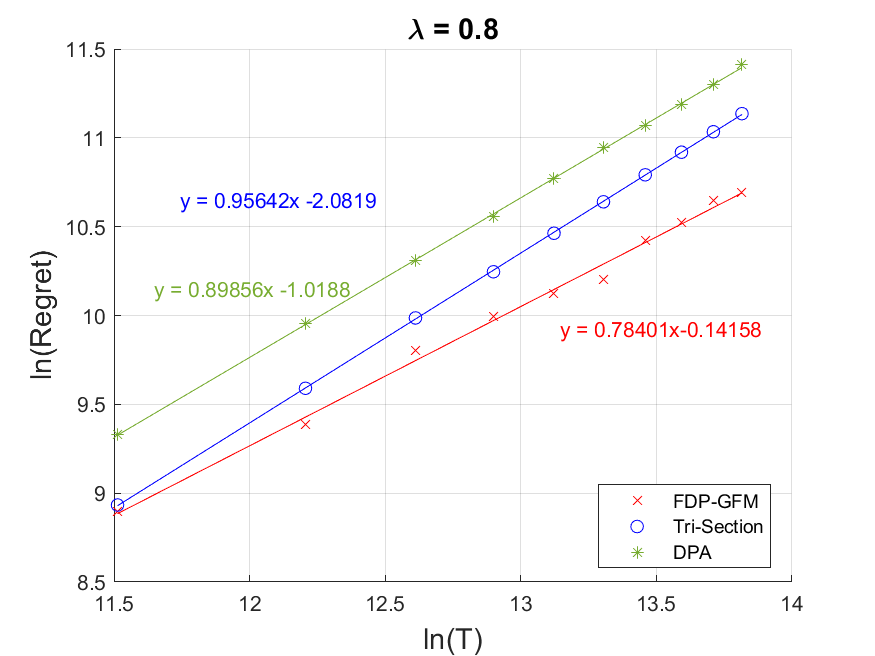}
\includegraphics[width =0.32\textwidth]{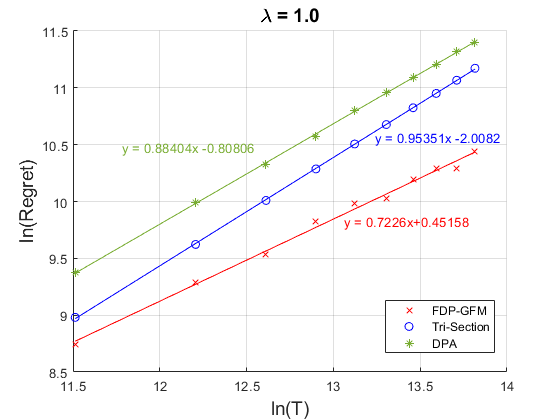}
\caption{\blue{The regret performance of Algorithm~\ref{alg:price-fairness-general}. Here the $x$-axis is the logarithm of the total number periods $T$ and the $y$-axis the log of the cumulative regret. We consider three values of the fairness-ware parameter $\lambda=0.0$, $0.2$, $0.5$, $0.8$ and $1.0$.}}
\label{experiment-figure-2}
\end{figure}
\subsection*{General Case}
In the experiment for the general fairness measure, we consider the demand fairness (i.e., $M_i(p^{(t)}_i) = D_i(p^{(t)}_i)$) and  the classical exponential demand function \ref{experiment:Exp} in Algorithm~\ref{alg:price-fairness-general} (FDP-GFM). Similar to the experiment setup of the Algorithm~\ref{alg:price-fairness-main}, we set $\lambda = 0$, $0.2$, $0.5$, $0.8$, and $1.0$ and let the maximum selling periods $T$ vary from 100,000 to 1,000,000. Furthermore, recalling the penalized regret in (\ref{eq:general-fairness-soft-constraint}), we set the parameter $\gamma = 1$ to balance the penalty and the original objective. We test each setting for 100 times and report the average performance. 

The result is shown in Figure~\ref{experiment-figure-2}. As one can see, the results are also quite similar to the previous one: the slope of the fitted line of the Algorithm~\ref{alg:price-fairness-general} (FDP-GFM) is close to $0.8=4/5$, which matches the theoretical regret bound of $\tilde{O}(T^\frac{4}{5})$. Similarly, the baseline algorithms perform much worse and the corresponding slopes are close to 1.

 \section{\blue{Conclusion and Future Directions}}
\label{sec:con}

\blue{This paper extends the static pricing under fairness constraints from \cite{cohen2022price} to the dynamic discriminatory pricing setting. We propose fairness-aware pricing policies that achieve $O(T^{4/5})$ regret and establish its optimality. 

There are several directions for future research. First, one can consider establishing the regret of aforementioned multi-stage policy for the parametric demand with relative fairness constraints, and examining whether it is optimal. Second, as this paper focuses on the fairness constraint,  we omit operational constraints, such as the inventory constraint. It would be interesting to explore the dynamic discriminatory pricing under the inventory constraint. 
Third, it is worth mentioning that the regret lower bound in Theorem~\ref{theorem:lb} assumes that $\lambda$ is bounded away from $0$ or $1$ (i.e., $\lambda\in(\epsilon, 1-\epsilon)$) and the lower bound degrades at the rate $\epsilon^2 T^{4/5}$ as $\lambda$ gets close to $0$ or $1$ --- which is consistent with the existence of $\sqrt{T}$-regret algorithms for $\lambda = 0$. It is quite interesting and challenging to conduct further fine-scaled study and establish the tight regret when $\lambda$ extremely close to $0$ or $1$.
Finally, with the advance of technology in decision-making, fairness has become a primary ethical concern, especially in the e-commerce domain. We would like to explore more fairness-aware revenue management problems.}


\section*{Acknowledgement}

Xi Chen and Yuan Zhou would like to thank the support from JPMorgan Faculty Research Awards.
We also thank helpful discussions from Ivan Brugere, Jiahao Chen, and Sameena Shah.

\bibliography{reference}
\bibliographystyle{apa-good}

\ECSwitch



\begin{center}
 		\large{\bf Supplementary Materials to `` Fairness-aware Online Price Discrimination with Nonparametric Demand Models''} 
 	\end{center}
 	
 	\vspace{10pt}

\section{Proof Omitted in Section~\ref{sec:upper}}
\subsection{Proof of Theorem \ref{thm:exp-unconstrained-opt} for {\sc ExploreUnconstrainedOPT}}
	\label{sec:alg:exp-unconstrainted-opt}
	
In this subsection, we establish the theoretical guarantee  for {\sc ExploreUnconstrainedOPT} in  Theorem \ref{thm:exp-unconstrained-opt}.
	
First, the following lemma upper bounds the number of time periods used by the algorithm.

\begin{lemma}\label{lem:exp-unconstrained-sample-complexity}
Each invocation of Algorithm~\ref{alg:exp-unconstrainted-opt} spends at most $O(\frac{K^4 \overline{p}^2}{C^2}  T^\frac{4}{5}\log T \log (\overline{p}  T))$ selling periods, where only a universal constant is hidden in the $O(\cdot)$ notation. \end{lemma}
\begin{proof}{Proof of Lemma~\ref{lem:exp-unconstrained-sample-complexity}.}
It is easy to verify that the length of the trisection interval $p_R - p_L$ shrinks by a factor of $2/3$ after each iteration, and therefore there are at most $\log_{3/2}((\overline{p} - \underline{p})  T^{1/5}) = O(\log (\overline{p}  T))$ iterations. Also note that within each iteration, the algorithm uses at most $O(\frac{K^4 \overline{p}^2}{C^2} T^{4/5} \log T)$ selling periods. The lemma then follows.  $\square$
\end{proof}

We then turn to upper bound the estimation error of the algorithm. For each iteration $r$, we define the following event
\[
\mathcal{A}_r := \{p^\sharp_z \in [p_L, p_R] \text{~at the end of iteration~} r\}.
\]
Let $r^*$ be the last iteration. We note that 1) $\mathcal{A}_0$ always holds, 2) $\mathcal{A}_{r^*}$, if holds, would imply the desired estimation error bound ($|\hat{p}^{\sharp}_z - p^\sharp_z| \leq 4 T^{-1/5}$). Therefore, to prove the desired error bound in Theorem~\ref{thm:exp-unconstrained-opt}, we first prove the following lemma.

\begin{lemma}\label{lem:exp-unconstrainted-opt-Ar}
For each $r \in \{1, 2, \dots, r^*\}$, we have that \[
\Pr[\mathcal{A}_r|\mathcal{A}_{r-1}] 
\geq 1 - 4 T^{-2}.
\]
\end{lemma}
\begin{proof}{Proof of Lemma~\ref{lem:exp-unconstrainted-opt-Ar}.}
Given the event $\mathcal{A}_{r-1}$, we focus on iteration $r$. During this iteration, by Azuma's inequality, we first have that for each trisection point $i \in \{1, 2\}$, it holds that
\begin{align*}
&\Pr\left[|\hat{d}_{m_i} - d_z(p_{m_i})| \leq \frac{16 C}{40K^2 \overline{p}} T^{-2/5}\right] \\
&\qquad \geq 1 - 2 \exp\left( - \left(\frac{ 16 C}{ 40K^2 \overline{p}} T^{-2/5}\right)^2 \cdot \frac{1}{2} \cdot \frac{25 K^4 \overline{p}^2}{ C^2 } T^{4/5} \ln T\right) = 1 - 2 T^{-2} .
\end{align*}
The rest of the proof will be conditioned on that 
\begin{align}
\forall i \in \{1, 2\}, |\hat{d}_{m_i} - d_z(p_{m_i})| \leq \frac{16 C}{40K^2 \overline{p}} T^{-2/5}, \label{eq:exp-unconstrainted-Ar-9}
\end{align}
which happens with probability at least $(1 - 4 T^{-2})$ by a union bound. 

To establish $\mathcal{A}_r$, let $p_L$ and $p_R$ be the values taken at the beginning of iteration $r$, and we discuss the following three cases.

\noindent {\it \underline{Case 1: $p^\sharp_z \in [p_{m_1}, p_{m_2}]$.}} $\mathcal{A}_r$ automatically holds in this case.

\noindent {\it \underline{Case 2: $p^\sharp_z \in [p_L, p_{m_1})$.}} In this case, by Line~\ref{line:alg-exploreunconstrainedopt-7} of the algorithm, to establish $\mathcal{A}_r$, we need to show that $\hat{d}_{m_1} (p_{m_1} - c) >\hat{d}_{m_2} (p_{m_2} - c)$. By Item~\ref{item:assumption-1-lipschitz} of Assumption~\ref{assumption:1}, we have that
\begin{align}
|d_z(p_{m_1}) - d_z(p_{m_2})| \geq \frac{1}{K} |p_{m_1} - p_{m_2}| \geq \frac{ 4 T^{-1/5}}{3K} . \label{eq:exp-unconstrainted-Ar-10}
\end{align}
Also, by Item~\ref{item:assumption-1-strong-concavity} of Assumption~\ref{assumption:1}, when $d_z(p^{\sharp}) > d_z(p_{m_1}) > d_z(p_{m_2})$,  we have that 
\begin{align}
d_z(p_{m_1}) (p_{m_1} - c) - d_z(p_{m_2}) (p_{m_2} - c)& = R_z(d_z(p_{m_1})) - R_z(d_z(p_{m_2})) \geq \frac{C}{2}  (d_z(p_{m_1}) - d_z(p_{m_2}))^2 \nonumber\\
& \geq \frac{16 C T^{-2/5}}{18K^2}, \label{eq:exp-unconstrainted-Ar-11}
\end{align}
where in the last inequality we applied Eq.~\eqref{eq:exp-unconstrainted-Ar-10}. Together with Eq.~\eqref{eq:exp-unconstrainted-Ar-9}, we have that
\begin{align*}
\hat{d}_{m_1} (p_{m_1}-c) - \hat{d}_{m_2} (p_{m_2}-c) \geq \frac{16 C T^{-2/5}}{18K^2} - 2 \times \frac{16 C}{40K^2 \overline{p}} T^{-2/5}  \times \overline{p} > 0.
\end{align*}
Therefore, $\mathcal{A}_r$ holds in this case.

\noindent {\it \underline{Case 3: $p^\sharp_z \in (p_{m_2}, p_R]$.}} This case can be similarly handled as Case 2 by symmetry. 

Combining the $3$ cases above, the lemma is proved.  $\square$
\end{proof}

Finally, since $r^* \leq O(\log(\overline{p}  T))$, we have that $\mathcal{A}_{r^*}$ holds with probability at least $1 - O(T^{-2} \log(\overline{p}  T))$. Together with Lemma~\ref{lem:exp-unconstrained-sample-complexity}, we prove Theorem~\ref{thm:exp-unconstrained-opt}.

\subsection{Proof of Theorem~\ref{thm:exp-constrained-opt} for {\sc ExploreConstrainedOPT}}
\label{sec:alg-exp-constrained-opt}

First, the following lemma upper bounds the number of time periods used by the algorithm.

\begin{lemma}\label{lem:exp-constrained-opt-sample-complexity}
Algorithm~\ref{alg:exp-constrained-opt} uses at most  $O(\overline{p}  T^{3/5} \ln T)$ selling periods, where only a universal constant is hidden in the $O(\cdot)$ notation.
\end{lemma}
\begin{proof}{Proof.}
For each price checkpoint $\ell_j$, the algorithm uses at most $6 T^{2/5} \ln T$ selling periods. Since there are $J = \lceil (\overline{p} - \underline{p})  T^{1/5}\rceil$ selling price checkpoints, the total selling periods used by the algorithm are at most $O(\overline{p}  T^{3/5} \ln T)$. $\square$
\end{proof}

We next turn to prove the (near-)optimality of the estimated prices $\hat{p}^*_1$ and $\hat{p}^*_2$. To this end, we first establish the following key relation between the price gap of the constrained optimal solution and that of the unconstrained optimal solution.

\begin{lemma}\label{lem:exp-constrained-opt-price-gap}
$\displaystyle{p^*_1 - p^*_2 = \lambda(p^\sharp_1 - p^\sharp_2)}$.
\end{lemma}
\begin{proof}{Proof.}
In this proof we assume without loss of generality that $p_1^\sharp \leq p_2^\sharp$ as the other case can be similarly handled by symmetry. 

Since $R_1(d)$ is a unimodal function and $d_1(p)$ is a monotonically decreasing function, we have that $R_1(p) = R_1(d_1(p))$ is a unimodal function. Similarly, $R_2(p) = R_2(d_2(p))$ is also a unimodal function. Under the price fairness constraint, we have that 
\begin{align}
(p_1^*, p_2^*) = \arg\max_{(p_1, p_2) : |p_1 - p_2| \leq \lambda |p^\sharp_1 - p^\sharp_2|} \{R_1(p_1) + R_2(p_2)\} . \label{eq:lem-exp-constrained-opt-price-gap-1}
\end{align}

We first claim that $p_1^* \leq p_2^\sharp$, since otherwise (when $p_1^* > p_2^{\sharp}$) the objective value of the feasible solution $(p_1, p_2) = (p_2^\sharp, p_2^\sharp)$ is 
\[
R_1(p_2^\sharp) + R_2(p_2^\sharp) > R_1(p_1^*) + R_2(p_2^\sharp) \geq R_1(p_1^*) + R_2(p_2^*),
\]
where the first inequality is due to the unimodality of $R_1(p)$. This contradicts to the optimality of $(p_1^*, p_2^*)$. 

We also claim that $p_1^* \geq p_1^\sharp$, since otherwise (when $p_1^* < p_1^\sharp$) we have $p_2^* + p_1^\sharp - p_1^* \leq p_1^\sharp + \lambda (p_2^\sharp - p_1^\sharp) \leq p_2^\sharp$ and $p_2^* + p_1^\sharp - p_1^* > p_2^*$, and the objective value of the feasible solution $(p_1, p_2) = (p_1^\sharp, p_2^* + p_1^\sharp - p_1^*)$ is 
\[
R_1(p_1^\sharp) + R_2(p_2^* + p_1^\sharp - p_1^*) \geq  R_1(p_1^*) + R_2(p_2^* + p_1^\sharp - p_1^*) > R_1(p_1^*) + R_2(p_2^*),
\]
where the second inequality is due to the unimodality of $R_2(p)$. This also contradicts to the optimality of $(p_1^*, p_2^*)$. 

To summarize, we have shown that $p_1^* \in [p_1^\sharp, p_2^\sharp]$.

Since $R_1(p_1)$ is monotonically decreasing when $p_1 \in [p_1^\sharp, p_2^\sharp]$,  by Eq.~\eqref{eq:lem-exp-constrained-opt-price-gap-1}, we have that 
\begin{align}
p_1^* = \arg\max_{p_1 \in [p_1^\sharp, p_2^\sharp] \cap [p_2^* \pm \lambda |p^\sharp_1 - p^\sharp_2|]} \{R_1(p_1)\} = \max\{p_1^\sharp, p_2^* +  \lambda (p^\sharp_1 - p^\sharp_2)\}. \label{eq:lem-exp-constrained-opt-price-gap-2}
\end{align}
Here, we use $[a\pm b]$ to denote the interval $[a-b, a+b]$ for any $a \in \mathbb{R}$ and $b \geq 0$.

In a similar way, we can also work with $p_2^*$ and show that
\begin{align}
p_2^* = \min\{p_2^\sharp, p_1^* -  \lambda (p^\sharp_1 - p^\sharp_2)\}. \label{eq:lem-exp-constrained-opt-price-gap-3}
\end{align}

Combining Eq.~\eqref{eq:lem-exp-constrained-opt-price-gap-2} and Eq.~\eqref{eq:lem-exp-constrained-opt-price-gap-3}, we conclude that $p^*_1 - p^*_2 = \lambda(p^\sharp_1 - p^\sharp_2)$ and the lemma is proved. $\square$
\end{proof}

The following lemma provide bounds for the $\xi$ parameter which is used in the algorithm to control the price gaps between the two customer groups.

\begin{lemma}\label{lem:exp-constrained-opt-xi-bound}
Suppose that $|\hat{p}_1^\sharp - p_1^\sharp| \leq  4 T^{-1/5}$ and $|\hat{p}_2^\sharp - p_2^\sharp| \leq  4 T^{-1/5}$, we have that $\lambda \xi \leq \lambda |p_1^\sharp - p_2^\sharp|$ and  $\lambda \xi \geq \max\{0, \lambda |p_1^\sharp - p_2^\sharp| - 16 T^{-1/5}$\}.
\end{lemma}
\begin{proof}{Proof.}
We first have that
\[
\lambda \xi = \lambda \max\{0, |\hat{p}_1^\sharp - \hat{p}_2^\sharp| - 8 T^{-1/5}\} \leq \lambda \max\{0,  |p_1^\sharp - p_2^\sharp| + 8 T^{-1/5} - 8 T^{-1/5}\} = \lambda |p_1^\sharp - p_2^\sharp| .
\]
We also have that
\begin{align*}
\lambda \xi &= \lambda \max\{0, |\hat{p}_1^\sharp - \hat{p}_2^\sharp| - 8 T^{-1/5}\} \\& \geq 
 \lambda \max\{0,  |p_1^\sharp - p_2^\sharp| -8 T^{-1/5} -8 T^{-1/5}\} \geq  \max\{0,  \lambda|p_1^\sharp - p_2^\sharp| -16 T^{-1/5}\} . \square
\end{align*}
\end{proof}

The following lemma shows that our discretization scheme always guarantees that there is a price check point to approximate the  constrained optimal prices.

\begin{lemma}\label{lem:exp-constrained-opt-tilde-j}
Suppose that $|\hat{p}_1^\sharp - p_1^\sharp| \leq  4 T^{-1/5}$ and $|\hat{p}_2^\sharp - p_2^\sharp| \leq  4 T^{-1/5}$, there exists $\tilde{j} \in \{1, 2, \dots, J\}$ such that both $p_1(\tilde{j}), p_2(\tilde{j}) \in [\underline{p}, \overline{p}]$ and $|p_1(\tilde{j}) - p_1^*| \leq 9 T^{-1/5}$, $|p_2(\tilde{j}) - p_2^*| \leq 9  T^{-1/5}$.
\end{lemma}
\begin{proof}{Proof.}
Consider $\tilde{j} = \arg\min_{j} |\ell_j - (p_1^* + p_2^*)/2|$, we have that $|\ell_{\tilde{j}} - (p_1^* + p_2^*)/2| \leq  T^{-1/5}$. Now, by Lemma~\ref{lem:exp-constrained-opt-price-gap} and Lemma~\ref{lem:exp-constrained-opt-xi-bound}, we have that $|\ell_{\tilde{j}} - \frac{\lambda \xi}{2} - p_1^*| \leq 9  T^{-1/5}$ and $|\ell_{\tilde{j}} + \frac{\lambda \xi}{2} - p_2^*| \leq 9  T^{-1/5}$. Therefore, we also have that $|p_1(\tilde{j}) - p_1^*|\leq 9  T^{-1/5}$ and $|p_2(\tilde{j}) - p_2^*|\leq 9  T^{-1/5}$. $\square$
\end{proof}

We now prove the following lemma for the (near-)optimality of the estimated constrained prices.

\begin{lemma} \label{lem:exp-constrained-opt-optimality-gap} Suppose that $|\hat{p}_1^\sharp - p_1^\sharp| \leq  4 T^{-1/5}$ and $|\hat{p}_2^\sharp - p_2^\sharp| \leq  4 T^{-1/5}$,
with probability $(1- 4(\overline{p}-\underline{p}) T^{-2})$ we have that $R_1(\hat{p}_1^*)+ R_2(\hat{p}_2^*) \geq R_1(p_1^*) + R_2(p_2^*) - (4+18  K) \overline{p} T^{-1/5}$.
\end{lemma}
\begin{proof}{Proof.}
By Azuma's inequality, for each $j \in \{1, 2, \dots, J\}$, with probability $1 - 4 T^{-3}$, it holds that
\begin{align}
|\hat{d}_1(j) - d_1(p_1(j))| \leq T^{-1/5} \qquad \text{and} \qquad |\hat{d}_2(j) - d_2(p_2(j))| \leq T^{-1/5}. \label{eq:lem-exp-constrained-opt-optimality-gap-1}
\end{align}
Therefore, by a union bound, Eq.~\eqref{eq:lem-exp-constrained-opt-optimality-gap-1} holds for each $j \in \{1, 2, \dots, J\}$ with probability at least $1 - 4J T^{-3} \geq 1 - 4 (\overline p - \underline p)  T^{-2}$. Conditioned on this event, we have that
\begin{align}
\forall j \in \{1, 2, \dots, J\}: \qquad |\hat{R}(j) - (R_1(p_1(j)) + R_2(p_2(j))| \leq 2 \overline{p} T^{-1/5}. \label{eq:lem-exp-constrained-opt-optimality-gap-2}
\end{align}
With Eq.~\eqref{eq:lem-exp-constrained-opt-optimality-gap-2}, and let $\tilde{j}$ be the index designated by Lemma~\ref{lem:exp-constrained-opt-tilde-j}, we have that 
\begin{align}
R_1(\hat{p}_1^*) + R_2(\hat{p}_2^*) &= R_1(p_1(j^*)) + R_2(p_2(j^*)) \geq \hat{R}(j^*) - 2\overline{p}T^{-1/5} \nonumber \\
&\geq \hat{R}(\tilde{j})- 2\overline{p}T^{-1/5} 
\geq R_1(p_1(\tilde{j})) + R_2(p_2(\tilde{j}))-  4\overline{p}T^{-1/5}. \label{eq:lem-exp-constrained-opt-optimality-gap-3}
\end{align}
By Lemma~\ref{lem:exp-constrained-opt-tilde-j} and Item~\ref{item:assumption-1-lipschitz} of Assumption~\ref{assumption:1}, we have that
\begin{align}
R_1(p_1(\tilde{j}))+ R_2(p_2(\tilde{j})) \geq R_1(p_1^*) + R_2(p_2^*) - 2 \times 9  T^{-1/5} \times \overline{p} K. \label{eq:lem-exp-constrained-opt-optimality-gap-4}
\end{align}
Combining Eq.~\eqref{eq:lem-exp-constrained-opt-optimality-gap-3} and Eq.~\eqref{eq:lem-exp-constrained-opt-optimality-gap-4}, we prove the lemma. $\square$
\end{proof}

We are now ready to prove Theorem~\ref{thm:exp-constrained-opt}. Note that the sample complexity is upper bounded due to Lemma~\ref{lem:exp-constrained-opt-sample-complexity}. So long as $|\hat{p}_1^\sharp - p_1^\sharp| \leq  T^{-1/5}$ and $|\hat{p}_2^\sharp - p_2^\sharp| \leq  T^{-1/5}$, the price fairness is always satisfied due to the first inequality shown in Lemma~\ref{lem:exp-constrained-opt-xi-bound}. Finally, the (near-)optimality of the estimated prices $\hat{p}_1^*$ and $\hat{p}_2^*$ is guaranteed by Lemma~\ref{lem:exp-constrained-opt-optimality-gap}. 

\section{Proof of Theorem~\ref{thm:main_upper}}
\begin{proof}{Proof of Theorem~\ref{thm:main_upper}.}
The proof will be carried out conditioned on the desired events of both Theorem~\ref{thm:exp-unconstrained-opt} and Theorem~\ref{thm:exp-constrained-opt}, which happens with probability at least $(1 - O(T^{-1}))$. We can first easily verify that the first 2 steps of Algorithm~\ref{alg:price-fairness-main} satisfy the fairness constraint; and the 3rd step also satisfies the fairness constraint since $|\hat{p}_1^* - \hat{p}_2^*| \leq \lambda |p_1^\sharp - p_2^\sharp|$ by Theorem~\ref{thm:exp-constrained-opt}. We then turn to bound the regret of the algorithm. Since the first two steps use at most $O(T^{\frac45} \log^2 T)$ time periods, they incur at most $O(T^{\frac45} \log^2 T)$ regret. By the desired event of Theorem~\ref{thm:exp-constrained-opt}, the regret incurred by the third step is at most $T \times O(T^{-\frac15}) = O(T^\frac45)$.   $\square$
\end{proof}

\section{Proof Omitted in Section~\ref{sec:lower}}

\subsection{Pinsker's Inequality}
\begin{lemma}\label{lemma:pinsker}
If $P$ and $Q$ are two probability distributions on a measurable space $(X, \Sigma)$, then for any event $A \in \Sigma$, it holds that
\[
\left| P(A) - Q(A) \right| \leq \sqrt{\frac{1}{2} \mathrm{KL}(P \| Q)},
\]
where 
\[ 
\mathrm{KL}(P \Vert Q) = \int_X \left(\ln \dv{P}{Q}\right) \dd{P}
\]
is the Kullback--Leibler divergence.
\end{lemma}

\subsection{Proof of Lemma~\ref{lemma:lb-demand-profit-properties}}
For readers' convenience, we restate the claims of Lemma~\ref{lemma:lb-demand-profit-properties} and present the proofs immediately after each claim.
\begin{enumerate}
	\item $d_i(p) \in [1/20, 1/4]$ for all $i \in \{1, 2, 3\}$ and $p \in [1, 2]$.
	\begin{proof}{Proof.}
		When $A \geq 10$ and $h \in (0,1)$, one can easily calculate that  that $d_i(p) \in [1/10, 1/4]$ when $p \in [1,2]$. $\square$
	\end{proof}
	\item $d_i(p)$ and $R_i(p)$ are continuously differentiable functions for all $i \in \{1, 2, 3\}$ and $p \in [1, 2]$.
	\begin{proof}{Proof.}
		Note that $R_i(p) = p \cdot d_i(p)$, thus we only need to prove $R_i(p)$ is a continuously differentiable function for each $i \in {1,2,3}$ and $p \in [1,2]$. Note that this is obviously true for  $R_1(p)$ and $R_3(p)$. Thus we only need to prove $R_2(p)$ is a continuously differentiable function on $[1,2]$. 
		
		To prove $R_2(p)$ is continuously differentiable at $p = 1 + \frac{5\sqrt{h}}{4}$, we only need to prove 
		$R_2((1 + \frac{5\sqrt{h}}{4})_{-}) = R_2((1 + \frac{5\sqrt{h}}{4})_{+})$
and
		$\partial_{-} R_2(1 + \frac{5\sqrt{h}}{4})) = \partial_{+} R_2(1 + \frac{5\sqrt{h}}{4})$.
		
		Recall that
		\[
		R_2(p) =
		\left\{ \begin{aligned}
			& \frac{1}{4} - \frac{1}{2A}(p - 1 + \frac{\sqrt{h}}{4})^2, & & p \in [1, 1 + \frac{5\sqrt{h}}{4}) \\
			& \frac{1}{4} - \frac{3}{2A}(p - 1 - \frac{3\sqrt{h}}{4})^2 - \frac{3h}{4A}, & & p \in [1 + \frac{5\sqrt{h}}{4}, 1 + \frac{7\sqrt{h}}{4}) \\
			& \frac{1}{4} - \frac{1}{A}(p - 1 - \frac{\sqrt{h}}{4})^2, & & p \in [1 + \frac{7\sqrt{h}}{4},2]
			\end{aligned}
			\right. .
		\]
			
		Let $G_1(p) = \frac{1}{4} - \frac{1}{2A}(p - 1 + \frac{\sqrt{h}}{4})^2$, $G_2(p) = \frac{1}{4} - \frac{3}{2A}(p - 1 - \frac{3\sqrt{h}}{4})^2 - \frac{3h}{4A}$, $G_3(p) = \frac{1}{4} - \frac{1}{A}(p - 1 - \frac{\sqrt{h}}{4})^2$, this means we only need to prove $G_1(p) = G_2(p)$ and $G'_1(p) = G'_2(p)$ when $p = 1 + \frac{5\sqrt{h}}{4}$, and prove that $G_2(p) = G_3(p)$ and $G'_2(p) = G'_3(p)$ when $p = 1 + \frac{7\sqrt{h}}{4}$.
		
		When $p = 1 + \frac{5\sqrt{h}}{4}$,
		\begin{align*}
			G_1(p) - G_2(p) & = \left[\frac{1}{4} - \frac{1}{2A}(p - 1 + \frac{\sqrt{h}}{4})^2\right] - \left[\frac{1}{4} - \frac{3}{2A}(p - 1 - \frac{3\sqrt{h}}{4})^2 - \frac{3h}{4A}\right] = 0,\\
			G'_1(p) - G'_2(p) & = \frac{4 - 4p - \sqrt{h}}{4A} - \frac{9\sqrt{h} + 12 - 12p}{4A} = 0.
		\end{align*}
		
		Similarly, when $p = 1 + \frac{7\sqrt{h}}{4}$,
		\begin{align*}
			G_2(p) - G_3(p) & = \left[\frac{1}{4} - \frac{3}{2A}(p - 1 - \frac{3\sqrt{h}}{4})^2 - \frac{3h}{4A}\right] - \left[\frac{1}{4} - \frac{1}{A}(p - 1 - \frac{\sqrt{h}}{4})^2 \right] = 0, \\
			G'_2(p) - G'_3(p) & = \frac{9\sqrt{h} + 12 - 12p}{4A} - \frac{8 + 2\sqrt{h} - 8p }{4A} = 0.
		\end{align*}
		
		This means $R_2(p)$ is a twice differentiable function in $p \in [1,2]$. $\square$
	\end{proof}
	\item For each $p \in [1,2], i \in \{1, 2, 3\}$, $\frac{\partial d_i}{\partial p} < -\frac{1}{40} < 0$, and $R_i$ is strongly concave as a function of $d_i$.
	\begin{proof}{Proof.}
		Since $0 < h \leq 0.01$ and $20 \leq A \leq 30$, we have that
		\begin{align*}
			\frac{\partial  d_1}{\partial p} & = \frac{-4A + h + 8\sqrt{h} - 16p^2 + 16}{16A p^2} \leq - \frac{1}{16} + \frac{0.81}{320} < - \frac{1}{30} < 0; \\
			\frac{\partial d_2}{\partial p} & = \left\{ \begin{aligned} 
				& \frac{-8A + h - 8\sqrt{h} - 16p^2 + 16}{32A p^2} \leq -\frac{1}{16} + \frac{0.01}{640} < -\frac{1}{30}, && p \in [1, 1 + \frac{5\sqrt{h}}{4}) \\
				& \frac{-8A + 51h + 72\sqrt{h} - 48p^2 + 48}{32Ap^2} \leq -\frac{1}{16} + \frac{7.71}{640} < - \frac{1}{30},	&& p \in [1 + \frac{5\sqrt{h}}{4}, 1 + \frac{7\sqrt{h}}{4})\\
		&\frac{-4A + h + 8\sqrt{h} - 16p^2 + 16}{16A p^2} \leq - \frac{1}{16} + \frac{0.81}{320} < - \frac{1}{30} ,	&& p \in [1 + \frac{7\sqrt{h}}{4}, 2) \\
			\end{aligned}\right.; \\
			\frac{\partial d_3}{\partial p} & = \frac{- A - 8p^2 + 32}{8Ap^2} < \frac{-\frac{1}{8} + \frac{24}{8A}}{p^2} < - \frac{1}{40}.
		\end{align*}
		Similarly we have 
		\begin{align*}
			\frac{\partial  d_1}{\partial p} & = \frac{-4A + h + 8\sqrt{h} - 16p^2 + 16}{16A p^2} \geq -\frac{1}{4} - \frac{1}{10} = -\frac{7}{20}; \\
			\frac{\partial d_2}{\partial p} & = \left\{ \begin{aligned} 
				& \frac{-8A + h - 8\sqrt{h} - 16p^2 + 16}{32A p^2} \geq  -\frac{1}{4} - \frac{1}{20} = -\frac{3}{10}, && p \in [1, 1 + \frac{5\sqrt{h}}{4}) \\
				& \frac{-8A + 51h + 72\sqrt{h} - 48p^2 + 48}{32Ap^2} \geq  -\frac{1}{4} - \frac{3}{20} = -\frac{2}{5},	&& p \in [1 + \frac{5\sqrt{h}}{4}, 1 + \frac{7\sqrt{h}}{4})\\
		&\frac{-4A + h + 8\sqrt{h} - 16p^2 + 16}{16A p^2} \geq - \frac{1}{4} - \frac{1}{10} =  -\frac{7}{20} ,	&& p \in [1 + \frac{7\sqrt{h}}{4}, 2) \\
			\end{aligned}\right.\\
			\frac{\partial d_3}{\partial p} & = \frac{- A - 8p^2 + 32}{4Ap^2} \geq -\frac{1}{8} - \frac{1}{10} = -\frac{9}{40}.
		\end{align*}
		Thus, for each $p \in [1,2], i \in \{ 1, 2, 3\}$, 
		
		$$- \frac{2}{5} \leq \frac{\partial d}{\partial p} < -\frac{1}{40} < 0.$$

	To prove that  $R_i$ is strongly concave as a function of $d_i$ (for each $i \in \{1, 2, 3\}$), we only need to show that the second-order semi-derivatives are upper bounded by a negative constant (which will be $-\frac{1}{12}$ in the following proof).
	
	Note that whenever $R_i$ is twice-differentiable with respect to $d_i$, we have that
	\[
	\frac{\partial^2 R_i}{(\partial d_i)^2} = \frac{\partial }{\partial p}\left(\frac{\partial R_i}{\partial d_i}\right)\left[ \frac{\partial p}{\partial d_i} \right] .
	\]

	For $i = 1$, since $R_1$ is continuously twice-differentiable with respect to $d_i$, we only need to check that (while noticing that $20 \leq A \leq 30$ and $0< h < 0.01$)
		\begin{align*}
			\frac{\partial}{\partial p} \frac{\partial R_1}{\partial d_1} & = \frac{\partial}{\partial p} \left[\left(\frac{\partial R_1}{\partial p}\right) \left( \frac{\partial p}{\partial d_1} \right)\right] \\
			& = \frac{\partial}{\partial p}\left[ \left(\frac{\sqrt{h} - 4p + 4}{2A}\right)\left(\frac{16A p^2}{-4A + h + 8\sqrt{h} - 16p^2 + 16} \right)\right] \\
			& = \frac{16p((\sqrt{h} - 4p + 4)^2(\sqrt{h} + 2p + 4) - 4A(\sqrt{h} - 6p + 4))}{(4A  - h - 8\sqrt{h} + 16p^2 - 16)^2} > 0\\
			& > \frac{16\cdot(0 - 4 \cdot 20 \cdot (0.1 - 6 + 4))}{(4 \cdot 30 + 16 \cdot 4 - 16)^2} \\
			& > \frac{1}{20}.
		\end{align*}
	And therefore,
	\[
	\frac{\partial^2 R_1}{(\partial d_1)^2} = \frac{\partial }{\partial p}\left(\frac{\partial R_1}{\partial d_1}\right)\left[ \frac{\partial p}{\partial d_1} \right] \leq \frac{1}{20} \left(- \frac{5}{2}\right) \leq -\frac{1}{8}.
	\]

	For $i = 3$,  we have that
		\begin{align*}
			\frac{\partial}{\partial p} \frac{\partial R_3}{\partial d_3} & = \frac{\partial}{\partial p} \left[\left(\frac{\partial R_3}{\partial p}\right) \left( \frac{\partial p}{\partial d_3} \right)\right] \\
			& = \frac{\partial}{\partial p} \left[\left( \frac{4-2p}{A} \right) \left( \frac{8Ap^2}{- A - 8p^2 + 32}\right) \right]\\
			& = \frac{16p(A(3p - 4) + 8(p + 4)(p -2)^2)}{(A + 8p^2 - 32)^2} \\
			& \geq \frac{8}{45},
		\end{align*}
	where the last inequality can be verified for $A \in [20, 30]$.
    Therefore,
	\[
	\frac{\partial^2 R_3}{\partial (d_3)^2} = \frac{\partial }{\partial p}\left(\frac{\partial R_3}{\partial d_3}\right)\left[ \frac{\partial p}{\partial d_3} \right] \leq \frac{8}{45} \left(- \frac{5}{2}\right) \leq -\frac{4}{9}.
	\]

	Finally, for $i=2$, we calculate the second-order derivatives for every interval where $R_i$ admits continuously second-order derivative with respect to $d_i$.
	
	When $p \in [1, 1 + \frac{5\sqrt{h}}{4})$, we have that
	\begin{align*}
	\frac{\partial}{\partial p} \frac{\partial R_2}{\partial d_2} & = \frac{\partial}{\partial p} \left[\left(\frac{\partial R_2}{\partial p}\right) \left( \frac{\partial p}{\partial d_2} \right)\right] \\
	& = \frac{\partial}{\partial p} \left[ \left(  \frac{4 - 4p  - \sqrt{h}}{4A}\right) \left( \frac{32A p^2}{-8A + h - 8\sqrt{h} - 16p^2 + 16}\right)\right] \\
	& = \frac{16p((\sqrt{h} + 4p - 4)^2 (2(p + 2) - \sqrt{h}) + 8A(\sqrt{h} + 6p - 4) )}{(8A - h  + 8\sqrt{h} + 16p^2 - 16)^2} > 0\\
	& > \frac{16 \cdot(0 + 8 \cdot 20 \cdot 2 )}{(8 \cdot 30 + 16 \cdot 4 - 16)^2} \\
	& > \frac{1}{30}.
	\end{align*}
		When $p \in (1 + \frac{5\sqrt{h}}{4}, 1 + \frac{7\sqrt{h}}{4})$, we have that
		\begin{align*}
			\frac{\partial}{\partial p} \frac{\partial R_2}{\partial d_2} & = \frac{\partial}{\partial p} \left[\left(\frac{\partial R_2}{\partial p}\right) \left( \frac{\partial p}{\partial d_2} \right)\right] \\
			& = \frac{\partial}{\partial p} \left[ \left( \frac{9\sqrt{h} - 12p + 12}{4A} \right)\left(\frac{32Ap^2}{-8A + 51h + 72\sqrt{h} - 48p^2 + 48} \right)\right] \\
			& = \frac{48p((6p - 3\sqrt{h} - 4)(8A - 51h - 82\sqrt{h} - 48) + 96p^3)}{(8A - 51h - 72\sqrt{h} + 48p^2 - 48)^2} > 0\\
			& > \frac{48 \cdot ( 2 \cdot (8 \cdot 20 - 0.51 - 0.82 - 48)+ 96)}{(8 \cdot 30 + 48 \cdot 4 - 48)^2} \\
			& > \frac{1}{30}.
		\end{align*}
		
		When $p \in (1 + \frac{7\sqrt{h}}{4}, 2]$, $R_2(p) = R_1(p)$, which means that $\frac{\partial}{\partial p} \frac{\partial R_2}{\partial d_2} > \frac{1}{20} > \frac{1}{30}$.
		
		Therefore, for each $p \in [1, 1 + \frac{5\sqrt{h}}{4}) \cup (1+\frac{5\sqrt{h}}{4}, (1 +1 + \frac{7\sqrt{h}}{4}) \cup \frac{7\sqrt{h}}{4}, 2]$, we have that
		\[
		\frac{\partial^2 R_2}{\partial d_2^2} = \frac{\partial }{\partial p}\left(\frac{\partial R_2}{\partial d_2}\right)\left[ \frac{\partial p}{\partial d_2} \right] \leq \frac{1}{30} \left(- \frac{5}{2}\right) \leq -\frac{1}{12}.
		\]
		This also means that the second-order semi-derivatives of $R_2$ with respect to $d_2$ at $p \in \{1 + \frac{5\sqrt{h}}{4}, 1 + \frac{7\sqrt{h}}{4}\}$ are also upper bounded by $-\frac{1}{12}$. $\square$
	\end{proof}
	\item For each $p \in [1, 1 + \frac{7\sqrt{h}}{4}]$, it holds that $|d_1(p) - d_2(p)| \leq \frac{h}{4A}$.
	\begin{proof}{Proof.}
	Since
		\begin{align*}
			|d_2(p) - d_1(p)| = \frac{|R_2(p) - R_1(p)|}{p} \leq |R_2(p) - R_1(p)|,
		\end{align*}
		we only need to show that $|R_2(p) - R_1(p)| \leq \frac{h}{4A}$, which can be verified by discussing the two cases that $p \in [1, 1 + \frac{5\sqrt{h}}{4})$ and $p \in [1 + \frac{5\sqrt{h}}{4}, 1 + \frac{7\sqrt{h}}{4})]$. $\square$
	\end{proof}

	
		
		
		

	\item For each $p \in [1, 1 + \frac{7\sqrt{h}}{4}]$, it holds that $D_{\mathrm{KL}}(\mathrm{Ber}(d_1(p))\| \mathrm{Ber}(d_2(p))) \leq 5h^2/3A^2$.
	\begin{proof}{Proof.}
		From item $(a)$ we know that $d_1(p)\in [\frac{1}{20}, \frac{1}{4}]$ and $d_2(p) \in [\frac{1}{20}, \frac{1}{4}]$. Thus,
		\begin{align*}
			& \quad D_{\mathrm{KL}}(\mathrm{Ber}(d_1(p) || \mathrm{Ber}(d_2(p)))\\
			& = d_1(p)\log\frac{d_1(p)}{d_2(p)} + (1 - d_1(p))\log\frac{1 - d_1(p)}{1 - d_2(p)}\\
			& = d_1(p) \log \left( 1 + \frac{d_1(p) - d_2(p)}{d_2(p)}\right) + (1 - d_1(p))\log \left(  1 + \frac{d_ 2(p) - d_1(p)}{1 - d_2(p)}\right)\\
			& \leq \left(d_1(p) - d_2(p)\right)\left( \frac{d_1(p)}{d_2(p)}\frac{1 - d_1(p)}{1 - d_2(p)}\right)\\
			& = \frac{(d_1(p) - d_2(p))^2}{d_2(p)(1 - d_2(p))} .
		\end{align*}
		
		From item (d), $|d_1(p) - d_2(p)| \leq \frac{h}{4A}$, thus 
		\begin{align*}
			D_{\mathrm{KL}}(\mathrm{Ber}(d_1(p) || \mathrm{Ber}(d_2(p))) \leq \frac{\frac{h^2}{16A^2}}{\frac{1}{20}(1 - \frac{1}{4})} \leq \frac{5h^2}{3A^2}. 
		\end{align*} 
		$\square$
	\end{proof}
	\item For any demand rate function $d(p)$ defined on $p \in [1, 2]$, let $p^\sharp(d) = \argmax_{p \in [1, 2]} \{p \cdot d(p)\}$ be the unconstrained clairvoyant solution; we have that $p^\sharp(d_1) = 1 + \frac{\sqrt{h}}{4}$, $p^\sharp(d_2) = 1$, and $p^\sharp(d_3) = 2$.
	\begin{proof}{Proof.}
		Since $R_1(p)$ and $R_3(p)$ are quadratic functions, one can easily show that $p^\sharp(d_1) = 1 + \frac{\sqrt{h}}{4}$ and $p^\sharp(d_3) = 2$. It is also straightforward to verify that $R_2(p)$ is monotonically decreasing when $p \in [1,2]$. Thus $p^\sharp(d_2) = 1$. $\square$
	\end{proof}
\end{enumerate}

\subsection{Proof of Lemma~\ref{lemma:lb-optimal-fairness-aware-solution}}
\begin{proof}{Proof of Lemma~\ref{lemma:lb-optimal-fairness-aware-solution}.}
We first compute $p^*(d_1; \mathcal I)$ and $p^*(d_3; \mathcal I)$. By the unimodality of $R_1(\cdot)$ and $R_3(\cdot)$, one can easily verify that  $p^\sharp(d_1)  \leq p^*(d_1; \mathcal I) \leq p^*(d_3; \mathcal I) \leq p^\sharp(d_3)$ and
\[
|p^*(d_1; \mathcal I) - p^*(d_3; \mathcal I)| = \lambda |p^\sharp(d_1) - p^\sharp(d_3)|.
\]
Let $\Delta := |p^\sharp(d_3) - p^\sharp(d_1)|$. We have that
\begin{align*}
p^*(d_1; \mathcal I) & =  \argmax_{p \in[p^\sharp(d_1), p^\sharp(d_3) - \lambda \Delta]} \left\{R_1(p) + R_3(p + \lambda \Delta )\right\} \\
& = \argmax_{p \in[p^\sharp(d_1), p^\sharp(d_3) - \lambda \Delta]} \left\{ \frac{3}{8} - \frac{1}{A}\left((p - 1 - \frac{\sqrt{h}}{4})^2 + (p + \lambda \Delta - 2)^2\right)\right\} \\
& = \argmax_{p \in[p^\sharp(d_1), p^\sharp(d_3) - \lambda \Delta]} \left\{ \frac{3}{8} - \frac{1}{A}\left((p - p^\sharp(d_1))^2 + (p + \lambda \Delta - p^\sharp(d_3))^2\right)\right\} \\
& =   \frac{1+\lambda }{2} p^\sharp(d_1) + \frac{1 - \lambda}{2} p^\sharp(d_3) ,
\end{align*}
where the third inequality uses Item~\ref{item:lb-dpp-f} of Lemma~\ref{lemma:lb-demand-profit-properties}. We then immediately get
\begin{align*}
p^*(d_3; \mathcal I)  =  p^*(d_1; \mathcal I) + \lambda \Delta  
 =  \frac{1-\lambda}{2} p^\sharp(d_1) + \frac{1 + \lambda}{2} p^\sharp(d_3).
\end{align*}

We then proceed to compute $p^*(d_2; \mathcal I')$ and $p^*(d_3; \mathcal I')$. Similarly, we have that $p^\sharp(d_2)  \leq p^*(d_2; \mathcal I') \leq p^*(d_3; \mathcal I') \leq p^\sharp(d_3)$ and 
\[
|p^*(d_2; \mathcal I') - p^*(d_3; \mathcal I')| = \lambda |p^\sharp(d_2) - p^\sharp(d_3)|.
\]
Let $\Delta' := |p^\sharp(d_3) - p^\sharp(d_2)| = 1$, and we have that
\begin{align*}
p^*(d_2; \mathcal I') & =  \argmax_{p \in[p^\sharp(d_2), p^\sharp(d_3) - \lambda \Delta']} \left\{R_2(p) + R_3(p + \lambda \Delta' )\right\}.
\end{align*}

Since $R_2(p) \leq \frac{1}{4}$ and $R_3(p)$ is monotonically increasing when $p \in [1,2]$, we have that
\begin{align}
\max_{p \in [1,1 + \frac{7\sqrt{h}}{4} ]} \{R_2(p)+ R_3(p + \lambda \Delta')\} & \leq
 \frac{1}{4} + R_3\left(1 + \frac{\sqrt{7h}}{4} + \lambda \Delta' \right)  
 = \frac{3}{8} - \frac{1}{A}\left(1 + \frac{7\sqrt{h}}{4} + \lambda - 2\right)^2  \label{r2-first-part-value-bound}.  
\end{align}
When $p = 1 + \frac{7\sqrt{h}}{2}$, note that
\begin{align}
\left[R_2(p) + R_3(p + \lambda \Delta')\right]\big|_{p = 1+\frac{7\sqrt{h}}{2}}
& = R_2\left(1 + \frac{7\sqrt{h}}{2}\right) + R_3\left(1 + \frac{7\sqrt{h}}{2}+ \lambda \Delta'\right) \nonumber \\
& = \frac{3}{8} - \frac{1}{A}\left(\frac{13\sqrt{h}}{4}\right)^2 - \frac{1}{A}\left(1 + \frac{7\sqrt{h}}{2} + \lambda - 2\right)^2 .  \label{r2-fix-point-value} 
\end{align}
Combining Eq.~\eqref{r2-first-part-value-bound} and Eq.~\eqref{r2-fix-point-value}, we have that
\begin{align*}
&\quad \left[R_2(p) + R_3(p + \lambda \Delta')\right]\big|_{p = 1+\frac{7\sqrt{h}}{2}} - \max_{p \in [1,1 + \frac{7\sqrt{h}}{4} ]} \{R_2(p)+ R_3(p + \lambda \Delta')\}\\
&\geq \frac{1}{A}\left(\left(1 + \frac{7\sqrt{h}}{4} + \lambda -2\right)^2 - \left(\frac{13\sqrt{h}}{4}\right)^2 - \left(1 + \frac{7\sqrt{h}}{2} + \lambda - 2\right)^2 \right) \\
& = \frac{\sqrt{h}}{4A}\left(14 - 14\lambda -79 \sqrt{h}\right) \geq 0,
\end{align*}
where the last inequality is because  $\lambda \leq 1 - \epsilon$ and $h \leq \frac{\epsilon^2}{40}$. Note that $h \leq \frac{\epsilon^2}{40}$ also implies that $1 + \frac{7\sqrt{h}}{2} \leq p^\sharp(d_3) - \lambda \Delta'$ is a valid candidate for $p^*(d_2; \mathcal I')$.
Therefore, we have that $p^*(d_2; \mathcal I') \geq 1 + \frac{7\sqrt{h}}{4}$, and  
\begin{align*}
p^*(d_2; \mathcal I') & =  \argmax_{p \in[1 + \frac{7\sqrt{h}}{4}, p^\sharp(d_3) - \lambda \Delta']} \left\{  R_2(p) + R_3(p + \lambda \Delta' ) \right\} \\
& =  \argmax_{p \in[1 + \frac{7\sqrt{h}}{4}, p^\sharp(d_3) - \lambda \Delta']} \left\{  R_1(p) + R_3(p + \lambda \Delta' ) \right\} \\
& =  \argmax_{p \in[1 + \frac{7\sqrt{h}}{4},p^\sharp(d_3) - \lambda \Delta']} \left\{ \frac{1}{2} - \frac{1}{A}\left((p - p^\sharp(d_1))^2 + (p + \lambda \Delta' - p^\sharp(d_3))^2\right)\right\}\\
& = \frac{1}{2}(p^\sharp(d_1) + p^\sharp(d_3)) - \frac{\lambda}{2}(p^{\sharp}(d_3) - p^\sharp(d_2)) .
\end{align*}
Finally, we have that
\begin{align*}
p^*(d_3; \mathcal I') & =  p^*(d_2; \mathcal I') + \lambda \Delta'  =  \frac{1}{2}(p^\sharp(d_1) + p^\sharp(d_3)) + \frac{\lambda}{2}(p^{\sharp}(d_3) - p^\sharp(d_2)). 
\end{align*}



\end{proof}

\subsection{Proof of Lemma~\ref{lemma:lb-regret-fairness-aware}}
\begin{proof}{Proof of Lemma~\ref{lemma:lb-regret-fairness-aware}.}
Similar to the proof of Lemma~\ref{lemma:lb-optimal-fairness-aware-solution}, we define $\Delta := |p^\sharp(d_3) - p^\sharp(d_1)| \leq 1$. By Lemma \ref{lemma:lb-optimal-fairness-aware-solution}, we have that
\[
R_1(p^*(d_1, \mathcal I)) = \frac{1}{4} - \frac{1}{A}\left(\frac{1-\lambda}{2} \Delta \right)^2, \qquad 
R_3(p^*(d_3, \mathcal I)) = \frac{1}{8} - \frac{1}{A}\left(\frac{1-\lambda}{2} \Delta \right)^2. 
\]

When $p_1 \in [1, 1 + \frac{7\sqrt{h}}{4}]$, we have that $p_3 \in [1,  1 + \frac{7\sqrt{h}}{4} + \lambda \Delta]$. Using the monotonicity of $R_3(\cdot)$ and the fact that $R_1(p_1) \leq 1/4$, we have that 
\begin{align}
&R_1(p^*(d_1, \mathcal I)) + R_2(p^*(d_3, \mathcal I))  - R_1(p_1) - R_3(p_3) \nonumber \\
 & \geq \frac{3}{8} - \frac{2}{A}\left(\frac{1 - \lambda}{2}\Delta \right)^2 - \frac{1}{4} - R_3\left(1 + \frac{7\sqrt{h}}{4} + \lambda \Delta \right)
 =  \frac{1}{A}\left(1 - \frac{7\sqrt{h}}{4} - \lambda \Delta\right)^2 - \frac{\Delta^2}{2A}\left( 1-\lambda \right)^2 . \label{eq-proof-lemma-lb-regret-fairness-aware-1}
 \end{align}
Note that
\begin{align}
\left(1 - \frac{7\sqrt{h}}{4} - \lambda \Delta\right)^2 - \frac{\Delta^2}{2}\left( 1-\lambda \right)^2 & = (1 - \lambda \Delta)^2 - \frac{1}{2} (\Delta - \lambda \Delta)^2 - \frac{7\sqrt{h}}{2} (1 - \lambda \Delta) + \frac{49h}{16} \nonumber \\
& \geq \frac{1}{2} (1 - \lambda \Delta)^2 -  \frac{7\sqrt{h}}{2} (1 - \lambda \Delta) + \frac{49h}{16}\nonumber  \\
& \geq \frac{1}{2} \epsilon^2 -  \frac{7\sqrt{h}}{2} + \frac{49h}{16} \geq \frac{\epsilon^2}{4}, \label{eq-proof-lemma-lb-regret-fairness-aware-2}
\end{align}
where in the first two inequalities, we used $\lambda \in (0, 1 - \epsilon)$ and $\Delta \in (0, 1]$, in the last inequality, we used that $h \leq \epsilon^4 / 400$. Combining Eq.~\eqref{eq-proof-lemma-lb-regret-fairness-aware-1} and Eq.~\eqref{eq-proof-lemma-lb-regret-fairness-aware-2}, we prove that
\[
R_1(p^*(d_1, \mathcal I)) + R_2(p^*(d_3, \mathcal I))  - R_1(p_1) - R_3(p_3) \geq \frac{\epsilon^2}{4A}.
\]

The second part of the lemma can be proved in the same way. $\square$
\end{proof}

\subsection{Proof of Lemma~\ref{lemma:lb-regret-I-prime}}
\begin{proof}{Proof of Lemma~\ref{lemma:lb-regret-I-prime}.}
Similar to the proof of Lemma~\ref{lemma:lb-optimal-fairness-aware-solution}, we define $\Delta := |p^\sharp(d_3) - p^\sharp(d_1)|$ and $\Delta' := |p^\sharp(d_3) - p^\sharp(d_2)|$.

Our assumption that $(p_2, p_3)$ satisfies the fairness condition of $\mathcal I$ is equivalent to that $|p_3 - p_2| \leq \lambda \Delta$.
By the unimodality of $R_2(\cdot)$ and $R_3(\cdot)$, we have that
\begin{align}
\max_{p_2, p_3: |p_2 - p_3| \leq \lambda \Delta} \{ R_2(p_2) + R_3(p_3) \} = \max_{p \in [p^\sharp(d_2), p^\sharp(d_3) - \lambda \Delta]} \{R_2(p) + R_3(p + \lambda \Delta) \} . \label{eq:lemma-lb-regret-I-prime-1}
\end{align}
Similar to Eq.~\eqref{r2-first-part-value-bound} and Eq.~\eqref{r2-fix-point-value} in the proof of Lemma~\ref{lemma:lb-optimal-fairness-aware-solution}, we have that
\begin{align*}
&\quad \left[R_2(p) + R_3(p + \lambda \Delta)\right]\big|_{p = 1+\frac{7\sqrt{h}}{2}} - \max_{p \in [1,1 + \frac{7\sqrt{h}}{4} ]} \{R_2(p)+ R_3(p + \lambda \Delta)\}\\
&\geq \frac{1}{A}\left(\left(1 + \frac{7\sqrt{h}}{4} + \lambda \Delta -2\right)^2 - \left(\frac{13\sqrt{h}}{4}\right)^2 - \left(1 + \frac{7\sqrt{h}}{2} + \lambda \Delta - 2\right)^2 \right) \\
& = \frac{\sqrt{h}}{4A}\left(14 - 14\lambda \Delta -79 \sqrt{h}\right) \geq 0,
\end{align*}
where the last inequality is because  $\lambda \Delta \leq 1 - \epsilon$ and $h \leq \frac{\epsilon^2}{40}$. Note that $h \leq \frac{\epsilon^2}{40}$ also implies that $1 + \frac{7\sqrt{h}}{2} \leq p^\sharp(d_3) - \lambda \Delta$ is a valid solution to the RHS of Eq.~\eqref{eq:lemma-lb-regret-I-prime-1}. Together with Eq.~\eqref{eq:lemma-lb-regret-I-prime-1}, we have that
\begin{align}
\max_{p_2, p_3: |p_2 - p_3| \leq \lambda \Delta}& \{ R_2(p_2) + R_3(p_3) \} =  \max_{p \in[1 + \frac{7\sqrt{h}}{4}, p^\sharp(d_3) - \lambda \Delta]} \left\{  R_2(p) + R_3(p + \lambda \Delta ) \right\} \nonumber \\
& =  \max_{p \in[1 + \frac{7\sqrt{h}}{4}, p^\sharp(d_3) - \lambda \Delta]} \left\{  R_1(p) + R_3(p + \lambda \Delta ) \right\} \nonumber \\
& =  \max_{p \in[1 + \frac{7\sqrt{h}}{4},p^\sharp(d_3) - \lambda \Delta]} \left\{ \frac{3}{8} - \frac{1}{A}\left((p - p^\sharp(d_1))^2 + (p + \lambda \Delta - p^\sharp(d_3))^2\right)\right\} \nonumber \\
& = \frac{3}{8} - \frac{1}{2A} (p^\sharp(d_3) - p^\sharp(d_1) - \lambda \Delta)^2 . \label{eq:lemma-lb-regret-I-prime-2}
\end{align}
By Lemma \ref{lemma:lb-optimal-fairness-aware-solution}, we compute that
\begin{align}
R_2(p^*(d_2; \mathcal I')) + R_3(p^*(d_3; \mathcal I')) &=  R_2(\frac{1}{2}(p^\sharp(d_1) + p^\sharp(d_3)) + \frac{\lambda\Delta'}{2}) + R_3(\frac{1}{2}(p^\sharp(d_1) + p^\sharp(d_3) - \frac{\lambda\Delta'}{2})\nonumber\\
&=\frac{3}{8} - \frac{1}{2A} (p^\sharp(d_3) - p^\sharp(d_1) - \lambda \Delta')^2 . \label{eq:lemma-lb-regret-I-prime-3}
\end{align}
Combining Eq.~\eqref{eq:lemma-lb-regret-I-prime-2} and Eq.~\eqref{eq:lemma-lb-regret-I-prime-3},  we conclude that
\begin{align*}
R_2(p^*(d_2; \mathcal I'))& + R_3(p^*(d_3; \mathcal I')) - [R_2(p_2) + R_3(p_3)]  \\
& \geq \frac{1}{2A} (p^\sharp(d_3) - p^\sharp(d_1) - \lambda \Delta)^2  - \frac{1}{2A} (p^\sharp(d_3) - p^\sharp(d_1) - \lambda \Delta')^2\\
& = \frac{\lambda}{2A} (2\Delta - \lambda(\Delta + \Delta')) (\Delta' - \Delta) \geq \frac{\lambda}{2A} \cdot 2(1-\lambda) \cdot \frac{\sqrt{h}}{4} \geq \frac{\epsilon \lambda \sqrt{h}}{4} . 
\end{align*}



\end{proof}

\subsection{Proof of Lemma~\ref{lemma:lb-KL-decomposition}}
\begin{proof}{Proof of Lemma~\ref{lemma:lb-KL-decomposition}.}
Let $\mathcal P_{\mathcal J, \pi}^{\leq t}$ be the probability measure over the first $t$ time periods of $\mathcal P_{\mathcal J, \pi}$ (for $\mathcal J \in \{\mathcal I, \mathcal I'\})$. Observe that $\mathcal P_{\mathcal I, \pi}^{\leq T}$ is exactly $\mathcal P_{\mathcal I, \pi}$. Therefore, to prove the lemma, it suffices to prove that, for any $t \in \{1, 2, 3, \dots, T\}$,
\begin{align}
D_{\mathrm{KL}}(\mathcal P_{\mathcal I, \pi}^{\leq t}\|\mathcal P_{\mathcal I', \pi}^{\leq t}) \leq D_{\mathrm{KL}}(\mathcal P_{\mathcal I, \pi}^{\leq t-1}\|\mathcal P_{\mathcal I', \pi}^{\leq t-1}) + \Pr_{\mathcal P_{\mathcal I, \pi}}\left[p^{(t)}(d_1; \mathcal I, \pi) \in \left[1, 1+\frac{7\sqrt{h}}{4}\right]\right] \cdot \frac{4h^2}{A^2} .
\label{eq:lemma-lb-KL-decomposition-1}
\end{align}
Let $H_{t}$ be the pricing and demand records for both customer groups during time periods $1, 2, \dots, t$, and we slightly abuse the notation by also denoting by $\mathcal P_{\mathcal J, \pi}^{\leq t}(H_t)$ the corresponding probability density for $\mathcal P_{\mathcal J, \pi}^{\leq t}$. Note that
\begin{align}
  & D_{\mathrm{KL}}(\mathcal P_{\mathcal I, \pi}^{\leq t}\|\mathcal P_{\mathcal I', \pi}^{\leq t})  = \E_{H_t \sim \mathcal P_{\mathcal I, \pi}^{\leq t}} \left[\ln \frac{\mathcal P_{\mathcal I, \pi}^{\leq t}(H_t)}{\mathcal P_{\mathcal I', \pi}^{\leq t}(H_t)}\right] \nonumber \\
  & \qquad =  \E_{(H_{t-1}, p_1^{(t)}, D_1^{(t)}, p_2^{(t)}, D_2^{(t)}) \sim \mathcal P_{\mathcal I, \pi}^{\leq t}} \left[\ln \frac{\mathcal P_{\mathcal I, \pi}^{\leq t-1}(H_{t-1})}{\mathcal P_{\mathcal I', \pi}^{\leq t-1}(H_{t-1})} + \ln \frac{\mathrm{Ber}(D_1^{(t)} | d_1(p_1^{(t)}))}{\mathrm{Ber}(D_1^{(t)}|d_2(p_1^{(t)}))}\right] \label{eq:lemma-lb-KL-decomposition-2} \\
  & \qquad = \E_{H_{t-1} \sim \mathcal P_{\mathcal I, \pi}^{\leq t-1}} \left[\ln \frac{\mathcal P_{\mathcal I, \pi}^{\leq t-1}(H_{t-1})}{\mathcal P_{\mathcal I', \pi}^{\leq t-1}(H_{t-1})}\right] + \E_{ p_1^{(t)} \sim \mathcal P_{\mathcal I, \pi}^{\leq t}} \left[D_{\mathrm{KL}}(\mathrm{Ber}(d_1(p_1^{(t)}))\|\mathrm{Ber}(d_2(p_1^{(t)})))\right]. \label{eq:lemma-lb-KL-decomposition-3}
\end{align}
In Eq.~\eqref{eq:lemma-lb-KL-decomposition-2}, $p_i^{(t)}$ and $D_i^{(t)}$ respectively denote the price for and the demand from the customer group $i$, and we use $\mathrm{Ber}(\cdot | \mu)$ to denote the probability mass of the Bernoulli distribution with parameter $\mu$. In  Eq.~\eqref{eq:lemma-lb-KL-decomposition-2}, we only consider the KL-divergence between the customer group $1$ because the demand rates and distributions of the customer group $2$ are the same for $\mathcal I$ and $\mathcal I'$. Note that the first term in Eq.~\eqref{eq:lemma-lb-KL-decomposition-3} is exactly $D_{\mathrm{KL}}(\mathcal P_{\mathcal I, \pi}^{\leq t-1}\|\mathcal P_{\mathcal I', \pi}^{\leq t-1})$, and we upper bound  second term in Eq.~\eqref{eq:lemma-lb-KL-decomposition-3}  by
\begin{align}
& \E_{ p_1^{(t)} \sim \mathcal P_{\mathcal I, \pi}^{\leq t}} \left[D_{\mathrm{KL}}(\mathrm{Ber}(d_1(p_1^{(t)}))\|\mathrm{Ber}(d_2(p_1^{(t)})))\right] \nonumber  \\
& \qquad \leq \Pr_{ p_1^{(t)} \sim \mathcal P_{\mathcal I, \pi}^{\leq t}}\left[p_1^{(t)} \in \left[1,1+\frac{7\sqrt{h}}{4}\right]\right] \cdot \sup_{p \in [1,1+\frac{7\sqrt{h}}{4}]} \left\{ D_{\mathrm{KL}}(\mathrm{Ber}(d_1(p))\|\mathrm{Ber}(d_2(p)))\right\} \nonumber \\
& \qquad \leq \Pr_ {p_1^{(t)} \sim P_{\mathcal I, \pi}^{\leq t}}\left[p_1^{(t)} \in \left[1,1+\frac{7\sqrt{h}}{4}\right]\right] \cdot \frac{4h^2}{A^2} , \label{eq:lemma-lb-KL-decomposition-4}
\end{align}
where the first inequality is because $d_1(p) = d_2(p)$ for all $p \in [1+\frac{7\sqrt{h}}{4}, 2]$, and the second inequality is due to Item (e) of Lemma~\ref{lemma:lb-demand-profit-properties}. Combining Eq.~\eqref{eq:lemma-lb-KL-decomposition-3} and Eq.~\eqref{eq:lemma-lb-KL-decomposition-4}, we prove Eq.~\eqref{eq:lemma-lb-KL-decomposition-1}, and therefore prove the lemma.  $\square$
\end{proof}
\subsection{Proof of Theorem~\ref{theorem:lb}}\label{subsec:proofoflbtheorem}
\begin{proof}{Proof of Theorem~\ref{theorem:lb}.}
We set $h = T^{-2/5}$ and $A = 10$. Note that when $T \geq (5/\eps)^{10}$, the assumptions of Lemmas~\ref{lemma:lb-regret-fairness-aware} and \ref{lemma:lb-regret-I-prime} are met. We now discuss the following two cases.

\noindent \underline{\it Case 1: $\sum_{t=1}^{T} \Pr_{\mathcal P_{\mathcal I, \pi}}[p^{(t)}(d_1; \mathcal I, \pi) \in [1, 1+\frac{7\sqrt{h}}{4}]] \geq A^2/(400 h^2)$.} Invoking Lemma~\ref{lemma:lb-regret-fairness-aware}, we have that the expected regret incurred by $\pi$ for instance $\mathcal I$ is at least
\[
\frac{A^2}{400 h^2} \cdot \frac{\epsilon^2}{4A} =\frac{\epsilon^2 A}{1600 h^2} = \frac{\epsilon^2 T^{4/5}}{160}.
\]

\noindent \underline{\it Case 2: $\sum_{t=1}^{T} \Pr_{\mathcal P_{\mathcal I, \pi}}[p^{(t)}(d_1; \mathcal I, \pi) \in [1, 1+\frac{7\sqrt{h}}{4}]] < A^2/(400 h^2)$.} Let $\mathcal E$ be the event that $\pi$ satisfies the fairness constraints for instance $\mathcal I$. By the assumption in our theorem statement, we have that
\begin{align}\label{eq:thm-lb-1}
\Pr_{\mathcal P_{\mathcal I, \pi}}[\mathcal E] \geq 0.9 .
\end{align}
Invoking Lemma~\ref{lemma:lb-KL-decomposition} and Pinsker's inequality (Lemma~\ref{lemma:pinsker}), we have that
\begin{align}\label{eq:thm-lb-2}
\left|\Pr_{\mathcal P_{\mathcal I, \pi}}[\mathcal E] - \Pr_{\mathcal P_{\mathcal I', \pi}}[\mathcal E] \right| \leq \sqrt{\frac{1}{2} \cdot \frac{A^2}{400h^2} \cdot \frac{4h^2}{ A^2}} \leq 0.1 .
\end{align}
Combining Eq.~\eqref{eq:thm-lb-1} and Eq.~\eqref{eq:thm-lb-2}, we have that 
\begin{align}
\Pr_{\mathcal P_{\mathcal I', \pi}}[\mathcal E] \geq 0.8 . \label{eq:thm-lb-3}
\end{align}
Note that Eq.~\eqref{eq:thm-lb-3} is very different from Eq.~\eqref{eq:thm-lb-1}. The probability lower bounded in Eq.~\eqref{eq:thm-lb-3} is that $\pi$ satisfies the fairness constraints for instance $\mathcal{I}$ (i.e., the $p^\sharp$'s are defined by $\mathcal{I}$) when we run  $\pi$ in \emph{instance $\mathcal I'$}. In contrast the probability concerned in Eq.~\eqref{eq:thm-lb-1} is for the same event when we run $\pi$ in \emph{instance $\mathcal I$}. While Eq.~\eqref{eq:thm-lb-1} is directly guaranteed by the assumption that $\pi$ is fairness-aware, Eq.~\eqref{eq:thm-lb-3} is not straightforward due to the mismatch between the underlying instance $\mathcal I'$ and the one that is used for defining the fairness constraints ($\mathcal I$). This is the reason that we crucially rely on Lemma~\ref{lemma:lb-KL-decomposition} and Eq.~\eqref{eq:thm-lb-2} to upper bound the distance between the probability measures induced by the two instances.

When $\mathcal E$ happens for instance $\mathcal I'$, by Lemma~\ref{lemma:lb-regret-I-prime}, we have that the regret incurred by the pricing strategy at each time step is at least $\frac{\epsilon \lambda \sqrt{h}}{4A}$. Therefore, the expected regret of $\pi$ for instance $\mathcal I'$ is at least
\[
\Pr_{\mathcal P_{\mathcal I', \pi}}[\mathcal E] \cdot T \cdot \frac{\epsilon \lambda \sqrt{h}}{4A} \geq 0.8 \cdot T \cdot  \frac{\epsilon^2 \sqrt{h}}{4A} \geq \frac{1}{5} \epsilon^2 T^{4/5}.
\]
Combining the two cases, we prove the theorem. $\square$
\end{proof}

\blue{
\section{Three-stage Framework may not Achieve Better-than-$T^{4/5}$ Regret for Linear Demands} \label{sec:three-stage-linear-demand}
This section explains that even with the linear demand functions, we may not hope to improve the $T^{4/5}$-type regret under the three-stage framework, as stated at the end of Section~\ref{sec:discussion-key-causes}.

We first explain that for the linear demand functions considered in \cite{Cohen:21:dynamic}, the optimal learning (convergence) rate of the unconstrained optimal prices is still $\tilde{\Theta}(a^{-1/4}(T))$ rather than $\tilde{\Theta}(a^{-1/2}(T))$ (i.e., the linear demand class does not help improve the learning of the constraint optimal prices).

By Theorem 1 of  \cite{keskin2014dynamic}, we know the lower bound of dynamic pricing with unknown linear demand over $\tilde{\Theta}(a(T))$ periods  is $\Omega(a^{1/2}(T))$, which implies that for any anticipate policy $\pi$, $$\mathcal R(p^\sharp)-\mathcal R(p_{a(T)}^{\pi})\geq \Omega(a^{-1/2}(T)).$$ 
Invoking the smoothness property of $\mathcal R(p)$, we have
\begin{align*}
    - \mathcal R(p_{a(T)}^{\pi})&\leq -\mathcal R(p^\sharp) - \nabla \mathcal R(p^\sharp)(p_{a(T)}^{\pi}-p^\sharp) + \frac{C}{2}(p_{a(T)}^{\pi}-p^\sharp)^2\\&= -\mathcal R(p^\sharp) + \frac{C}{2}(p_{a(T)}^{\pi}-p^\sharp)^2, 
\end{align*}
where $C$ is a positive absolute constant. Combining the above two inequalities, we could get the conclusion that for any anticipate policy $\pi$,
\[ |p_{a(T)}^{\pi}-p^\sharp|\geq \Omega(a^{-1/4}(T)).\]
Therefore, the regret of the aforementioned three-stage framework will ultimately incur regret $\tilde{O}(T^{4/5})$, just like the nonparametric demand.
}

\section{Proof Omitted in Section~\ref{sec:general}}
\subsection{Proof of Theorem~\ref{thm:exp-constrained-opt-general} for {\sc ExploreConstrainedOPTGeneral}} \label{sec:alg-exp-constrained-opt-general}

First, the following lemma upper bounds the number of the selling periods used by the algorithm:
	    
\begin{lemma}\label{lem:general-lemma1}
Algorithm~\ref{alg:exp-constrained-opt-general} uses at most $O(\overline{p}  T^\frac{3}{5} \log T)$ selling periods, where only a universal constant is hidden in $O(\cdot)$ notation.
\end{lemma}
\begin{proof}{Proof.}
For each price checkpoint $\ell_j$ the algorithm uses at most $6  T^\frac{2}{5} \ln T$ selling periods. Since there are $J = \lceil (\overline{p} - \underline{p}) T^\frac{1}{5} \rceil$ selling price checkpoints, the total number of selling periods used by the algorithm is at most $O(\overline{p}  T^\frac{3}{5} \log T)$. $\square$
\end{proof}
	
We then turn to upper bound the penalized regret incurred by the estimated prices $\hat{p}_1^*$ and $\hat{p}_2^*$. Define 
\[
G(p_1, p_2) := R_1(p_1) + R_2(p_2) - \gamma \max\left(|M_1(p_1) - M_2(p_2)| - \lambda \left| M_1(p_1^\sharp)  - M_2(p_2^\sharp) \right|,  0\right).
\]
Note that $G(p_1^*, p_2^*) = R_1(p_1^*) + R_2(p_2^*)$ and therefore the Left-Hand-Side of Eq.~\eqref{eq:thm-exp-constrained-opt-general} equals to $G(p_1^*, p_2^*) - G(\hat{p}_1^*, \hat{p}_2^*)$. To upper bound this quantity, and noting that both $\hat{p}_1^*$ and $\hat{p}_2^*$ are selected from the discretized price checkpoints $\{\ell_j\}_{j \in \{1, 2, \dots, J\}}$, we first prove the following lemma which shows that it suffices to choose the prices from the discretized price checkpoints. In other words, Lemma~\ref{lem:checkpoints-price-for-general-M} upper bounds the regret due to the discretization method.

\begin{lemma}
\label{lem:checkpoints-price-for-general-M}
$\displaystyle{\max_{j_1, j_2 \in \{1, 2, \dots, J\}} \{G(\ell_{j_1}, \ell_{j_2})\}  \geq G(p_1^*, p_2^*) -  2( \overline{p} K +  \gamma K') \cdot T^{-\frac{1}{5}}}$.
\end{lemma}
\begin{proof}{Proof.}
For each customer group $i \in \{1, 2\}$, we find the nearest price checkpoint, namely $\ell_{t_i^*}$ to the optimal fairness-aware price $p_i^*$. Note that we always have that $|\ell_{t_i^*} - p_i^*| \leq T^{-\frac15}$.

By item~\ref{item:assumption-1-lipschitz} and~\ref{item:assumption-1-lipchtiz-M} of Assumption~\ref{assumption:1},  we have that
\begin{align*}
& \left|G(p_1^*,p_2^*) - G(\ell_{t_1^*}, \ell_{t_2^*})\right| \\
& \leq  |R_1(p_1^*)  - R_1(\ell_{t_1^*})| +  |R_2(p_2^*) - R_2(\ell_{t_2^*})| +  
\gamma |M_1(p_1^*) - M_1(\ell_{t_1^*})| + \gamma |M_2(p_2^*) - M_2(\ell_{t_2^*})| \\
& \leq  2 (\overline{p}K + \gamma K') T^{-\frac{1}{5}}.
\end{align*} 
Note that $\max_{j_1, j_2 \in \{1, 2, \dots, J\}} \{G(\ell_{j_1}, \ell_{j_2})\}  \geq  G(\ell_{t_1^*}, \ell_{t_2^*})$, and we prove the lemma. $\square$
\end{proof}

The following lemma uniformly upper bounds the estimation error for $G$ at all pairs of price checkpoints.

\begin{lemma} \label{lem:estimation-on-hat-G-for-general-M}
Suppose that $|\hat{p}_i^\sharp - p_i^\sharp| \leq 4T^\frac{1}{5}$ holds for each  $i \in \{1, 2\}$. With probability at least $(1 - 12(\overline{p} - \underline{p})T^{-3})$, we have that \[
\left|\hat{G}(\ell_{j_1}, \ell_{j_2}) - G(\ell_{j_1}, \ell_{j_2})\right| \leq 2\left( \overline{p} + \gamma\overline{M} + \gamma\lambda (\overline{M} + 5 K') \right)T^{-\frac{1}{5}}
\]
holds for all $j_1, j_2 \in \{1, 2, \dots, J\}$.
\end{lemma}
	
\begin{proof}{Proof.}
For each $i \in \{1, 2\}$, since $|\hat{p}_i^\sharp - p_i^\sharp| \leq 4T^{-\frac{1}{5}}$ and $|\ell_{t_i} - \hat{p}_i^\sharp| \leq T^{-\frac{1}{5}}$ (due to the rounding operation at Line~\ref{line:alg-explore-constrained-opt-general-7}), we have that $|p_i^\sharp - \ell_{t_i}| \leq 5T^{-\frac{1}{5}}$. By item~\ref{item:assumption-1-lipchtiz-M} of Assumption~\ref{assumption:1}, we have that 
\begin{align}
\left| M_i(\ell_{t_i}) - M_i(p_i^\sharp)\right| \leq 5K'T^{-\frac{1}{5}}. \label{eq:general-M-function-est-1}
\end{align}
	    
For each price checkpoint $\ell_j$ and each customer group $i \in \{1, 2\}$, by Azuma's inequality, with probability at least $(1 - 2T^{-3})$, we have that
\begin{align}
\left|\hat{d}_i(\ell_j) - d_i(\ell_j)\right| \leq T^{-\frac{1}{5}}. \label{eq:general-M-fucntion-est-2}
\end{align} 
	   
Therefore, by a union bound, Eq.~\eqref{eq:general-M-fucntion-est-2} holds for all $j \in \{1, 2, \dots, J\}$ and all $i \in \{1, 2\}$ with probability at least $1 - 4(\overline{p} - \underline{p})T^{-2}$. Conditioned on this event, we have that
\begin{align}
\left|\hat{R}_i(\ell_j) - R_i(\ell_j)\right| \leq \overline{p} T^{-\frac{1}{5}}, \qquad  \forall j \in \{1, 2, \dots, J\}, i \in \{1, 2\}. \label{eq:general-M-fucntion-est-3}
\end{align}
	    
Similarly, for each price checkpoint $\ell_j$ and each customer group $i \in \{1, 2\}$, by Azuma's inequality, with probability at least $(1 - 2T^{-3})$,
\begin{align}
\left|\hat{M}_i(\ell_i) - M_i(\ell_i)\right| \leq \overline{M} T^{-\frac{1}{5}}.     \label{eq:general-M-fucntion-est-4}
\end{align}
	     
By a union bound, Eq.~\eqref{eq:general-M-fucntion-est-4} holds for all $j \in \{1, 2, \dots J\}$ and all $i \in \{1, 2\}$ with probability at least $1 - 4(\overline{p} - \underline{p})T^{-2}$. Conditioned on this event, we have that
\[
\left|\hat{M}_1(\ell_{t_1}) - M_1(\ell_{t_1})\right| \leq \overline{M} T^{-\frac{1}{5}}, \quad\text{and}\quad \left|\hat{M}_2(\ell_{t_2}) - M_2(\ell_{t_2})\right| \leq \overline{M} T^{-\frac{1}{5}}.
\]
Together with Eq.~\eqref{eq:general-M-function-est-1}, we have that
\begin{align}
\Big| \big| \hat{M}_1(\ell_{t_1})  - \hat{M}_2(\ell_{t_2}) \big| - \big|M_1(p_1^\sharp) - M_2(p_2^\sharp)\big|  \Big| & \leq (2\overline{M} + 10K')T^{-\frac{1}{5}}. \label{eq:general-M-fucntion-est-5}  
\end{align}
	       
Now, combining Eq.~\eqref{eq:general-M-fucntion-est-3} and Eq.~\eqref{eq:general-M-fucntion-est-5}, and by the definition of $G(\cdot, \cdot)$, for any $j_1, j_2 \in \{1, 2, \dots, J\}$, we have that
\begin{align*}
&\left|\hat{G}(\ell_{j_1}, \ell_{j_2}) - G(\ell_{j_1}, \ell_{j_2})\right|  \\
& \leq  |\hat{R}_1(\ell_{j_1}) - R_1(\ell_{j_1})| + |\hat{R}_2(\ell_{j_2}) - R_2(\ell_{j_2})| + \gamma \left(|\hat{M}_1(\ell_{j_1}) - M_1(\ell_{j_1})| + |\hat{M}_2(\ell_{j_2}) - M_2(\ell_{j_2})|\right) \\
& \quad\qquad\qquad\qquad\qquad\qquad\qquad\qquad\qquad\qquad  +   \gamma\lambda \Big| \big| \hat{M}_1(\ell_{t_1})  - \hat{M}_2(\ell_{t_2}) \big| - \big|M_1(p_1^\sharp) - M_2(p_2^\sharp)\big|  \Big|  \\
& \leq  \left( 2\overline{p} + 2\gamma \overline{M} + 
 \gamma\lambda (2\overline{M} + 10K') \right)T^{-\frac{1}{5}} .
\end{align*}
Finally, collecting the failure probabilities, we prove the lemma. $\square$
\end{proof}

Combining Lemma~\ref{lem:checkpoints-price-for-general-M} and Lemma~\ref{lem:estimation-on-hat-G-for-general-M}, we are able to prove Theorem~\ref{thm:exp-constrained-opt-general}.

\begin{proof}{Proof of Theorem~\ref{thm:exp-constrained-opt-general}.}
Conditioned on that the desired event of Lemma~\ref{lem:estimation-on-hat-G-for-general-M} (which happens with probability at least $1 - 12(\overline{p} - \underline{p})T^{-3} \geq 1 - O(T^{-1})$, we have that
\begin{align*}
G(\hat{p}_1^*, \hat{p}_2^*) &\geq \hat{G}(\hat{p}_1^*, \hat{p}_2^*) - 2\left( \overline{p} + \gamma\overline{M} + 2\gamma\lambda (\overline{M} + 5K') \right)T^{-\frac{1}{5}}\\
& = \max_{j_1, j_2 \in \{1, 2, \dots, J\}} \{\hat{G}(\ell_{j_1}, \ell_{j_2})\} - 2\left( \overline{p} + \gamma\overline{M} + 2\gamma\lambda (\overline{M} + 5K') \right)T^{-\frac{1}{5}}\\
& \geq \max_{j_1, j_2 \in \{1, 2, \dots, J\}} \{G(\ell_{j_1}, \ell_{j_2})\} - 4\left( \overline{p} + \gamma\overline{M} + 2\gamma\lambda (\overline{M} + 5K') \right)T^{-\frac{1}{5}} \\
& \geq G(p_1^*, p_2^*) - 2 (\overline{p}K + \gamma K') T^{-\frac{1}{5}} - 4\left( \overline{p} + \gamma\overline{M} + 2\gamma\lambda (\overline{M} + 5K') \right)T^{-\frac{1}{5}} .
\end{align*}
Here, the first two inequalities are due to the desired event of Lemma~\ref{lem:estimation-on-hat-G-for-general-M}, the equality is by Line~\ref{line:alg-explore-constrained-opt-general-9} of the algorithm, and the last inequality is due to Lemma~\ref{lem:checkpoints-price-for-general-M}.

Observing that the Left-Hand-Side of Eq.~\eqref{eq:thm-exp-constrained-opt-general} equals to $G(p_1^*, p_2^*) - G(\hat{p}_1^*, \hat{p}_2^*)$, we prove the theorem. $\square$
\end{proof}

\subsection{Proof of Theorem~\ref{thm:main-upper-general}}

\begin{proof}{Proof of Theorem~\ref{thm:main-upper-general}.}
The proof will be carried out conditioned on the desired events of both Theorem~\ref{thm:exp-unconstrained-opt} and Theorem~\ref{thm:exp-constrained-opt-general}, which happens with probability at least $(1 - O(T^{-1}))$. Since the first two steps use at most $O(T^{\frac45} \log^2 T)$ selling periods, they incur at most $O(T^{\frac45} \log^2 T) \times O(\overline{p} + \overline{M}) = O(T^{\frac45} \log^2 T)$ penalized regret. By the desired event of Theorem~\ref{thm:exp-constrained-opt-general}, the penalized regret incurred by the third step is at most $T \times O(T^{-\frac15}) = O(T^\frac45)$.  $\square$
\end{proof}

\blue{
\section{Lower Bound for Penalized Regret under the Soft Constraints}\label{sec:lb-soft-constraint}
In Section~\ref{sec:lower}, we presented the $\Omega(T^{4/5})$  lower bound of the standard regret, assuming the hard fairness constraint is satisfied. In this section, we aim to establish a lower bound for the expected penalized regret under the soft constraint. Formally, we prove the following lower bound theorem.

\begin{theorem} \label{theorem:soft-lb}
Suppose that $\pi$ is an online pricing algorithm. Then for any $\lambda \in (\epsilon, 1-\epsilon)$,  $\gamma\geq 2/A$ and $T \geq \epsilon^{-C_\mathrm{LB}}$ (where $C_\mathrm{LB} > 0$ is the same  universal constant as in Theorem~\ref{theorem:lb}), there exists a pricing instance such that the expected penalized regret of $\pi$ is at least $\frac{1}{200}\epsilon^2 T^{4/5}$.
\end{theorem}
In the proof of Theorem~\ref{theorem:soft-lb}, we use the same hard instances in Section~\ref{sec:lower}, and thus Lemma~\ref{lemma:lb-demand-profit-properties} and Lemma~\ref{lemma:lb-optimal-fairness-aware-solution} still hold. Specifically, we construct two problem instances $\mathcal I = \mathcal I(d_1, d_3)$ and $\mathcal I' = \mathcal I(d_2, d_3)$, where $d_i(p) =  R_i(p)/ p$ for $i \in \{1, 2, 3\}$, and we define $R_i$'s as follows.
\begin{align*}
R_1(p) &= \frac{1}{4} - \frac{1}{A}(p - 1 - \frac{\sqrt{h}}{4})^2, \qquad\qquad \quad ~ p \in [1, 2],\\
R_2(p) &=
		\left\{ \begin{aligned}
			& \frac{1}{4} - \frac{1}{2A}(p - 1 + \frac{\sqrt{h}}{4})^2, & & p \in [1, 1 + \frac{5\sqrt{h}}{4}), \\
			& \frac{1}{4} - \frac{3}{2A}(p - 1 - \frac{3\sqrt{h}}{4})^2 - \frac{3h}{4A}, & & p \in [1 + \frac{5\sqrt{h}}{4}, 1 + \frac{7\sqrt{h}}{4}), \\
			& \frac{1}{4} - \frac{1}{A}(p - 1 - \frac{\sqrt{h}}{4})^2, & & p \in [1 + \frac{7\sqrt{h}}{4},2],\\
		\end{aligned}\right.			\\
R_3(p) &= \frac{1}{8} - \frac{1}{A}(p - 2)^2, \qquad \qquad \qquad\quad~~~p \in [1,2].
\end{align*}
Here, $A \geq 1$ is a large enough universal constant and $h \geq 0$ depends on $T$, both of which will be chosen later. 
For any problem instance $\mathcal J \in \{\mathcal I, \mathcal I'\}$, and any demand function $d$ that is employed by a customer group in $\mathcal J$, we denote by $p^*(d; \mathcal J)$ the price for the customer group in the optimal fairness-aware clairvoyant solution to $\mathcal J$.}

\blue{The proof of Theorem~\ref{theorem:soft-lb} follows the similar proof structure  of Theorem~\ref{theorem:lb}. In order to deal with the fact that the pricing policy does not satisfy the fairness constraints, the major differences between the proof of Theorem~\ref{theorem:soft-lb} compared with that
of Theorem~\ref{theorem:lb} include: 1) in the proof of Lemma~\ref{lemma:soft-lb-regret-fairness-aware} and Lemma~\ref{lemma:soft-lb-regret-I}, by imposing the carefully designed penalty term we establish the price of choosing a cheap first-group price and violating the fairness constraints; 2) in the proof of the main theorem, we discuss two more sub-cases when the probability distributions are close enough in two problem instances.}

\medskip
\blue{\noindent{\bf \underline{The price of a cheap first-group price.}} By Lemma~\ref{lemma:lb-optimal-fairness-aware-solution}, we see that when $h \leq  \epsilon^2/40$, we have that both $p^*(d_1; \mathcal I)$ and $p^*(d_2; \mathcal I')$ are greater than $1 + \frac{7\sqrt{h}}{4}$. For any pricing strategy $(p, p')$, we say it is \emph{cheap for the first group} if $p \leq 1 + \frac{7\sqrt{h}}{4}$. The following lemma lower bounds the regret of a fairness-aware pricing strategy when it is cheap for the first group (and therefore deviates from the optimal solution).


\begin{lemma}\label{lemma:soft-lb-regret-fairness-aware}
Suppose that $h \leq \eps^4/400$ and $\gamma\geq 2/A$. For any pricing strategy $(p_1, p_3)$ for the problem instance $\mathcal I = \mathcal I(d_1, d_3)$, if $p_1 \in [1, 1 + \frac{7\sqrt{h}}{4}]$, we have that
\begin{align*}
    &\left[R_1(p^*(d_1; \mathcal I)) + R_3(p^*(d_3; \mathcal I))-R_1(p_1) - R_3(p_3) \right]
 + \gamma \max\left(\left|p_1 - p_3\right| - \lambda \left|p_1^\sharp - p_3^\sharp\right|, 0\right) \geq \frac{\epsilon^2}{4A}.
\end{align*}
\end{lemma}}
\blue{
\begin{proof}{Proof of Lemma~\ref{lemma:soft-lb-regret-fairness-aware}.}
For convenience, we define $\Delta := |p^\sharp(d_3) - p^\sharp(d_1)| = 1-\frac{\sqrt{h}}{4}$. By Lemma \ref{lemma:lb-optimal-fairness-aware-solution}, we have that
\[
R_1(p^*(d_1, \mathcal I)) = \frac{1}{4} - \frac{1}{A}\left(\frac{1-\lambda}{2} \Delta \right)^2, \qquad 
R_3(p^*(d_3, \mathcal I)) = \frac{1}{8} - \frac{1}{A}\left(\frac{1-\lambda}{2} \Delta \right)^2. 
\]
Since  $\lambda \in (0, 1 - \epsilon)$, $\Delta \in (0, 1]$,  and $h \leq \epsilon^4 / 400$, it holds that $1 + \frac{7\sqrt{h}}{4} + \lambda \Delta<2$. Let $\alpha$ be a constant in $[0, 1-\frac{7\sqrt{h}}{4} - \lambda \Delta]$. Note that 
\begin{align}\nonumber
    &R_3\left(1 + \frac{7\sqrt{h}}{4} + \lambda \Delta+\alpha\right)-\gamma\alpha \\&=\frac18- \frac{1}{A}\left(\frac{7\sqrt{h}}{4}+\lambda\Delta+\alpha-1\right)^2-\gamma\alpha \nonumber \\&= R_3\left(1 + \frac{7\sqrt{h}}{4} + \lambda \Delta\right)-\frac{1}{A}\alpha^2-\frac{2}{A}\left(\frac{\sqrt{7h}}{4}+\lambda\Delta-1\right) \alpha-\gamma\alpha  \nonumber \\& \label{eq:softcheapfirst1}\leq R_3\left(1 + \frac{7\sqrt{h}}{4} + \lambda \Delta\right),
\end{align}
where the last inequality is because $\frac{1}{A}\alpha^2+\frac{2}{A}\left(\frac{\sqrt{7h}}{4}+\lambda\Delta-1\right) \alpha+\gamma\alpha\geq 0$ with $\gamma\geq 2/A \geq \frac{2}{A}\left(1-\frac{\sqrt{7h}}{4}-\lambda\Delta\right)$.

Since $p_1 \in [1, 1 + \frac{7\sqrt{h}}{4}]$, using the monotonicity of $R_3(\cdot)$ and the fact that $R_1(p_1) \leq 1/4$, by Eq.~\eqref{eq:softcheapfirst1} we obtain
\begin{align*}
    &R_1(p_1) + R_3(p_3) - \gamma \max\left(\left|p_1 - p_3\right| - \lambda \left|p_1^\sharp - p_3^\sharp\right|, 0\right)\\&\leq \frac14 +\max\{R_3\left(1 + \frac{7\sqrt{h}}{4} + \lambda \Delta\right), R_3\left(1 + \frac{7\sqrt{h}}{4} + \lambda \Delta+\alpha\right)-\gamma\alpha\} 
    \\&\leq \frac14 + R_3\left(1 + \frac{7\sqrt{h}}{4} + \lambda \Delta\right). 
\end{align*}
Then following the proof of Lemma~\ref{lemma:lb-regret-fairness-aware}, we could get the conclusion
\begin{align*}
    &\left[R_1(p^*(d_1; \mathcal I)) + R_3(p^*(d_3; \mathcal I))-R_1(p_1) - R_3(p_3) \right]
 + \gamma \max\left(\left|p_1 - p_3\right| - \lambda \left|p_1^\sharp - p_3^\sharp\right|, 0\right) \geq \frac{\epsilon^2}{4A}.
\end{align*}
$\square$
\end{proof}}

\medskip
\blue{
\noindent{\bf \underline{The price of violating the fairness constraints.}} 
In the following lemma, we show that significant regret would occur when the  pricing strategy violates the fairness constraint by $\lambda \frac{\sqrt{h}}{8}$ due to the penalty term. 

\begin{lemma}\label{lemma:soft-lb-regret-I}
Suppose that $h \leq \epsilon^2/40$ , $\gamma\geq 2/A$  and $(p_1, p_3)$ is a pricing strategy that satisfies  
\[
|p_1 - p_3| > \lambda \left(|p^\sharp(d_3) - p^\sharp(d_1)|+\frac{\sqrt{h}}{8}\right).
\]
 we have that
\[
\left[R_1(p^*(d_1; \mathcal I)) + R_3(p^*(d_3; \mathcal I))-R_1(p_1) - R_3(p_3) \right] + \gamma \max\left(\left|p_1 - p_3\right| - \lambda \left|p_1^\sharp - p_3^\sharp\right|, 0\right) \geq \frac{\lambda \sqrt{h}}{8A}.
\]
\end{lemma}

\begin{proof}{Proof of Lemma~\ref{lemma:soft-lb-regret-I}.}
For any $p_1\in [p^\sharp(d_1), p^\sharp(d_3)-\lambda (1-\frac{\lambda\sqrt{h}}{8})]$, let $\alpha$ be a constant in $[0,  p^\sharp(d_3)-p_1-\lambda+\frac{\lambda\sqrt{h}}{8}]$. It holds that 
\begin{align}\nonumber
    &R_3\left( p_1 + \lambda(1-\frac{\sqrt{h}}{8})+\alpha\right)-\gamma\alpha \\&=\frac18- \frac{1}{A}\left(p_1 + \lambda(1-\frac{\sqrt{h}}{8})-2+\alpha\right)^2-\gamma\alpha \nonumber \\&= R_3\left(p_1 + \lambda(1-\frac{\sqrt{h}}{8})\right)-\frac{1}{A}\alpha^2-\frac{2}{A}\left(p_1 + \lambda(1-\frac{\sqrt{h}}{8})-2\right) \alpha-\gamma\alpha  \nonumber \\& \label{eq:softcheapfirst1I}\leq R_3\left(p_1 + \lambda(1-\frac{\sqrt{h}}{8})\right),
\end{align}
where the last inequality is because $\frac{1}{A}\alpha^2+\frac{2}{A}\left(p_1 + \lambda(1-\frac{\sqrt{h}}{8})\right) \alpha+\gamma\alpha\geq 0$ with $\gamma\geq 1 \geq \frac{2}{A}\left(2-p_1- \lambda(1-\frac{\sqrt{h}}{8})\right)$.

For any price strategy $(p_1, p_3)$  that satisfies  
$|p_1 - p_3| > \lambda \left(|p^\sharp(d_3) - p^\sharp(d_1)|+\frac{\sqrt{h}}{8}\right)$, by the monotonicity of $R_3(\cdot)$ and  Eq.~\eqref{eq:softcheapfirst1I} it holds that
\begin{align}\label{eq:starstar0}
    R_1(p_1) + R_3(p_3)  - \gamma \max\left(\left|p_1 - p_3\right| - \lambda \left|p_1^\sharp - p_3^\sharp\right|, 0\right) \leq  R_1(p_1) + R_3\left(p_1 + \lambda(1-\frac{\sqrt{h}}{8})\right)  - \gamma \frac{\lambda\sqrt{h}}{8}.
\end{align}
By the similar proof to Lemma~\ref{lemma:lb-optimal-fairness-aware-solution}, we have that 
\begin{align*}
  p_1^{**} =   \arg\max_{p_1\in[p^\sharp(d_1), p^\sharp(d_3)-\lambda (1-\frac{\lambda\sqrt{h}}{8})]} R_1(p_1) + R_3\left(p_1 + \lambda(1-\frac{\sqrt{h}}{8})\right) = p^{*}(d_1,\mathcal I) - \frac{\lambda\sqrt{h}}{16}.
\end{align*}
By the definition of $R_1(\cdot)$, we obtain 
\begin{align}\nonumber
R_1(p_1^{**}) &= \frac{1}{4} - \frac{1}{A}\left(p^{*}(d_1,\mathcal I) - \frac{\lambda\sqrt{h}}{16} - 1 - \frac{\sqrt{h}}{4}\right)^2\\&=   R_1(p^*(d_1; \mathcal I)) -  \frac{1}{A} \left(\frac{\lambda\sqrt{h}}{16}\right)^2 + \frac{\lambda\sqrt{h}}{8A} \left(p^{*}(d_1,\mathcal I) - \frac{\lambda\sqrt{h}}{16} - 1 - \frac{\sqrt{h}}{4}\right)
\nonumber\\&\leq R_1(p^*(d_1; \mathcal I)) +  \frac{\lambda\sqrt{h}}{8A} \left(p^{*}(d_1,\mathcal I) - \frac{\lambda\sqrt{h}}{16} - 1 - \frac{\sqrt{h}}{4}\right)
\nonumber\\&\leq R_1(p^*(d_1; \mathcal I)) +  \frac{\lambda\sqrt{h}}{16A}\label{eq:starstar1},
\end{align}
where the last inequality is due to $\left(p^{*}(d_1,\mathcal I) - \frac{\lambda\sqrt{h}}{16} - 1 - \frac{\sqrt{h}}{4}\right)\leq \frac{1}{2}$.

  Let $p_3^{**} =  p_1^{**} + \lambda(1-\frac{\sqrt{h}}{8}) = p^{*}(d_3,\mathcal I) + \frac{\lambda\sqrt{h}}{16}$. Similarly, it also holds that
\begin{align}\label{eq:starstar2}
   R_3(p_3^{**})\leq  R_3(p^*(d_3; \mathcal I))+ \frac{\lambda\sqrt{h}}{16A}.
\end{align}
Combining Eq.~\eqref{eq:starstar1} and Eq.~\eqref{eq:starstar2}, we have
\begin{align*}
    R_1(p^*(d_1; \mathcal I)) +
    R_3(p^*(d_3; \mathcal I))\geq R_1(p_1^{**})+ R_3(p_3^{**}) - \frac{\lambda\sqrt{h}}{8A}.
\end{align*}
Invoking Eq.~\eqref{eq:starstar0} with the above inequality, by $\gamma>2/A$ we get the conclusion that
\begin{align*}
\left[R_1(p^*(d_1; \mathcal I)) + R_3(p^*(d_3; \mathcal I))-R_1(p_1) - R_3(p_3) \right] + \gamma \max\left(\left|p_1 - p_3\right| - \lambda \left|p_1^\sharp - p_3^\sharp\right|, 0\right) \geq \frac{\lambda \sqrt{h}}{8A}.
\end{align*}

\Halmos
\end{proof}
}

\medskip
\blue{
\noindent{\bf \underline{The price of identifying the wrong instance.}} If a pricing strategy misidentifies the underlying instance $\mathcal I'$ by $\mathcal I$ and satisfies the fairness condition of $\mathcal I$, we show in the following lemma that the significant regret would occur when we apply such a pricing strategy to $\mathcal I'$.   The proof of Lemma~\ref{lemma:soft-lb-regret-I-prime} is almost the same as that of Lemma~\ref{lemma:lb-regret-I-prime}.

\begin{lemma}\label{lemma:soft-lb-regret-I-prime}
Suppose that $h \leq \epsilon^2/40$ and $(p_2, p_3)$ is a pricing strategy that satisfies  
\[
|p_2 - p_3| \leq \lambda \left(|p^\sharp(d_3) - p^\sharp(d_1)|+\frac{\sqrt{h}}{8}\right).
\]
 we have that
\[
\left[R_2(p^*(d_2; \mathcal I')) + R_3(p^*(d_3; \mathcal I'))-R_2(p_2) - R_3(p_3) \right] + \gamma \max\left(\left|p_2 - p_3\right| - \lambda \left|p_2^\sharp - p_3^\sharp\right|, 0\right) \geq \frac{\epsilon \lambda \sqrt{h}}{8A}.
\]
\end{lemma}}

\blue{
With the technical lemmas in hand, we can turn to the proof of Theorem~\ref{theorem:soft-lb} now.
\begin{proof}{Proof of Theorem~\ref{theorem:soft-lb}.}
We set $h = T^{-2/5}$ and $A = 10$. Note that when $T \geq (5/\eps)^{10}$, the assumptions of Lemmas~\ref{lemma:soft-lb-regret-fairness-aware} and \ref{lemma:soft-lb-regret-I-prime} are met. We now discuss the following two cases.

\noindent \underline{\it Case 1: $\sum_{t=1}^{T} \Pr_{\mathcal P_{\mathcal I, \pi}}[p^{(t)}(d_1; \mathcal I, \pi) \in [1, 1+\frac{7\sqrt{h}}{4}]] \geq A^2/(400 h^2)$.} Invoking Lemma~\ref{lemma:soft-lb-regret-fairness-aware}, we have that the expected regret incurred by $\pi$ for instance $\mathcal I$ is at least
\[
\frac{A^2}{400 h^2} \cdot \frac{\epsilon^2}{4A} =\frac{\epsilon^2 A}{1600 h^2} = \frac{\epsilon^2 T^{4/5}}{160}.
\]

\noindent \underline{\it Case 2: $\sum_{t=1}^{T} \Pr_{\mathcal P_{\mathcal I, \pi}}[p^{(t)}(d_1; \mathcal I, \pi) \in [1, 1+\frac{7\sqrt{h}}{4}]] < A^2/(400 h^2)$.} Let $\tilde T$
 be the total number of the periods that the distance of the two prices is within $\lambda \left(|p^\sharp(d_3) - p^\sharp(d_1)|+\frac{\sqrt{h}}{8}\right)$.



Combining Lemma~\ref{lemma:lb-KL-decomposition} and Pinsker's inequality (Lemma~\ref{lemma:pinsker}), we have that
\begin{align}\label{eq:soft-thm-lb-2}
\left\|\Pr_{\mathcal P_{\mathcal I, \pi}}- \Pr_{\mathcal P_{\mathcal I', \pi}}\right\|_{\mathrm{TV}} \leq \sqrt{\frac{1}{2} \cdot \frac{A^2}{400h^2} \cdot \frac{4h^2}{ A^2}} \leq 0.1 .
\end{align}
With Eq.~\eqref{eq:soft-thm-lb-2}, it holds that
\begin{align}
  \nonumber  |\E_{\mathcal P_{\mathcal I, \pi}}[\tilde T] - \E_{\mathcal P_{\mathcal I', \pi}}[\tilde T]|&\leq \sum_{t=1}^T t \cdot \left| \Pr_{\mathcal P_{\mathcal I, \pi}}[\tilde T = t]-\Pr_{\mathcal P_{\mathcal I', \pi}}[\tilde T = t]\right| \\&\leq T \left\| \Pr_{\mathcal P_{\mathcal I, \pi}}-\Pr_{\mathcal P_{\mathcal I', \pi}} \right\|_{\mathrm{TV}}
   \nonumber  \\&\leq 0.1T.\label{eq:soft-pinsker}
\end{align}
We now discuss the following two sub-cases.
\begin{itemize}
    \item \noindent \underline{\it Case 2a: $\E_{\mathcal P_{\mathcal I, \pi}}[\tilde T]\leq \frac{T}{2}$.}
      Invoking Lemma~\ref{lemma:soft-lb-regret-I}, the expected penalized regret of $\pi$ for instance $\mathcal I$ is at least
\[
(T-\E_{\mathcal P_{\mathcal I, \pi}}[\tilde T])\frac{\lambda \sqrt{h}}{8A}\geq \frac{1}{2}\cdot T\cdot \frac{\lambda \sqrt{h}}{8A}\geq \frac{1}{160}\epsilon T^{4/5}.
\]
    \item \noindent \underline{\it Case 2b: $\E_{\mathcal P_{\mathcal I, \pi}}[\tilde T]> \frac{T}{2}$.}
Combining the Eq.~\eqref{eq:soft-pinsker}, we have $$\E_{\mathcal P_{\mathcal I', \pi}}[\tilde T]>0.4 T.$$
      Invoking Lemma~\ref{lemma:soft-lb-regret-I-prime}, the expected penalized regret of $\pi$ for instance $\mathcal I'$ is at least
\[
\E_{\mathcal P_{\mathcal I', \pi}}[\tilde T]\frac{\lambda \sqrt{h}}{4A}\geq 0.4 \cdot T\cdot \frac{\epsilon\lambda \sqrt{h}}{8A}\geq \frac{1}{200}\epsilon^2 T^{4/5}.\]
\end{itemize}
Combining the above cases, we prove the theorem. $\square$
\end{proof}
}

\blue{
\section{{Extension: the General Discrepancy Function}}\label{sec: extension general f}
In the main body of this paper, we have aimed at achieving fairness via mandating the small \emph{difference} between the prices (or other fairness measures defined by $M(\cdot)$). In this extension, we consider a \emph{general discrepancy function} $f(\cdot, \cdot)$ between the prices (or other fairness measures) to define the fairness constraints. Specifically,
we use $|f(M_1(p_1),M_2(p_2))|$ to substitute  $|M_1(p_1)-M_2(p_2)|$  in the fairness constraint and the penalized regret definition, and show that our algorithmic framework can be adapt to work with this class of even more general fairness constraints. 

Let $\{p_1^*, p_2^*\}$ denote the \emph{fairness-aware clairvoyant} solution, i.e., the optimal solution to the following static optimization problem,
\begin{align}\label{eq:ext-static_fair}
\max_{p_1, p_2 \in [\underline p, \overline p]} \qquad & R_1(p_1) + R_2(p_2),\\
\text{subject to} \qquad & \left|f(p_1,p_2)\right| \leq \lambda \left|f(p_1^\sharp,p_2^\sharp)\right|. \nonumber
\end{align}
The \emph{regret} is defined as the difference between the expected total revenue and the fairness-aware clairvoyant solution:
\begin{equation}
\label{eq:ext-def_reg}
\mathrm{Reg}_T := \mathbb E \sum_{t=1}^T\left[ R_1(p_1^*) + R_2(p_2^*) - R_1(p_{1}^{(t)}) - R_2(p_{2}^{(t)}) \right],
\end{equation}
where $p_1^*$ and $p_2^*$ are the fairness-aware clairvoyant solutions defined in Eq.~\eqref{eq:ext-static_fair}.

For general fairness measure, the seller aims to minimize the following cumulative penalized regret:
\begin{align*}
\mathrm{Reg}_T^{\mathrm{soft}} &:= \mathbb E \sum_{t=1}^T\Big\{ \left[R_1(p_1^*) + R_2(p_2^*) - R_1(p_1^{(t)}) - R_2(p_2^{(t)})\right]\\
& \qquad\qquad\qquad\qquad + \gamma \max\left(\left| f(M_1(p_1^{(t)}),M_2(p_2^{(t)}))\right| - \lambda \left|f(M_1(p_1^\sharp),M_2(p_2^\sharp)) \right|, 0\right) \Big\}.
\end{align*}}



\blue{
We impose the following assumptions on the function $f(x,y)$,
\begin{assumption}\label{assumption:f}
    The discrepancy function $f(x,y):[0,\overline M]\times[0,\overline M]\mapsto\mathbb R$ satisfies 
\begin{enumerate}
\item When $x=y$, it holds that $f(x,y)= 0$.
     \item \label{item:lipofM}  Lipschitz continuous property :
    \begin{align*}
        \frac{1}{L_f}|x_1 - x_2|  & \leq |f(x_1,y)-f(x_2,y)| \leq L_f|x_1-x_2|, & \quad\forall y \in(0,\overline M];\\
        \frac{1}{L_f}|y_1 - y_2| & \leq |f(x,y_1)-f(x,y_2)| \leq L_f|y_1-y_2|, & \quad\forall x \in(0,\overline M].
    \end{align*}

\item Monotonicity property : $|f(x,y)|$ increases as $|x-y|$ increases, and for any fixed $x$.

\item Fix $p_1$, for some constant $\xi>0$, we could get $p_2>p_1$ such that $|f(p_1,p_2)| =  \xi$, and we denote $p_2$ as $f^{-1}(p_1; \xi)$.
\end{enumerate}
\end{assumption}
It is easy to note that $f(x,y) = x-y$ considered in this paper satisfies the above assumptions. Meanwhile, some non-additive form of $f(x,y)$ such as  $f(x,y) = \ln\left(\frac{x+\epsilon}{y+\epsilon}\right)$, also meets these assumption.
}




\blue{
\subsection{Price Fairness with Hard Constraint}\label{sec:general-discrepancy-price-fairness-hard-constraint}
To deal with the general discrepancy function, we only need to adapt the subroutine {\sc ExploreUnconstrainedOPT} (Algorithm~\ref{alg:exp-unconstrainted-opt}) to the following Algorithm~\ref{alg:exp-constrained-opt-f}. The key difference between  Algorithm~\ref{alg:exp-constrained-opt-f} and Algorithm~\ref{alg:exp-unconstrainted-opt}, is that in Line 4 of Algorithm~\ref{alg:exp-constrained-opt-f}, we search the beginning point instead of the middle point. }

\begin{algorithm}[!ht]\blue{
\caption{\sc ExploreConstrainedOPTGeneralDiscrepancy\label{alg:exp-constrained-opt-f}}		
\SetKwInOut{Input}{Input}
\SetKwInOut{Output}{Output}
\Input{the estimated unconstrained optimal prices $\hat{p}_1^\sharp$ and $\hat{p}_2^\sharp$,  assuming that $\hat{p}_1^\sharp \leq \hat{p}_2^\sharp$ (without loss of generality)}
\Output{the estimated constrained optimal prices $\hat{p}_1^*$ and  $\hat{p}_2^*$}

$\xi \leftarrow \max\{ |f(\hat{p}_1^\sharp, \hat{p}_2^\sharp)| - 8 L_f  T^{-1/5}, 0\}$\; 

$J \leftarrow \lceil (\overline{p}-\underline{p}) T^\frac{1}{5}\rceil$ and create $J$ price checkpoints $\ell_1, \ell_2, \dots, \ell_J$ where $\ell_j \leftarrow \underline{p} + \frac{j}{J} (\overline{p} - \underline{p})$\; 

\For{each $\ell_j$}{
    Repeat the following offerings for $6 T^{2/5} \ln T$ selling periods: offer price $p_1(j) \leftarrow  \ell_j $ to customer group $1$ and price $p_2(j) \leftarrow \min\{\overline{p}, f^{-1}(\ell_j; {\lambda \xi})\}$ to customer group $2$\;
    Denote the average demand from customer group $i \in \{1, 2\}$ by $\hat{d}_i(j)$\;
    $\hat{R}(j) \leftarrow \hat{d}_1(j) (p_1(j) - c) + \hat{d}_2(j) (p_2(j) - c)$\;
    }
$j^* \leftarrow \argmax_{j \in \{1, 2, \dots, J\}} \{\hat{R}(j)\}$\;
\Return $\hat{p}^*_1 \leftarrow p_1(j^*)$ and $\hat{p}^*_2 \leftarrow p_2(j^*)$\;
}
\end{algorithm}

\blue{
We then present the theoretical results as follows. Since the proof of the Lemma~\ref{lem:exp-constrained-opt-price-gap-f} and Lemma~\ref{lem:exp-constrained-opt-xi-bound-f} is similar to that in Section~\ref{sec:alg-exp-constrained-opt}, we omit the proof details.

In analogous to Lemma~\ref{lem:exp-constrained-opt-price-gap}, the following lemma  establishes the relation between the  price difference of the constrained optimal solution
and that of the unconstrained optimal solution.
\begin{lemma}\label{lem:exp-constrained-opt-price-gap-f}
$\displaystyle{f(p^*_1, p^*_2) = \lambda f(p^\sharp_1,p^\sharp_2)}$.
\end{lemma}}
\blue{
In analogous to Lemma~\ref{lem:exp-constrained-opt-xi-bound},the following lemma provide bounds for the $\xi$ parameter which is used in the algorithm to control the price gaps between the two customer groups.

\begin{lemma}\label{lem:exp-constrained-opt-xi-bound-f}
Suppose that $|\hat{p}_1^\sharp - p_1^\sharp| \leq  4 T^{-1/5}$ and $|\hat{p}_2^\sharp - p_2^\sharp| \leq  4 T^{-1/5}$, we have that $\lambda \xi \leq \lambda |f(p_1^\sharp, p_2^\sharp)|$ and  $\lambda \xi \geq \max\{0, \lambda |f(p_1^\sharp, p_2^\sharp)| - 16L_f T^{-1/5}$\}.
\end{lemma}}

\blue{
In analogous to Lemma~\ref{lem:exp-constrained-opt-tilde-j}, the following lemma shows that our discretization scheme always guarantees that there is a price check point to approximate the  constrained optimal prices. 
\begin{lemma}\label{lem:exp-constrained-opt-tilde-j-f}
There exists $\tilde{j} \in \{1, 2, \dots, J\}$ such that both $p_1(\tilde{j}), p_2(\tilde{j}) \in [\underline{p}, \overline{p}]$ and $|p_1(\tilde{j}) - p_1^*| \leq T^{-1/5}$, $|p_2(\tilde{j}) - p_2^*| \leq 17 L^2_f T^{-1/5}$.
\end{lemma}
\begin{proof}{Proof.}
By the definition of $J$, we know there exists $\tilde{j} \in \{1, 2, \dots, J\}$ such that both $p_1(\tilde{j}), p_2(\tilde{j}) \in [\underline{p}, \overline{p}]$ and $|p_1(\tilde{j}) - p_1^*| \leq T^{-1/5}$. Thus it is sufficient to show such $\tilde{j}$ meets the requirement for $p_2$.

By Lipschitz property~\ref{assumption:f}\ref{item:lipofM}, Lemma~\ref{lem:exp-constrained-opt-xi-bound-f} and the choice of $|p_1(\tilde{j}) - p_1^*| \leq T^{-1/5}$, we have:
\begin{align*}
    |p_2(\tilde{j}) - p_2^*| & \leq L_f \left| f\left(p_1(\tilde{j}), p_2^*\right) - f\left(p_1(\tilde{j}),p_2(\tilde{j})\right)\right| \\
    & \leq L_f \left| f\left(p_1(\tilde{j}), p_2^*\right) - f\left( p_1^*, p_2^*\right)\right| + L_f \left|f\left( p_1^*, p_2^*\right)- f\left(p_1(\tilde{j}),p_2(\tilde{j})\right)\right| \\
    & \leq L^2_f\left| p_1(\tilde{j}) - p_1^*\right| + 16 L^2_f T^{-1/5} \\
    & \leq 17 L^2_f T^{-1/5}.
\end{align*}
Therefore, we prove this lemma. \Halmos
\end{proof}}

\blue{
Note this lemma is the same as~Lemma~\ref{lem:exp-constrained-opt-tilde-j} up to constant factors of $L^2_f$. Thus with such lemma we immediately get the following theorem. We omit the proof since it is almost the same as the proof of Theorem~\ref{thm:main_upper}.

\begin{theorem}\label{thm:main_upper-f}
With probability $(1 - O(T^{-1}))$, modified Algorithm~\ref{alg:price-fairness-main} (substituting Algorithm~\ref{alg:exp-constrained-opt} by Algorithm~\ref{alg:exp-constrained-opt-f}) satisfies the fairness constraint and its regret is at most $O(T^{\frac45}\log^2 T)$. Here, the $O(\cdot)$ notation only hides the polynomial dependence on $\overline{p}$, $K$, $L_f$ and $1/C$.
\end{theorem}}

\blue{
\subsection{General Fairness Measure with Soft Constraint}
We may combine the techniques in Section~\ref{sec:general-discrepancy-price-fairness-hard-constraint} and Section~\ref{sec:general} to deal with the soft fairness constraint with both a general discrepancy function $f$ and a general fairness measure $M_i(\cdot)$. In this case, we may obtain a similar $T^{4/5}$-type soft regret bound. We omit the details for this setting since the techniques used here are quite repetitive and their combination is quite straightforward.
 
}

\blue{
\section{Extension: the Mult-group Setting}\label{sec:extension multiple group}

In this section, we extend our algorithms and theoretical results  to encompass multi-group setting in terms of  both price fairness and general fairness cases. Specifically, we adopt the fairness constraints in \cite{cohen2022price}: $\left|M_i(p_{i}) - M_j(p_{j})\right| \leq \lambda \max_{1 \leq i' < j' \leq N}\left|M_{i'}(p_{i'}^\sharp) - M_{j'}(p_{j'}^\sharp)\right|$ for all $1 \leq i < j \leq N$ pairs.}

\blue{
We first define the unconstrained optimal solution in the multi-group setting as follows. 
\begin{equation}\label{eq:mul_rev}
p_i^{\sharp} = \argmax_{p \in [\underline p, \overline p]} \; R_i(p):= (p - c) d_i(p), \qquad \forall i \in \{1, 2,\dots,N\},
\end{equation}
Following the definition in \cite{cohen2022price}, let $\{p_1^*, p_2^*,\dots,p_N^*\}$ denote the \emph{fairness-aware clairvoyant} solution, i.e., the optimal solution to the following static optimization problem,
\begin{align}\label{eq:mul_static_fair}
\max_{p_1, p_2,\dots,p_N \in [\underline p, \overline p]} \qquad & \sum_{i=1}^NR_i(p_i),\\
\text{subject to} \qquad & \left|M_i(p_{i}) - M_j(p_{j})\right| \leq \lambda \max_{1 \leq i' < j' \leq N}\left|M_{i'}(p_{i'}^\sharp) - M_{j'}(p_{j'}^\sharp)\right|\text{ for all } 1 \leq i < j \leq N. \nonumber
\end{align}
The seller would like to minimize the \emph{regret}, which is the difference between the expected total revenue and the fairness-aware clairvoyant solution:
\begin{equation}
\label{eq:mul_def_reg}
\mathrm{Reg}_T := \mathbb E \sum_{t=1}^T\left[ \sum_{i=1}^NR_i(p_i^*) -\sum_{i=1}^NR_i(p_i^{(t)}) \right],
\end{equation}
where $p_i^{(t)}, i\in[N]$ satisfy the hard constraint defined in Eq.~\eqref{eq:mul_static_fair}.}

\blue{
For general fairness measure, the seller aims to minimize the following cumulative penalized regret:
\begin{align*}
\mathrm{Reg}_T^{\mathrm{soft}} &:= \mathbb E \sum_{t=1}^T\Big\{ \left[\sum_{i=1}^NR_i(p_i^*) -\sum_{i=1}^NR_i(p_i^{(t)})\right] \\
& \qquad\qquad +
\gamma \sum_{{1 \leq i < j \leq N}}\max\left(\left|M_i(p_{i}^{(t)}) - M_j(p_{j}^{(t)})\right| - \lambda \max_{1 \leq i' < j' \leq N}\left|M_{i'}(p_{i'}^\sharp) - M_{j'}(p_{j'}^\sharp)\right|, 0\right)\Big\},
\end{align*}where the first term is the standard regret, the second term is the penalty term for violating the fairness constraint, and $\gamma$ is a pre-defined parameter to balance between the regret and the fairness constraint and assumed to be $O(1)$.}

\blue{
Throughout this section, we will make the following standard assumptions on demand functions and fairness measures, which simply extend Assumption~\ref{assumption:1} from two groups to $N$ groups: 

\begin{assumption}\label{assumption:2}

\begin{enumerate}
\item \label{item:assumption-2-lipschitz} The demand-price functions are monotonically decreasing and injective Lipschitz, i.e., there exists a constant $K \geq 1$ such that for each group $i \in [N]$, it holds that 
\[
\frac{1}{K} |p-p'| \leq |d_i(p) - d_i(p')| \leq K |p-p'|, \qquad \forall p,p' \in [\underline{p}, \overline{p}] .
\]

\item \label{item:assumption-2-strong-concavity} The revenue-demand functions are strongly concave, i.e., there exists a constant $C > 0$ such that for each group $i \in [N]$, it holds that
\[
R_i(\tau d + (1-\tau) d') \geq \tau R_i(d) + (1-\tau) R_i(d') + \frac{1}{2}C \tau(1-\tau)(d-d')^2 , \qquad \forall d, d', \tau \in [0, 1].
\]

\item \label{item:assumption-2-lipchtiz-M} The fairness measures are Lipschitz, i.e., there exists a constant $K^{'}$ such that for each group $i \in [N]$, it holds that:
\[
|M_i(p) - M_i(p')| \leq K^{'} |p-p'|, \qquad \forall p, p' \in [\underline{p}, \overline{p}]
\]

\item \label{item:assumption-2-M-UB} There exists a constant $\overline{M} \geq 1$ such that the noisy observation $M_i^{(t)} \in [0, \overline{M}]$ for every selling period $t$ and customer group $i \in [N]$.
\end{enumerate}
\end{assumption}}

\blue{
\subsection{Price Fairness for Multiple Groups}
 Our fairness-aware pricing  for multi-group algorithm (see Algorithm~\ref{alg:mul-price-fairness-main}) has the same structure as Algorithm~\ref{alg:price-fairness-main}.
 The {\sc ExploreUnconstrainedOPTMultiGroup} subroutine (Algorithm~\ref{alg:mul-exp-unconstrainted-opt}) is adapted from  {\sc ExploreUnconstrainedOPT} (Algorithm~\ref{alg:exp-unconstrainted-opt}) by simply considering $N$ groups, 
while the {\sc ExploreConstrainedOPTMultiGroup} subroutine (Algorithm~\ref{alg:mul-exp-constrained-opt}) have several significant differences with {\sc ExploreConstrainedOPT} (Algorithm~\ref{alg:exp-constrained-opt}).

In general, after setting up the checking points as possible mean prices, {\sc ExploreConstrainedOPT} (Algorithm~\ref{alg:exp-constrained-opt}) sets up a fairness-aware
price range(i.e., $[\ell_j - \frac{\lambda \xi}{2}$,$\ell_j + \frac{\lambda \xi}{2}]$) and offers the boundary prices to two groups, while {\sc ExploreConstrainedOPTMultiGroup} (Algorithm~\ref{alg:mul-exp-constrained-opt}) bounds all customer groups within such fairness-aware price range. For the group(s) with an unconstrained optimal price learned within the range, the subroutine offers the learned price. Otherwise, the closest possible price in the range is offered to the group(s) whose learned price is located outside the range. Thus, {\sc ExploreConstrainedOPT} (Algorithm~\ref{alg:exp-constrained-opt}) can be seen as a special case of {\sc ExploreConstrainedOPTMultiGroup} (Algorithm~\ref{alg:mul-exp-constrained-opt}) with a group size, $N$, equals to 2.

\begin{algorithm}[!t]
\blue{
        \caption{Fairness-aware Dynamic Pricing for MultiGroup \label{alg:mul-price-fairness-main}}		
        For each group $z \in [N]$, run {\sc ExploreUnconstrainedOPTMultiGroup}  (Algorithm~\ref{alg:exp-constrained-opt}) separately with the input of group $z$, and obtain the estimation of the optimal price without fairness constraint $\hat{p}_i^\sharp$.

        Given $(\hat{p}_1^\sharp, \hat{p}_2^\sharp,\dots,\hat{p}_N^\sharp)$, run {\sc ExploreConstrainedOPTMultiGroup} (Algorithm~\ref{alg:mul-exp-constrained-opt}), and obtain  $(\hat{p}_1^*, \hat{p}_2^*,\dots,\hat{p}_N^*)$.

        For each of the remaining selling periods, offer $\hat{p}_z^*$ to the customer group $z$.}
    \end{algorithm}

\begin{algorithm}[!t]
\blue{
\caption{{\sc ExploreUnconstrainedOPTMultiGroup}\label{alg:mul-exp-unconstrainted-opt}}		
\SetKwInOut{Input}{Input}
\SetKwInOut{Output}{Output}
\SetKw{True}{TRUE}
\Input{the customer group index $z \in[N]$}
\Output{the estimated unconstrained optimal price $\hat{p}^\sharp$ for group $z$}
$p_L \leftarrow \underline{p}$, $p_R \leftarrow \overline{p}$, $r \leftarrow 0$\;
\While{$|p_L - p_R| > 4 T^{-1/5}$ \label{line:mul-alg-exploreunconstrainedopt-2}}{
	$r \leftarrow r + 1$\;
	$p_{m_1} \leftarrow \frac{2}{3}p_L + \frac{1}{3}p_R$, $p_{m_2} \leftarrow \frac{1}{3}p_L + \frac{2}{3}p_R$\;
		
	Offer price $p_{m_1}$ to \emph{all customer groups} for {$\frac{25 K^4 \overline{p}^2}{C^2 } T^{4/5} \ln (N T)$} selling periods and denote the average demand from customer group $z$ by $\hat{d}_{m_1}$\; \label{line:mul-alg-exploreunconstrainedopt-5}
	Offer price $p_{m_2}$ to \emph{all customer groups} for {$\frac{25 K^4 \overline{p}^2}{C^2 } T^{4/5} \ln (N T)$} selling periods and denote the average demand from customer group $z$ by $\hat{d}_{m_2}$\; \label{line:mul-alg-exploreunconstrainedopt-6}
	{\bf if} $\hat{d}_{m_1}(p_{m_1}-c) > \hat{d}_{m_2}(p_{m_2}-c)$ {\bf then} $p_R \leftarrow p_{m_2}$; {\bf else} $p_L \leftarrow p_{m_1}$\; \label{line:mul-alg-exploreunconstrainedopt-7}
}

\Return $\hat{p}^\sharp =  \frac{1}{2}(p_L + p_R)$\;}
\end{algorithm}


\begin{algorithm}[!t]
\blue{
\caption{\sc ExploreConstrainedOPTMultiGroup\label{alg:mul-exp-constrained-opt}}		
\SetKwInOut{Input}{Input}
\SetKwInOut{Output}{Output}
\Input{the estimated unconstrained optimal prices $(\hat{p}_1^\sharp,\hat{p}_2^\sharp,\dots,\hat p_N^\sharp)$.}
\Output{the estimated constrained optimal prices $(\hat{p}_1^*,\hat{p}_2^*,\dots,\hat p_N^*)$.}

$\xi \leftarrow \max\{ \max_{1\leq i'< j'\leq N}\{|\hat{p}_{i'}^\sharp - \hat{p}_{j'}^\sharp|\} - 8  T^{-1/5}, 0\}$\; 

$J \leftarrow \lceil (\overline{p}-\underline{p}) T^\frac{1}{5}\rceil$ and create $J$ price checkpoints $\ell_1, \ell_2, \dots, \ell_J$ where $\ell_j \leftarrow \underline{p} + \frac{j}{J} (\overline{p} - \underline{p})$\;

\For{each $\ell_j$}{
    Repeat the following offerings for {$6 T^{2/5} \ln (N T)$} selling periods:
    
    \For{$z=1,2,\dots,N$}{
    \If{$\hat p_z^\sharp\in(\ell_j - \frac{\lambda \xi}{2}, \ell_j + \frac{\lambda \xi}{2})$}{
    offer price $p_z(j) \leftarrow\hat p_z^\sharp$ to customer group $z$\;
    }
    \ElseIf{$\hat p_z^\sharp\leq\ell_j - \frac{\lambda \xi}{2}$}{offer price $p_z(j) \leftarrow\max\{\underline{p}, \ell_j - \frac{\lambda \xi}{2}\}$ to customer group $z$\;}
    \Else{offer price $p_z(j) \leftarrow\min\{\overline{p}, \ell_j + \frac{\lambda \xi}{2}\}$ to customer group $z$\;}}
    
    Denote the average demand from customer group $z\in [N]$ by $\hat{d}_z(j)$\;
    $\hat{R}(j) \leftarrow \sum_{z=1}^N(\hat{d}_i(j) (p_z(j) - c)) $\;
}
$j^* \leftarrow \argmax_{j \in \{1, 2, \dots, J\}} \{\hat{R}(j)\}$\;
\Return $\hat{p}^*_z \leftarrow p_z(j^*)$ for all $z\in[N]$\;}
\end{algorithm}

For Algorithm \ref{alg:mul-exp-unconstrainted-opt}, we prove the following upper bounds on the number of selling periods used by the algorithm and its estimation error.

\begin{theorem} \label{thm:mul-exp-unconstrained-opt}
For any input $z \in [N]$, Algorithm~\ref{alg:mul-exp-unconstrainted-opt} uses at most   \blue{$O(\frac{K^4 \overline{p}^2}{C^2} N T^\frac{4}{5}\log (NT) \log (\overline{p} T))$} selling periods and satisfies the fairness constraint during each period. Let $\hat{p}_z^\sharp$ be the output of the procedure.  \blue{With probability $(1 - O(T^{-2} \log(\overline{p}  T)))$, it holds that $|\hat{p}_z^\sharp - p_z^\sharp| \leq   4 T^{-\frac{1}{5}}$ for all $z\in [N]$.} Here, only universal constants are hidden in the $O(\cdot)$ notations.
\end{theorem}

We omit the proof of Theorem~\ref{thm:mul-exp-unconstrained-opt} since it can be directly adapted from Theorem~\ref{thm:exp-unconstrained-opt}.
We then state the following guarantee for Algorithm~\ref{alg:mul-exp-constrained-opt}.

\begin{theorem} \label{thm:mul-exp-constrained-opt}
Suppose that $|\hat{p}_z^\sharp - p_z^\sharp| \leq  4 T^{-1/5}$ for all $z\in[N]$. Algorithm~\ref{alg:mul-exp-constrained-opt} uses at most { $O(\overline{p} N T^{3/5} \log (NT))$} selling periods and satisfies the price fairness constraint during each selling period. With probability $(1- O(\overline{p}  T^{-2}))$, the procedure returns a pair of price $(\hat{p}_1^*, \hat{p}_2^*,\dots,\hat{p}_N^*)$ such that for any $i,j\in[N]$, $|\hat{p}_i^* - \hat{p}_j^*| \leq \lambda\max_{1\leq i'< j'\leq N} |p_{i'}^\sharp - p_{j'}^\sharp|$ and 
\[
\sum_{i=1}^NR_i(p_i^*) -\sum_{i=1}^NR_i(\hat p_i^{*})
 \leq {O( K N\overline{p} T^{-1/5})}.
\]
Here, only universal constants are hidden in the $O(\cdot)$ notations.
\end{theorem}
}

\blue{	
Based on Theorems \ref{thm:mul-exp-unconstrained-opt} and \ref{thm:mul-exp-constrained-opt}, we are ready to state the regret bound of our main algorithm. 

\begin{theorem}\label{thm:mul-main_upper}
With probability $(1 - O(T^{-1}))$, Algorithm~\ref{alg:mul-price-fairness-main} satisfies the fairness constraint and its regret is at most {$O(N T^{\frac45}\log (NT) \log T)$}. Here, the $O(\cdot)$ notation only hides the polynomial dependence on $\overline{p}$, $K$, and $1/C$.
\end{theorem}
Since the proof is similar to that of Theorem~\ref{thm:main_upper}, we omit the proof of Theorem~\ref{thm:mul-main_upper}.
}


\blue{
\subsection{Proof of Theorem~\ref{thm:mul-exp-constrained-opt} for {\sc ExploreConstrainedOPTMultiGroup}}
\label{sec:mul-alg-exp-constrained-opt}

First, the following lemma upper bounds the number of time periods used by the algorithm.

\begin{lemma}\label{lem:mul-exp-constrained-opt-sample-complexity}
Algorithm~\ref{alg:mul-exp-constrained-opt} uses at most  {$O(\overline{p} N T^{3/5} \log (N T))$} selling periods, where only a universal constant is hidden in the $O(\cdot)$ notation.
\end{lemma}

We next turn to prove the (near-)optimality of the estimated prices $\hat{p}^*_z$ for all $z\in[N]$. To this end, we first establish the following key relation between between the constrained optimal solution and the unconstrained optimal solution.
}

\blue{
\begin{lemma}\label{lem:mul-exp-constrained-opt-price-gap}
 There exist two constant $L^*< R^*\in[\min_z p_z^\sharp, \max_z p_z^\sharp]$  such that 
 \[ R^*-L^* = \lambda\max_{1\leq i'< j'\leq N} |p_{i'}^\sharp - p_{j'}^\sharp| \qquad\text{ and }\qquad \begin{aligned}
\left\{
\begin{array}{cl}
    p_z^* \ &\ = L^*, \qquad p_z^\sharp\leq L^*;\\
    p_z^* \ &\ = p_z^\sharp , \qquad p_z^\sharp\in(L^*,R^*); \\
      p_z^* \ &\ = R^*,\qquad p_z^\sharp\geq R^*.
\end{array}
    \right.
\end{aligned} \]
\end{lemma}
}

\blue{
\begin{proof}{Proof.}
    In this proof we assume $p_1^{\sharp}\leq p_2^{\sharp}\leq\dots\leq p_N^{\sharp}$ without loss of generality as the other case can be similarly handled by symmetry.
    We decompose the proof of this lemma into three steps as follows.
    
\noindent \underline{\bf Step I: prove $p_z^*\in[p_1^\sharp,p_N^\sharp]$.}
Since $R_z(d)$ is a unimodal function and $d_z(p)$ is a monotonically decreasing function, we have that $R_z(p) = R_z(d_z(p))$ is a unimodal function for all $z\in[N]$. 

By the definition of optimal solution $({p}_1^*, {p}_2^*,\dots,{p}_N^*)$, we have 
\begin{align}
({p}_1^*, {p}_2^*,\dots,{p}_N^*) = \arg\max_{(p_1, p_2,\dots,p_N) : \max_{1\leq i'< j'\leq N} |p_{i'}^* - p_{j'}^*| \leq \lambda\max_{1\leq i'< j'\leq N} |p_{i'}^\sharp - p_{j'}^\sharp|} \sum_{z=1}^N R_z(p_z) . \label{eq:mul-lem-exp-constrained-opt-price-gap-1}
\end{align}
If $p_z^*\in[p_1^\sharp,p_N^\sharp],\forall z\in[N]$, we obtain the desired conclusion.
And if one of the following three cases occur, we will get a contradiction.
\begin{itemize}
    \item  \underline{\it Case 1: there exist $z$ and $z'$ such that $p_z^*<p_1^\sharp$ and $p_{z'}^*>p_N^\sharp$.}
In this case, $p_{z'}^*-p_z^*> p_N^\sharp- p_1^\sharp\geq \lambda\max_{1\leq i'< j'\leq N} |p_{i'}^\sharp - p_{j'}^\sharp|$, which contradicts with the constraint of optimal solution. 
    \item \underline{\it Case 2: $p_z^*>p_1^\sharp,\forall z\in[N] $ and there exists $z\in[N]$ such that
$p_z^*>p_N^\sharp$.}
In this case, we could define a new solution, which have large value that than the optimal solution, to get the contradiction. Specifically,  we can use $p_N^\sharp$ to substitute  $p_z^*$  for all $z\in[N]$ such that $p_z^*>p_N^\sharp$. And by the unimodal property of $R_z(p)$, we could have 
\[
\sum_{\{z\in[N] \text{ and } p_z^*>p_N^\sharp\}} R_z(p_N^\sharp) > \sum_{\{z\in[N] \text{ and } p_z^*>p_N^\sharp\}} R_z(p_z^*) 
\]

    \item  \underline{\it Case 3: $p_z^*<p_N^\sharp,\forall z\in[N] $ and there exists $z\in[N]$ such that
$p_z^*<p_1^\sharp$.}
In a similar way with Case 2, we can also get a contradiction in this case.
\end{itemize}

\noindent \underline{\bf Step II: prove $\max_{1\leq i'< j'\leq N} |p_{i'}^* - p_{j'}^*|= \lambda\max_{1\leq i'< j'\leq N} |p_{i'}^\sharp - p_{j'}^\sharp|$ .}
Since we already have $\max_{1\leq i'< j'\leq N} |p_{i'}^* - p_{j'}^*|\leq  \lambda\max_{1\leq i'< j'\leq N} |p_{i'}^\sharp - p_{j'}^\sharp|$, we only need to prove there exist $1\leq i'< j'\leq N$ such that $|p_{i'}^* - p_{j'}^*| =  \lambda\max_{1\leq i'< j'\leq N} |p_{i'}^\sharp - p_{j'}^\sharp|$.

With the fact that $p_z^*\in[p_1^\sharp,p_N^\sharp],\forall z\in[N]$, we claim that $|p_N^*-p_1^*| = \max_{1\leq i'< j'\leq N} |p_{i'}^* - p_{j'}^*|$, since if it is not true we can move $p_1^*$ and $p_N^*$ to the boundary of $({p}_1^*, {p}_2^*,\dots,{p}_N^*)$ to get a larger revenue.

By the same argument as Eq.~\eqref{eq:lem-exp-constrained-opt-price-gap-2} and Eq.~\eqref{eq:lem-exp-constrained-opt-price-gap-3} in the proof of Lemma~\ref{lem:exp-constrained-opt-price-gap}, we can obtain that 
 \[p_N^*-p_1^* = \lambda (p_N^\sharp- p_1^\sharp).\]
 Therefore, we get a conclusion that  $p_N^*-p_1^* =\max_{1\leq i'< j'\leq N} |p_{i'}^* - p_{j'}^*|= \lambda\max_{1\leq i'< j'\leq N} |p_{i'}^\sharp - p_{j'}^\sharp|= \lambda (p_N^\sharp- p_1^\sharp).$

\noindent \underline{\bf Step III: discuss the value of $p_z^*$.}
Let $L^*= p_1^*$ and $R^* = p_N^*$. Then we separate $z\in[N]$ into the following three case according to value of $p_z^\sharp$, and discuss the value of $p_z^*$ individually.

\begin{itemize}

    \item \underline{\it Case 1: $p_z^\sharp\in(L^*,R^*)$.}
In this case,  since the unconstrained optimal solution $p_z^\sharp$ of $R_z(p)$ is contained in the interval $(L^*,R^*)$, we claim that $p_z^*= p_z^\sharp$. If it is not true, we could use  $p_z^\sharp $ to substitute $p_z^*$ and get a larger revenue 
without violating the constraint $\max_{1\leq i'< j'\leq N} |p_{i'}^* - p_{j'}^*| \leq \lambda\max_{1\leq i'< j'\leq N} |p_{i'}^\sharp - p_{j'}^\sharp|$. 
    \item  \underline{\it Case 2: $p_z^\sharp\leq L^*$.}
    Since by the result in Step II, we have $p_z^*\in[L^*, R^*], \forall z\in[N] $.
In this case, if  $p_z^*> L^*$ and by the unimodal property of $R_z(p)$, we could have $R_z(L^*)>R_z(p_z^*)$, where we get a contradiction. Therefore, we have $p_z^*=  L^*$ for all $z\in[N]$ satisfying $p_z^\sharp\leq L^*$.

    \item  \underline{\it Case 3: $p_z^\sharp\geq R^*$.}
In a similar way with Case 2, we have $p_z^*=  R^*$ for all $z\in[N]$ satisfying $p_z^\sharp\geq R^*$.
\end{itemize}
  Therefore, by the conclusion of Step III, this lemma is proved.  \Halmos
\end{proof}}

\blue{
The following lemma provides bounds for the $\xi$ parameter which is used in Algorithm~\ref{alg:mul-exp-constrained-opt}  to control the price gaps among customer groups.

\begin{lemma}\label{lem:mul-exp-constrained-opt-xi-bound}
Suppose that $|\hat{p}_z^\sharp - p_z^\sharp| \leq  4 T^{-1/5}$ for all $z\in[N]$, we have that $\lambda \xi \leq \lambda \max_{1\leq i'< j'\leq N}\{|{p}_{i'}^\sharp - {p}_{j'}^\sharp|\}$ and  $\lambda \xi \geq \max\{ \lambda \max_{1\leq i'< j'\leq N}\{|{p}_{i'}^\sharp - {p}_{j'}^\sharp|, 0\} - 16 T^{-1/5}$\}.
\end{lemma}
}

\blue{
\begin{proof}{Proof.}
 By $\max_{1\leq i'< j'\leq N}\{|\hat{p}_{i'}^\sharp - \hat{p}_{j'}^\sharp|\}\leq \max_{1\leq i'< j'\leq N}\{|{p}_{i'}^\sharp - {p}_{j'}^\sharp|+8  T^{-1/5}\}$, we first have that
\begin{align*}
    \lambda \xi &= \lambda \max\{ \max_{1\leq i'< j'\leq N}\{|\hat{p}_{i'}^\sharp - \hat{p}_{j'}^\sharp|\} - 8  T^{-1/5}, 0\} \\&\leq \lambda \max\{0,  \max_{1\leq i'< j'\leq N}\{|{p}_{i'}^\sharp - {p}_{j'}^\sharp|\} + 8 T^{-1/5} - 8 T^{-1/5}\} \\&= \lambda \max_{1\leq i'< j'\leq N}\{|\hat{p}_{i'}^\sharp - \hat{p}_{j'}^\sharp|\} .
\end{align*}
By $\max_{1\leq i'< j'\leq N}\{|\hat{p}_{i'}^\sharp - \hat{p}_{j'}^\sharp|\}\geq \max_{1\leq i'< j'\leq N}\{|{p}_{i'}^\sharp - {p}_{j'}^\sharp|-8  T^{-1/5}\}$,
 we also have that
\begin{align*}
\lambda \xi &= \lambda \max\{ \max_{1\leq i'< j'\leq N}\{|\hat{p}_{i'}^\sharp - \hat{p}_{j'}^\sharp|\} - 8  T^{-1/5}, 0\} \\& \geq 
 \lambda \max\{\max_{1\leq i'< j'\leq N}\{|{p}_{i'}^\sharp - {p}_{j'}^\sharp|\} - 8  T^{-1/5} -8 T^{-1/5}, 0\} \\&\geq  \max\{\lambda\max_{1\leq i'< j'\leq N}\{|\hat{p}_{i'}^\sharp - \hat{p}_{j'}^\sharp|\} -16 T^{-1/5}, 0\} . \square
\end{align*}
\end{proof}}

\blue{
The following lemma shows that our discretization scheme always guarantees that there is a price check point to approximate the  constrained optimal prices.

\begin{lemma}\label{lem:mul-exp-constrained-opt-tilde-j}
Suppose that $|\hat{p}_z^\sharp - p_z^\sharp| \leq  4 T^{-1/5}$ for all $z\in\{ 1, 2, \dots, N\}$, there exists $\tilde{j} \in \{1, 2, \dots, J\}$ such that $|p_z(\tilde{j}) - p_z^*| \leq 22 T^{-1/5}$ for all $z\in[N]$.
\end{lemma}}
\blue{
\begin{proof}{Proof.}
Without loss of generality, we assume $p_1^{\sharp}\leq p_2^{\sharp}\leq\dots\leq p_N^{\sharp}$ in this proof.
Noting that $L^* = p_1^*$ and $R^* = p_N^*$, by  Lemma~\ref{lem:mul-exp-constrained-opt-price-gap} we have 
\begin{equation}\label{eq:mul-j-0}
    p_N^*-p_1^* = \lambda (p_{N}^\sharp - p_{1}^\sharp) \qquad \text{ and }   \qquad
\begin{aligned}
\left\{
\begin{array}{cl}
    p_z^* \ &\ = p_1^*, \qquad p_z^\sharp\leq p_1^*;\\
    p_z^* \ &\ = p_z^\sharp , \qquad p_z^\sharp\in(p_1^*,p_N^*); \\
      p_z^* \ &\ = p_N^*,\qquad p_z^\sharp\geq p_N^*.
\end{array}
    \right.
\end{aligned}
\end{equation}
Consider $\tilde{j} = \arg\min_{j} |\ell_j - (p_1^* + p_N^*)/2|$, we have that $|\ell_{\tilde{j}} - (p_1^* + p_N^*)/2| \leq  T^{-1/5}$. 
By $p_N^*-p_1^* = \lambda (p_{N}^\sharp - p_{1}^\sharp)$ and Lemma~\ref{lem:mul-exp-constrained-opt-xi-bound}, we have
\begin{align}
    \nonumber|\ell_{\tilde j} + \frac{\lambda \xi}{2} - p_N^*|&\leq\nonumber
    \left|\ell_{\tilde j} - \frac{p_1^* + p_N^*}{2}\right|+\left|\frac{p_N^* - p_1^*}{2}-\frac{\xi}{2}\right|\\&\leq 9T^{-1/5}.\label{eq:mul-j-1}
\end{align}
Similarly we could also have 
\begin{align}
    \left|\ell_{\tilde j} - \frac{\lambda \xi}{2} - p_1^*\right|\leq 9T^{-1/5}.\label{eq:mul-j-2}
\end{align}

By the definition of $p_z(\tilde j)$ in Algorithm~\ref{alg:mul-exp-constrained-opt}, we could discuss the following three cases to prove this lemma.

\noindent\underline{\it Case 1: $\hat p_z^\sharp\in( \ell_{\tilde j} - \frac{\lambda \xi}{2} , \ell_{\tilde j} + \frac{\lambda \xi}{2})$.} 
Since $|\hat{p}_z^\sharp - p_z^\sharp| \leq  4 T^{-1/5}$ for all $z\in[N]$, we have 
$p_z^\sharp\in( \ell_{\tilde j} - \frac{\lambda \xi}{2}- 4T^{-1/5} , \ell_{\tilde j} + \frac{\lambda \xi}{2}+ 4T^{-1/5})$.
Combing Eq.~\eqref{eq:mul-j-1} and Eq.~\eqref{eq:mul-j-2} with the above statement, we could have 
\begin{align*}
    p_z^\sharp\in( p_1^*- 13T^{-1/5} ,  p_N^*+ 13T^{-1/5}).
\end{align*}
By Eq.~\eqref{eq:mul-j-0}, we obtain
\begin{align}\label{eq:mul-j-3}
    p_z^*\in( p_z^\sharp- 13T^{-1/5} ,  p_z^\sharp+ 13T^{-1/5}).
\end{align}
 By definition, we have $p_z(\tilde j) = \hat p_z^\sharp$ in this case. Therefore, 
\begin{align*}
    |p_z(\tilde j) - p_z^*| &= |\hat p_z^\sharp - p_z^*|
    \\&\leq 13 T^{-1/5}.
\end{align*}

\noindent\underline{\it Case 2: $\hat p_z^\sharp \geq  \ell_{\tilde j} + \frac{\lambda \xi}{2}$.} 
Since $|\hat{p}_z^\sharp - p_z^\sharp| \leq  4 T^{-1/5}$ for all $z\in[N]$, we have $p_z^\sharp \geq  \ell_{\tilde j} + \frac{\lambda \xi}{2} - 4T^{-1/5}$.
Invoking Eq.~\ref{eq:mul-j-1} to above statement, we have 
\begin{align*}
    p_z^\sharp \geq  p_N^* - 13T^{-1/5}.
\end{align*}
Combining  Eq.~\ref{eq:mul-j-0} with the above inequality, we obtain
\begin{align} \label{eq:mul-j-4}
   p_z^*\in [p_N^* - 13T^{-1/5},p_N^*].
\end{align}
Noting that in this case $p_z(\tilde j) = \min\{\overline{p}, \ell_{\tilde j} + \frac{\lambda \xi}{2}\}$, by  Eq.~\eqref{eq:mul-j-1} and Eq.~\eqref{eq:mul-j-4} we have
\begin{align*}
    |p_z(\tilde j) - p_z^*| &= |\min\{\overline{p}, \ell_{\tilde j} + \frac{\lambda \xi}{2}\} - p_z^*|
    \\&\leq  |\ell_{\tilde j} + \frac{\lambda \xi}{2} - p_z^*|
    \\&\leq  |\ell_{\tilde j} + \frac{\lambda \xi}{2} - p_N^*| +|p_N^*- p_z^*|
    \\&\leq 22 T^{-1/5}.
\end{align*}

\noindent\underline{\it Case 3: $\hat p_z^\sharp\leq \ell_{\tilde j} - \frac{\lambda \xi}{2}$.} 
Similarly with Case 2, we have 
\begin{align*}
    |p_z(\tilde j) - p_z^*| \leq 22 T^{-1/5}.
\end{align*}
Therefore, combining the above three cases we prove this lemma.
 $\square$
\end{proof}}

\blue{
We now prove the following lemma for the (near-)optimality of the estimated constrained prices.
\begin{lemma} \label{lem:mul-exp-constrained-opt-optimality-gap}Suppose that $|\hat{p}_z^\sharp - p_z^\sharp| \leq 4T^{-1/5}$ holds for all $z \in \{1, 2, \dots, N\}$, then with probability at least $(1- (\overline{p}-\underline{p}) T^{-2})$, we have that $\sum_{z=1}^N R_z(\hat{p}_z^*) \geq \sum_{z=1}^N R_z({p}_z^*) - (22 K + 2) N \overline{p} T^{-1/5}$. 
\end{lemma}
 \begin{proof}{Proof.}
By Azuma's inequality, for each $j \in \{1, 2, \dots, J\}, z\in \{1, 2, \dots, N\}$, with probability $1 - N^{-3} T^{-3}$, it holds that
 \begin{align}
 |\hat{d}_z(j) - d_z(p_z(j))| \leq T^{-1/5}. \label{eq:mul-lem-exp-constrained-opt-optimality-gap-1}
 \end{align}
 Therefore, by a union bound, Eq.~\eqref{eq:mul-lem-exp-constrained-opt-optimality-gap-1} holds for each $j \in \{1, 2, \dots, J\}$ and $z \in \{1, 2, \dots, N\}$ with probability at least $1 - J N^{-2}T^{-3} \geq 1 - (\overline p - \underline p)  T^{-2}$. Conditioned on this event, we have that
 \begin{align}
 \forall j \in \{1, 2, \dots, J\}, \forall z \in \{1, 2, \dots, n\}: \qquad |\hat{R}(j) - \sum_{z = 1}^N R_z(p_z(j))| \leq N\overline{p} T^{-1/5}. \label{eq:mul-lem-exp-constrained-opt-optimality-gap-2}
 \end{align}
 With Eq.~\eqref{eq:mul-lem-exp-constrained-opt-optimality-gap-2}, and let $\tilde{j}$ be the index designated by Lemma~\ref{lem:mul-exp-constrained-opt-tilde-j}, we have that 
 \begin{align}
 \sum_{z = 1}^N R_z(\hat{p}_z^*)  &= \sum_{z = 1}^N R_z(p_z(j^*)) \geq \hat{R}(j^*) - N \overline{p}T^{-1/5} \nonumber \\
 &\geq \hat{R}(\tilde{j})- N \overline{p}T^{-1/5} 
 \geq \sum_{z=1}^N R_z(p_z(\tilde{j}))-   2N \overline{p}T^{-1/5}. \label{eq:mul-lem-exp-constrained-opt-optimality-gap-3}
 \end{align}
 By Lemma~\ref{lem:mul-exp-constrained-opt-tilde-j} and Item~\ref{item:assumption-1-lipschitz} of Assumption~\ref{assumption:2}, we have that
 \begin{align}
 \sum_{z = 1}^N R_z(p_z(\tilde{j})) \geq \sum_{z = 1}^N R_z(p_z^*) - 22 N T^{-1/5} \times \overline{p} K. \label{eq:mul-lem-exp-constrained-opt-optimality-gap-4}
 \end{align}
 Combining Eq.~\eqref{eq:mul-lem-exp-constrained-opt-optimality-gap-3} and Eq.~\eqref{eq:mul-lem-exp-constrained-opt-optimality-gap-4}, we prove the lemma. $\square$
 \end{proof}

With these technical lemmas in hand,  we can easily prove  Theorem~\ref{thm:mul-exp-constrained-opt}.}

\blue{
\subsection{General Fairness Measure under the Soft Constraints}
In this section, we study fairness-aware dynamic pricing  in the multi-group setting with  general fairness measure $\{M_i(p)\}$ and soft constraints.

We first  adapt Algorithm~\ref{alg:price-fairness-general} and Algorithm~\ref{alg:exp-constrained-opt-general} to the multi-group version, Algorithm~\ref{alg:mul-price-fairness-general} and Algorithm~\ref{alg:mul-exp-constrained-opt-general}. Then, we state the following guarantee for Algorithm~\ref{alg:mul-exp-constrained-opt-general}.}
\begin{algorithm}[!ht]\blue{
        \caption{Fairness-aware Dynamic Pricing for General Fairness Measure  in Multi-group Case \label{alg:mul-price-fairness-general}}		
        For each group $i \in \{1, 2,\dots,N\}$, run {\sc ExploreUnconstrainedOPTMultiGroup} separately with the input $z = i$, and obtain the estimation of the optimal price without fairness constraint $\hat{p}_i^\sharp$.

        Given $(\hat{p}_1^\sharp, \hat{p}_2^\sharp,\dots,\hat{p}_N^\sharp)$, run {\sc ExploreConstrainedOPTGeneralMultiGroup} (Algorithm~\ref{alg:mul-exp-constrained-opt-general}), and obtain  $(\hat{p}_1^*, \hat{p}_2^*,\dots,\hat{p}_N^*)$.

        For each of the remaining selling periods, offer $\hat{p}_i^*$ to the customer group $i$.}
    \end{algorithm}

\begin{algorithm}[!ht]\blue{
\caption{\sc ExploreConstrainedOPTGeneralMultiGroup\label{alg:mul-exp-constrained-opt-general}}		
\SetKwInOut{Input}{Input}
\SetKwInOut{Output}{Output}
\Input{the estimated unconstrained optimal prices $(\hat{p}_1^\sharp,  \hat{p}_2^\sharp,\dots,\hat{p}_N^\sharp)$\;}
\Output{the estimated constrained optimal prices $(\hat{p}_1^*,  \hat{p}_2^*,\dots,\hat{p}_N^*)$\;}


$J \leftarrow \lceil (\overline{p}-\underline{p}) T^\frac{1}{5}\rceil$ and create $J$ price checkpoints $\ell_1, \ell_2, \dots, \ell_J$ where $\ell_j \leftarrow \underline{p} + \frac{j}{J} (\overline{p} - \underline{p})$\; 

\For{each $\ell_j$}{
    Repeat the following offering for $6  T^{2/5} \ln (NT)$ selling periods: offer price $\ell_j$ to $N$ customer groups\; 
    For each customer group $i \in [N]$, denote the average demand from the customer group $\hat{d}_i(\ell_j)$, and the average of the observed fairness measurement value by $\hat{M}_i(\ell_j)$\;
    Let $\hat{R}_i(\ell_j) \leftarrow \hat{d}_i(\ell_j) \cdot (\ell_j - c)$, for each $i \in [N]$\;
}

For each $i \in [N]$, round up $\hat{p}_i^\sharp$ to the nearest price checkpoint, namely $\ell_{t_i}$\; \label{line:mul-alg-explore-constrained-opt-general-7}


For all tuples $j_1, j_2,\dots,j_N \in \{1, 2, \dots, J\}$, let $\hat{G}(\ell_{j_1}, \ell_{j_2},\dots,\ell_{j_N}) \leftarrow \sum_{i=1}^N\hat{R}_i(\ell_{j_i}) - \gamma \sum_{{1 \leq m < n \leq N}}\max\left(\left|\hat{M}_m(\ell_{j_m}) - \hat{M}_n(\ell_{j_n})\right| - \lambda \max_{1 \leq i' < j' \leq N}\left|\hat{M}_{i'}(\ell_{t_{i'}}) - \hat{M}_{j'}(\ell_{t_{j'}})\right|, 0\right)$\; \label{line:mul-alg-explore-constrained-opt-general-8}

Let $(j_1^*, j_2^*,\dots,j_N^*) \leftarrow \argmax_{j_1, j_2,\dots,j_N \in \{1, 2, \dots, J\}} \{ \hat{G}(\ell_{j_1}, \ell_{j_2},\dots,\ell_{j_N}) \}$\; \label{line:mul-alg-explore-constrained-opt-general-9}

\Return $(\hat{p}^*_1, \hat{p}^*_2,\dots,\hat{p}^*_N) \leftarrow (\ell_{j_1^*}, \ell_{j_2^*},\dots,\ell_{j_N^*})$\;}
\end{algorithm}

\blue{
\begin{theorem}\label{thm:mul-exp-constrained-opt-general}
    Suppose that $|\hat{p}_i^\sharp - p_i^\sharp| < 4T^{-\frac{1}{5}} $ for any $i\in[N]$. Also assume that $\gamma \leq O(1)$.  Algorithm~\ref{alg:mul-exp-constrained-opt-general} uses at most $O(NT^\frac{3}{5}\log(N T))$ selling periods in total, and with probability at least $(1 - O(T^{-1}))$, the procedure returns a pair of prices $(\hat{p}_1^*, \hat{p}_2^*,\dots,\hat{p}_N^*)$ such that
\begin{align}
& \left[\sum_{i=1}^NR_i(p_i^*) -\sum_{i=1}^NR_i(\hat{p}_i^*)\right] \\
& \qquad\qquad +
\gamma \sum_{{1 \leq i < j \leq N}}\max\left(\left|M_i(\hat{p}_i^*) - M_j(\hat{p}_j^*)\right| - \lambda \max_{1 \leq i' < j' \leq N}\left|M_{i'}(p_{i'}^\sharp) - M_{j'}(p_{j'}^\sharp)\right|, 0\right)\leq O\left(N^2T^{-\frac{1}{5}}\right). \label{eq:mul-thm-exp-constrained-opt-general}
\end{align}
Here the $O(\cdot)$ notation hides the polynomial dependence on $\overline{p}$, $\gamma$, $K$, $K'$, $\overline{M}$ and $C$.
\end{theorem}}

\blue{Combining Theorem~\ref{thm:mul-exp-unconstrained-opt} and Theorem~\ref{thm:mul-exp-constrained-opt-general}, we are ready to state the regret bound of the algorithm.
\begin{theorem} \label{thm:mul-main-upper-general}
Assume that $\gamma \leq O(1)$. With probability $(1 - O(T^{-1}))$, the cumulative penalized regret of Algorithm~\ref{alg:mul-price-fairness-general} is at most $\mathrm{Reg}_T^{\mathrm{soft}} \leq O(N^2T^{4/5}\log T\log(NT))$. Here the $O(\cdot)$ notation hides the polynomial dependence on $\overline{p}$, $\gamma$, $K$, $K'$, $\overline{M}$, and $C$.
\end{theorem}
 Since the proof is similar to that of Theorem~\ref{thm:main-upper-general}, we omit the proof of Theorem~\ref{thm:mul-main-upper-general}.   }

\blue{    
\subsection{Proof of Theorem~\ref{thm:mul-exp-constrained-opt-general} for {\sc ExploreConstrainedOPTGeneralMultiGroup}} \label{sec:mul-alg-exp-constrained-opt-general}

First, the following lemma, whose proof is  the same as that of Lemma~\ref{lem:general-lemma1}, upper bounds the number of the selling periods used by the algorithm.
\begin{lemma}
Algorithm~\ref{alg:mul-exp-constrained-opt-general} uses at most $O(\overline{p} N T^\frac{3}{5} \log (NT))$ selling periods, where only a universal constant is hidden in $O(\cdot)$ notation.
\end{lemma}}

\blue{ 
We then turn to upper bound the penalized regret incurred by the estimated prices $\hat{p}_1^*$ and $\hat{p}_2^*$. Define 
\[G(p_1, p_2,\dots,p_N) :=\sum_{i=1}^NR_i(p_i)  -
\gamma \sum_{{1 \leq i < j \leq N}}\max\left(\left|M_i({p}_i) - M_j({p}_j)\right| - \lambda \max_{1 \leq i' < j' \leq N}\left|M_{i'}(p_{i'}^\sharp) - M_{j'}(p_{j'}^\sharp)\right|, 0\right).\]

Note that $G(p_1^*, p_2^*,\dots,p_N^*) = \sum_{i=1}^NR_i(p_i^*)$ and therefore the Left-Hand-Side of Eq.~\eqref{eq:mul-thm-exp-constrained-opt-general} equals to $G(p_1^*, p_2^*,\dots,p_N^*)  - G(\hat p_1^*, \hat p_2^*,\dots,\hat p_N^*) $. The following Lemma~\ref{lem:mul-checkpoints-price-for-general-M} upper bounds the regret due to the discretization method.}

\blue{
\begin{lemma}
\label{lem:mul-checkpoints-price-for-general-M}
$\displaystyle{\max_{j_1, j_2,\dots,j_N \in \{1, 2, \dots, J\}} \{G(\ell_{j_1}, \ell_{j_2},\dots,\ell_{j_N})\}  \geq G(p_1^*, p_2^*,\dots,p_N^*) - (N\overline{p}K + (N^2-N)\gamma K') T^{-\frac{1}{5}}}$.
\end{lemma}}\blue{
\begin{proof}{Proof.}
For each customer group $i \in [N]$, we find the nearest price checkpoint, namely $\ell_{t_i^*}$ to the optimal fairness-aware price $p_i^*$. Note that we always have that $|\ell_{t_i^*} - p_i^*| \leq T^{-\frac15}$.

Combining $|M_i(\ell_{t_i^*})-M_j(\ell_{t_j^*})|\leq  |M_i(\ell_{t_i^*})-M_j(p_i^*)|+ |M_i(\ell_{t_j^*})-M_j(p_j^*)|+ |M_i(p_i^*)-M_j(p_j^*)|$ with the fact that $\max(|M_i(p_i^*) - M_j(p_j^*)| - \lambda\max_{1 \leq i' < j' \leq N}|M_{i'}(p_{i'}^\sharp) - M_{j'}(p_{j'}^\sharp)|, 0)=0$, by item ~\ref{item:assumption-1-lipchtiz-M} of Assumption~\ref{assumption:2} for any $i,j\in[N]$ we have
\begin{align}
    &\max\left(\left|M_i(\ell_{t_i^*}) - M_j(\ell_{t_j^*})\right| - \lambda \max_{1 \leq i' < j' \leq N}\left|M_{i'}(p_{i'}^\sharp) - M_{j'}(p_{j'}^\sharp)\right|, 0\right)\nonumber\\&\leq |M_i(\ell_{t_i^*})-M_i(p_i^*)|+ |M_j(\ell_{t_j^*})-M_j(p_j^*)|+ \max\left(\left|M_i(p_i^*) - M_j(p_j^*)\right| - \lambda \max_{1 \leq i' < j' \leq N}\left|M_{i'}(p_{i'}^\sharp) - M_{j'}(p_{j'}^\sharp)\right|, 0\right)\nonumber\\&\leq 2K'T^{-\frac{1}{5}}.\label{eq:mul_discretization_1}
\end{align}
By \eqref{eq:mul_discretization_1} and item~\ref{item:assumption-1-lipschitz}  of Assumption~\ref{assumption:2}, we obtain
\begin{align*}
& \left|G(p_1^*,p_2^*,\dots,p_N^*) - G(\ell_{t_1^*}, \ell_{t_2^*},\dots,\ell_{t_N^*})\right|
\\&= \sum_{i=1}^NR_i(p_i^*) - \sum_{i=1}^N R_i(\ell_{t_i^*}) + \gamma \sum_{{1 \leq i < j \leq N}}\max\left(\left|M_i(\ell_{t_i^*}) - M_j(\ell_{t_j^*})\right| - \lambda \max_{1 \leq i' < j' \leq N}\left|M_{i'}(p_{i'}^\sharp) - M_{j'}(p_{j'}^\sharp)\right|, 0\right)
\\&\leq \sum_{i=1}^N|R_i(p_i^*)-R_i(\ell_{t_i^*})| + \gamma \sum_{{1 \leq i < j \leq N}}\max\left(\left|M_i(\ell_{t_i^*}) - M_j(\ell_{t_j^*})\right| - \lambda \max_{1 \leq i' < j' \leq N}\left|M_{i'}(p_{i'}^\sharp) - M_{j'}(p_{j'}^\sharp)\right|, 0\right)
\\
& \leq   (N\overline{p}K + (N^2-N)\gamma K') T^{-\frac{1}{5}}.
\end{align*} 
Note that $\max_{j_1, j_2,\dots,j_N \in \{1, 2, \dots, J\}}G(\ell_{j_1}, \ell_{j_2},\dots,\ell_{j_N})  \geq G(t_1^*, t_2^*,\dots,t_N^*)$, and thus we prove the lemma. $\square$

\end{proof}}

\blue{The following lemma uniformly upper bounds the estimation error for $G$ at all pairs of price checkpoints.

\begin{lemma} \label{lem:mul-estimation-on-hat-G-for-general-M}
Suppose that $|\hat{p}_i^\sharp - p_i^\sharp| \leq 4T^\frac{1}{5}$ holds for each  $i \in [N]$. With probability at least $(1 - 12(\overline{p} - \underline{p})T^{-3})$, we have that \[
\left|\hat{G}(\ell_{j_1}, \ell_{j_2},\dots,\ell_{j_N}) - G(\ell_{j_1}, \ell_{j_2},\dots,\ell_{j_N})\right| \leq \left( 2N\overline{p} + 2\gamma \overline{M} + \gamma\lambda (2\overline{M} + 10K') \right)T^{-\frac{1}{5}}
\]
holds for all $j_1, j_2,\dots,j_N \in \{1, 2, \dots, J\}$.
\end{lemma}
	
\begin{proof}{Proof.}
For each $i \in [N]$, since $|\hat{p}_i^\sharp - p_i^\sharp| \leq 4T^{-\frac{1}{5}}$ and $|\ell_{t_i} - \hat{p}_i^\sharp| \leq T^{-\frac{1}{5}}$ (due to the rounding operation at Line~\ref{line:mul-alg-explore-constrained-opt-general-7}), we have that $|p_i^\sharp - \ell_{t_i}| \leq 5T^{-\frac{1}{5}}$. By item~\ref{item:assumption-1-lipchtiz-M} of Assumption~\ref{assumption:2}, we have that 
\begin{align}
\left| M_i(\ell_{t_i}) - M_i(p_i^\sharp)\right| \leq 5K'T^{-\frac{1}{5}}. \label{eq:mul-general-M-function-est-1}
\end{align}
	    
For each price checkpoint $\ell_j$ and each customer group $i \in [N]$, by Azuma's inequality, with probability at least $(1 - 2(NT)^{-3})$, we have that
\begin{align}
\left|\hat{d}_i(\ell_j) - d_i(\ell_j)\right| \leq T^{-\frac{1}{5}}. \label{eq:mul-general-M-fucntion-est-2}
\end{align} 
	   
Therefore, by a union bound, Eq.~\eqref{eq:mul-general-M-fucntion-est-2} holds for all $j \in \{1, 2, \dots, J\}$ and all $i \in [N]$ with probability at least $1 - 2(\overline{p} - \underline{p})T^{-2}$. Conditioned on this event, we have that
\begin{align}
\left|\hat{R}_i(\ell_j) - R_i(\ell_j)\right| \leq \overline{p} T^{-\frac{1}{5}}, \qquad  \forall j \in \{1, 2, \dots, J\}, i \in [N]. \label{eq:mul-general-M-fucntion-est-3}
\end{align}
	    
Similarly, for each price checkpoint $\ell_j$ and each customer group $i \in [N]$, by Azuma's inequality, with probability at least $(1 - 2(NT)^{-3})$,
\begin{align}
\left|\hat{M}_i(\ell_i) - M_i(\ell_i)\right| \leq \overline{M} T^{-\frac{1}{5}}.     \label{eq:mul-general-M-fucntion-est-4}
\end{align}
	     
By a union bound, Eq.~\eqref{eq:mul-general-M-fucntion-est-4} holds for all $j \in \{1, 2, \dots J\}$ and all $i \in \{1, 2\}$ with probability at least $1 - 4(\overline{p} - \underline{p})T^{-2}$. Conditioned on this event, we have that
\[
\left|\hat{M}_i(\ell_{t_i}) - M_i(\ell_{t_i})\right| \leq \overline{M} T^{-\frac{1}{5}}, \quad\text{for all } i\in[N].
\]
By the above inequality, we have 
\begin{align*}
   &\max_{1 \leq i' < j' \leq N}\left|\hat{M}_{i'}(\ell_{t_{i'}}) - \hat{M}_{j'}(\ell_{t_{j'}})\right|-\max_{1 \leq i' < j' \leq N}\left|{M}_{i'}(\ell_{t_{i'}}) -{M}_{j'}(\ell_{t_{j'}})\right|\\&\leq 2\overline{M} T^{-\frac{1}{5}}+\max_{1 \leq i' < j' \leq N}\left|M_{i'}(\ell_{t_{i'}}) - {M}_{j'}(\ell_{t_{j'}})\right|-\max_{1 \leq i' < j' \leq N}\left|{M}_{i'}(\ell_{t_{i'}}) -{M}_{j'}(\ell_{t_{j'}})\right|
   \\& = 2\overline{M} T^{-\frac{1}{5}}.
\end{align*}
Similarly, we could get $$\max_{1 \leq i' < j' \leq N}\left|{M}_{i'}(\ell_{t_{i'}}) -{M}_{j'}(\ell_{t_{j'}})\right|-\max_{1 \leq i' < j' \leq N}\left|\hat{M}_{i'}(\ell_{t_{i'}}) - \hat{M}_{j'}(\ell_{t_{j'}})\right|\leq 2\overline{M} T^{-\frac{1}{5}}.$$
And thus we have, 
\begin{equation}\label{eq:mul-max_1}
    \left|\max_{1 \leq i' < j' \leq N}\left|{M}_{i'}(\ell_{t_{i'}}) -{M}_{j'}(\ell_{t_{j'}})\right|-\max_{1 \leq i' < j' \leq N}\left|\hat{M}_{i'}(\ell_{t_{i'}}) - \hat{M}_{j'}(\ell_{t_{j'}})\right|\right|\leq 2\overline{M} T^{-\frac{1}{5}}.
\end{equation}
By \eqref{eq:mul-max_1}, for any $0\leq m<n\leq N$ we have  
\begin{align}\nonumber
   &\Big|\max\left(\left|\hat{M}_m(\ell_{j_m}) - \hat{M}_n(\ell_{j_n})\right| - \lambda \max_{1 \leq i' < j' \leq N}\left|\hat{M}_{i'}(\ell_{t_{i'}}) - \hat{M}_{j'}(\ell_{t_{j'}})\right|, 0\right)\\&\nonumber\qquad-\max\left(\left|{M}_m(\ell_{j_m}) - {M}_n(\ell_{j_n})\right| - \lambda \max_{1 \leq i' < j' \leq N}\left|{M}_{i'}(\ell_{t_{i'}}) - {M}_{j'}(\ell_{t_{j'}})\Big|, 0\right)\right|\\&\nonumber\leq\left|\hat{M}_m(\ell_{j_m}) - {M}_m(\ell_{j_m})\right|+\left|\hat{M}_n(\ell_{j_n}) - {M}_n(\ell_{j_n})\right|\\&\nonumber\qquad+\lambda \left|\max_{1 \leq i' < j' \leq N}\left|\hat{M}_{i'}(\ell_{t_{i'}}) - \hat{M}_{j'}(\ell_{t_{j'}})\right|-\max_{1 \leq i' < j' \leq N}\left|{M}_{i'}(\ell_{t_{i'}}) -{M}_{j'}(\ell_{t_{j'}})\right|\right|\\&\leq 2\overline{M}(1+\lambda) T^{-\frac{1}{5}}.\label{eq:mul-max_2}
\end{align}
By \eqref{eq:mul-general-M-function-est-1}, with the same method as above we could also get
\begin{align}\nonumber
   &\Big|\max\left(\left|{M}_m(\ell_{j_m}) - {M}_n(\ell_{j_n})\right| - \lambda \max_{1 \leq i' < j' \leq N}\left|{M}_{i'}(p_{i'}^\sharp) - {M}_{j'}(p_{j'}^\sharp)\right|, 0\right)\\&\nonumber\qquad-\max\left(\left|{M}_m(\ell_{j_m}) - {M}_n(\ell_{j_n})\right| - \lambda \max_{1 \leq i' < j' \leq N}\left|{M}_{i'}(\ell_{t_{i'}}) - {M}_{j'}(\ell_{t_{j'}})\Big|, 0\right)\right|\\&\leq 10K'\lambda T^{-\frac{1}{5}}.\label{eq:mul-max_3}
\end{align}
Combining \eqref{eq:mul-max_2} and \eqref{eq:mul-max_3} we obtain
\begin{align}\nonumber
   &\Big|\max\left(\left|\hat{M}_m(\ell_{j_m}) - \hat{M}_n(\ell_{j_n})\right| - \lambda \max_{1 \leq i' < j' \leq N}\left|\hat{M}_{i'}(\ell_{t_{i'}}) - \hat{M}_{j'}(\ell_{t_{j'}})\right|, 0\right)\\&\nonumber\qquad-\max\left(\left|{M}_m(\ell_{j_m}) - {M}_n(\ell_{j_n})\right| - \lambda \max_{1 \leq i' < j' \leq N}\left|{M}_{i'}(p_{i'}^\sharp) - {M}_{j'}(p_{j'}^\sharp)\Big|, 0\right)\right|\\&\leq (2\overline{M}+\lambda(2\overline{M}+ 10K'))T^{-\frac{1}{5}}.\label{eq:mul-max_4}
\end{align}

Now, combining Eq.~\eqref{eq:mul-general-M-fucntion-est-3} and Eq.~\eqref{eq:mul-max_4}, and by the definition of $G(\cdot, \cdot)$, for any $j_1, j_2,\dots,j_N \in \{1, 2, \dots, J\}$, we have that
\begin{align*}
&\left|\hat{G}(\ell_{j_1}, \ell_{j_2},\dots,\ell_{j_N}) - G(\ell_{j_1}, \ell_{j_2},\dots,\ell_{j_N})\right|  \\
& \leq \sum_{i=1}^N |\hat{R}_i(\ell_{j_i}) - R_i(\ell_{j_i})| +\gamma\Big|\max\left(\left|\hat{M}_m(\ell_{j_m}) - \hat{M}_n(\ell_{j_n})\right| - \lambda \max_{1 \leq i' < j' \leq N}\left|\hat{M}_{i'}(\ell_{t_{i'}}) - \hat{M}_{j'}(\ell_{t_{j'}})\right|, 0\right)\\&\nonumber\qquad\qquad\qquad\qquad-\max\left(\left|{M}_m(\ell_{j_m}) - {M}_n(\ell_{j_n})\right| - \lambda \max_{1 \leq i' < j' \leq N}\left|{M}_{i'}(p_{i'}^\sharp) - {M}_{j'}(p_{j'}^\sharp)\Big|, 0\right)\right| \\
& \leq  \left( 2N\overline{p} + 2\gamma \overline{M} + \gamma\lambda (2\overline{M} + 10K') \right)T^{-\frac{1}{5}}.
\end{align*}
Finally, collecting the failure probabilities, we prove the lemma. $\square$
\end{proof}}

\blue{	
Combining Lemma~\ref{lem:mul-checkpoints-price-for-general-M} and Lemma~\ref{lem:mul-estimation-on-hat-G-for-general-M}, we are able to prove Theorem~\ref{thm:mul-exp-constrained-opt-general}.

\begin{proof}{Proof of Theorem~\ref{thm:mul-exp-constrained-opt-general}.}
Conditioned on that the desired event of Lemma~\ref{lem:mul-estimation-on-hat-G-for-general-M} (which happens with probability at least $1 - 12(\overline{p} - \underline{p})T^{-3} \geq 1 - O(T^{-1})$, we have that
\begin{align*}
G(\hat{p}_1^*, \hat{p}_2^*,\dots,\hat{p}_N^*) &\geq \hat{G}(\hat{p}_1^*, \hat{p}_2^*,\dots,\hat{p}_N^*) - \left( N\overline{p} + 2\gamma\overline{M} + 2\gamma\lambda (\overline{M} + 5K') \right)T^{-\frac{1}{5}}\\
& = \max_{j_1, j_2,\dots,j_N \in \{1, 2, \dots, J\}} \{\hat{G}(\ell_{j_1}, \ell_{j_2},\dots,\ell_{j_N})\} - \left( N\overline{p} + 2\gamma\overline{M} + 2\gamma\lambda (\overline{M} + 5K') \right)T^{-\frac{1}{5}}\\
& \geq \max_{j_1, j_2 ,\dots,j_N \in \{1, 2, \dots, J\}} \{G(\ell_{j_1}, \ell_{j_2},\dots,\ell_{j_N}))\} - 2\left( N\overline{p} + 2\gamma\overline{M} + 2\gamma\lambda (\overline{M} + 5K') \right)T^{-\frac{1}{5}} \\
& \geq G(p_1^*, p_2^*,\dots,p_N^*) - (N\overline{p}K + (N^2-N)\gamma K') T^{-\frac{1}{5}} - 4\left( N\overline{p} + \gamma\overline{M} + 2\gamma\lambda (\overline{M} + 5K') \right)T^{-\frac{1}{5}} .
\end{align*}
Here, the first two inequalities are due to the desired event of Lemma~\ref{lem:mul-estimation-on-hat-G-for-general-M}, the equality is by Line~\ref{line:mul-alg-explore-constrained-opt-general-9} of the algorithm, and the last inequality is due to Lemma~\ref{lem:mul-checkpoints-price-for-general-M}.

Observing that the Left-Hand-Side of Eq.~\eqref{eq:mul-thm-exp-constrained-opt-general} equals to $G(p_1^*, p_2^*,\dots,p_N^*) - G(\hat{p}_1^*, \hat{p}_2^*,\dots,p_N^*)$, we prove the theorem. $\square$
\end{proof}}

\blue{
\section{Additional Numerical Results}\label{sec:extra_exp}
In this section, as the supplement of the numerical experiments in Section~\ref{sec:numerical} of the main paper, we present the numerical results of Algorithm~\ref{alg:price-fairness-main} under the linear demand function (Figure~\ref{experiment-figure-linear}) and the inverse proportional demand function (Figure~\ref{experiment-figure-unimodal}). The detailed definitions of the above two types of demand functions can be found at the beginning of Section~\ref{sec:numerical}.

We use the same log-scaled axes and linear-fitting results as in Section~\ref{sec:numerical} to better illustrate the relationship between the regret and the total number of time periods.

We see that when $\lambda = 0$, the two baseline algorithms theoretically achieve $\sqrt{T}$-type regret and indeed beat our algorithm(FDP-DL) in the experiments. On the other hand, our algorithm maintains a stable $T^{4/5}$-type regret under both types of demand functions and for various $\lambda$. Our algorithm performs significantly better than the baselines when $\lambda$ becomes greater than $0$.

Regarding the results of the inverse proportional demand function, we note that since such functions do not meet  Assumptions~\ref{assumption:1}\ref{item:assumption-1-lipschitz} and \ref{assumption:1}\ref{item:assumption-1-strong-concavity}, the corresponding numerical results illustrate the robustness of our algorithm even when some of the theoretical assumptions are not fully satisfied.

}

\begin{figure}[!h]
\centering
\includegraphics[width =0.32\textwidth]{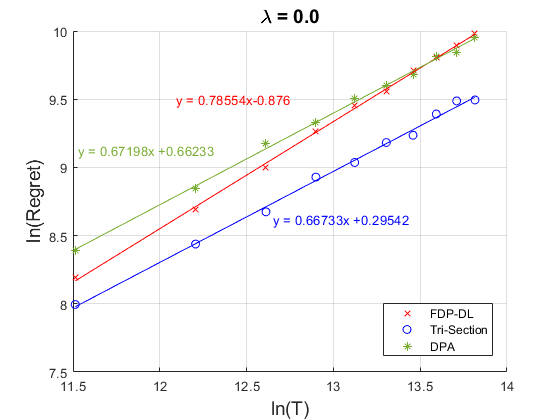}
\includegraphics[width =0.32\textwidth]{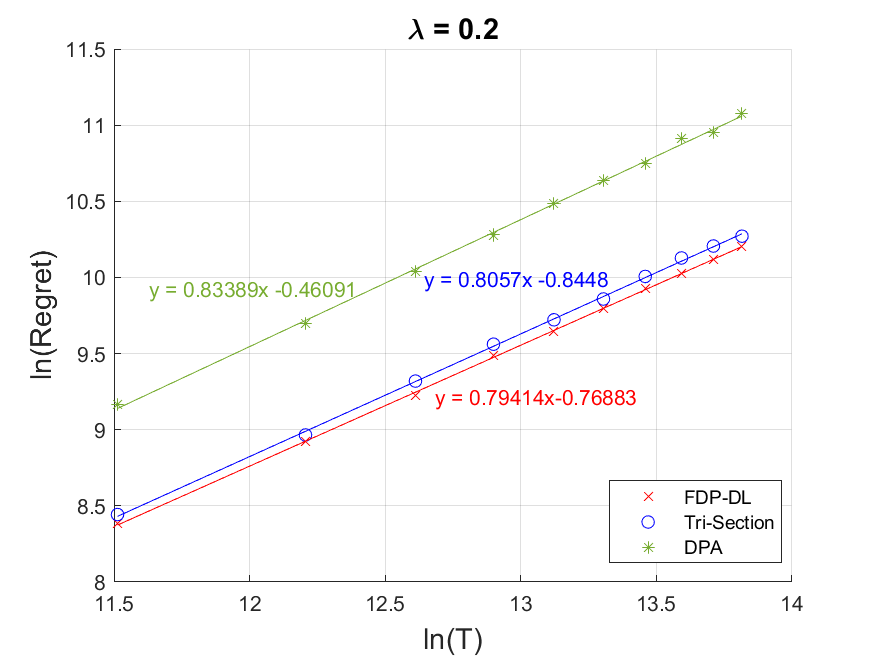}
\includegraphics[width =0.32\textwidth]{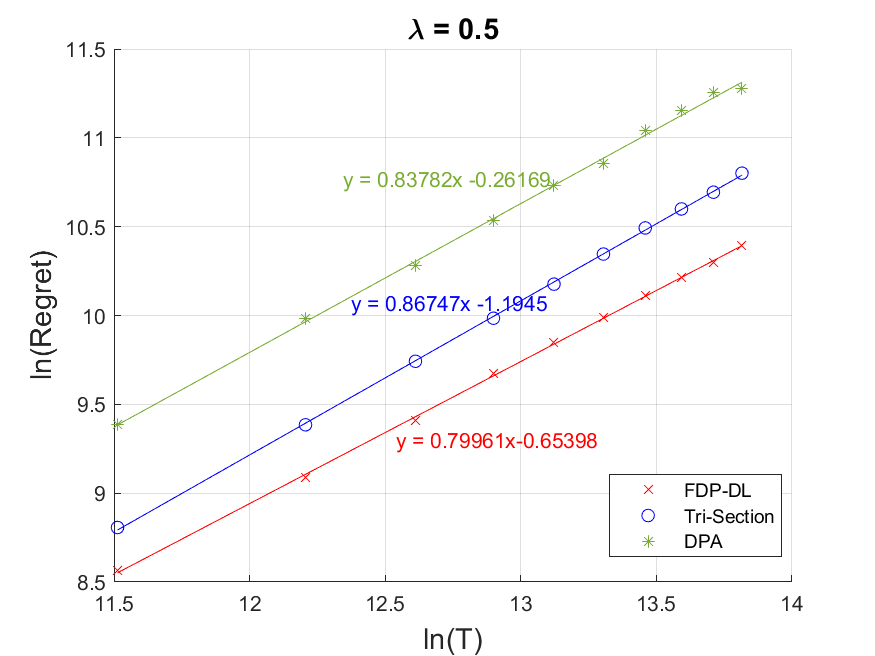}
\includegraphics[width =0.32\textwidth]{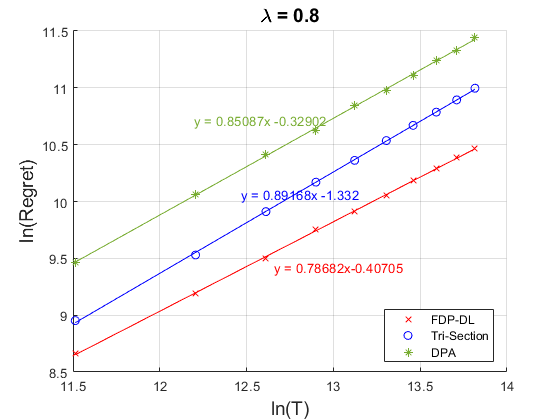}
\includegraphics[width =0.32\textwidth]{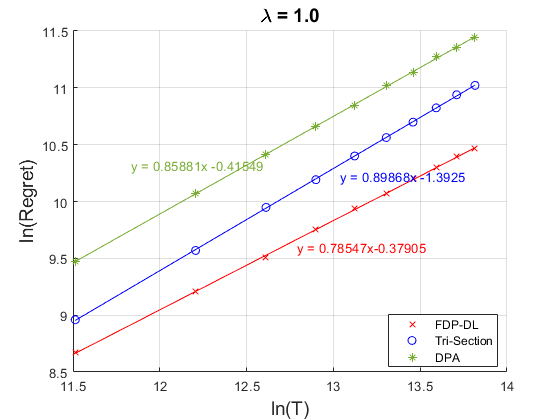}
\caption{\blue{The regret performance of Algorithm~\ref{alg:price-fairness-main} under the linear demand function.  Here the $x$-axis is the logarithm of the total number periods $T$ and the $y$-axis is the logarithm of the cumulative regret. We consider three values of the fairness-ware parameter $\lambda=0$, $0.2$, $0.5$, $0.8$ and $1.0$.}}
\label{experiment-figure-linear}
\end{figure}
\begin{figure}[!h]
\centering
\includegraphics[width =0.32\textwidth]{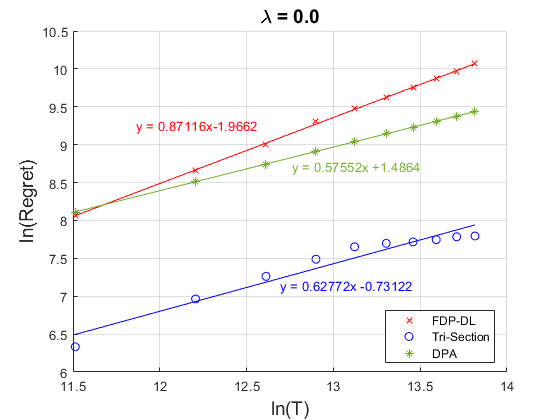}
\includegraphics[width =0.32\textwidth]{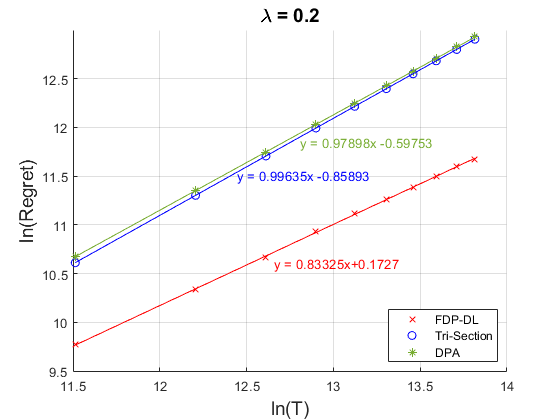}
\includegraphics[width =0.32\textwidth]{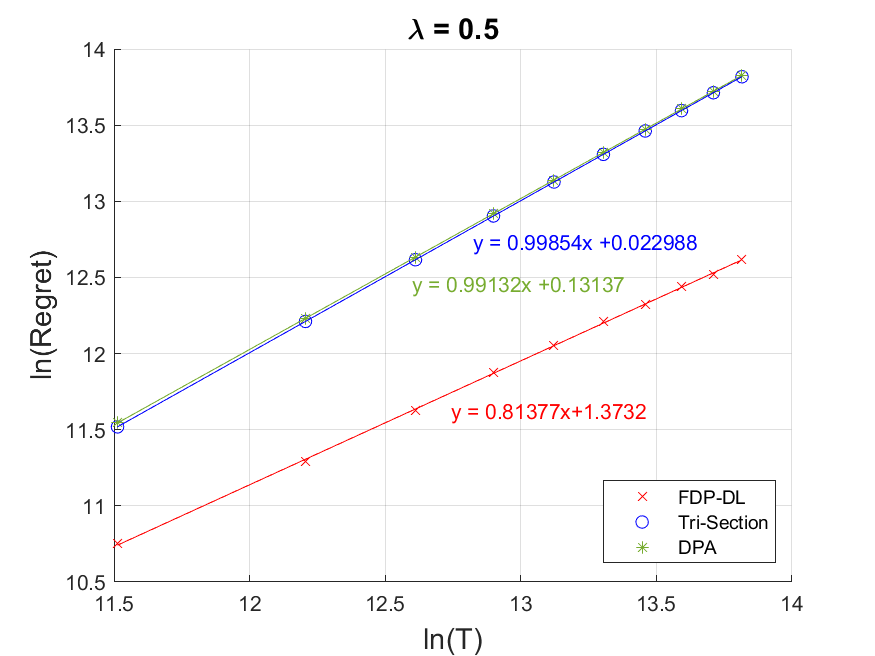}
\includegraphics[width =0.32\textwidth]{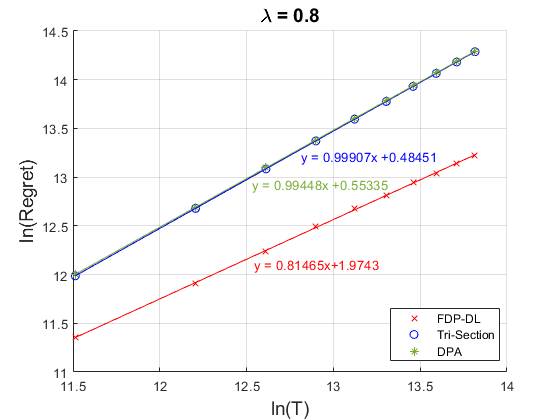}
\includegraphics[width =0.32\textwidth]{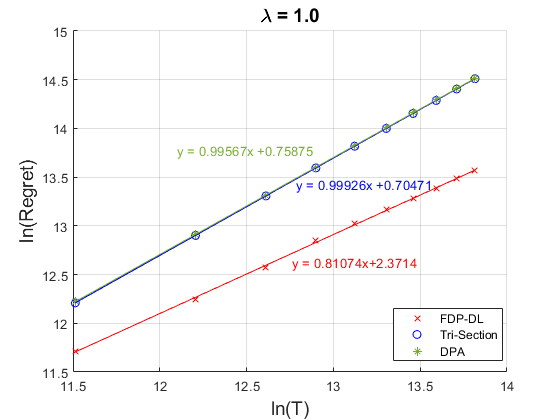}
\caption{\blue{The regret performance of Algorithm~\ref{alg:price-fairness-main} under the inverse proportional demand function. Here the $x$-axis is the logarithm of the total number periods $T$ and the $y$-axis is the logarithm of the cumulative regret. We consider three values of the fairness-ware parameter $\lambda=0$, $0.2$, $0.5$, $0.8$ and $1.0$.}}
\label{experiment-figure-unimodal}
\end{figure}





\end{document}